\documentclass[twoside,11pt]{article}

\usepackage{blindtext}

\usepackage{jmlr2e}
\usepackage{xcolor}
\usepackage{amsmath}
\usepackage{tikz}
\usepackage{comment}
\newcommand{\rhotr}{\rho^{\rm tr}}
\newcommand{\rhote}{\rho^{\rm te}}
\newcommand{\ptr}{p^{\rm tr}}
\newcommand{\pte}{p^{\rm te}}
\DeclareMathOperator*{\argmin}{argmin}

\usepackage{lastpage}
\ShortHeadings{When is Importance Weighting Correction Needed for Covariate Shift Adaptation?}{Gogolashvili, Zecchin, Kanagawa, Kountouris and Filippone}
\firstpageno{1}

\begin{document}

\title{When is Importance Weighting Correction Needed for Covariate Shift Adaptation?}
%\title{Importance Weighting Correction for Regularized Least-Squares Algorithms}

\author{\name Davit Gogolashvili$^*$ \email davit.gogolashvili@eurecom.fr \\
       \addr Data Science Department, EURECOM, France
       \AND
       \name Matteo Zecchin$^*$ \email matteo.zecchin@eurecom.fr \\
       \addr Communication Systems Department, EURECOM, France 
       \AND
       \name Motonobu Kanagawa \email  motonobu.kanagawa@eurecom.fr \\      
       \addr Data Science Department, EURECOM, France
       \AND
       \name Marios Kountouris \email  marios.kountouris@eurecom.fr \\
       \addr Communication Systems Department, EURECOM, France
       \AND
       \name Maurizio Filippone \email  maurizio.filippone@eurecom.fr \\
       \addr Data Science Department, EURECOM, France}

\editor{ }

\maketitle
\def\thefootnote{*}\footnotetext{Equal contribution}

\begin{abstract}
This paper investigates when the importance weighting (IW) correction is needed to address {\em covariate shift}, a common situation in supervised learning where the input distributions of training and test data differ. Classic results show that the IW correction is needed when the model is parametric and misspecified. In contrast, recent results indicate that the IW correction may not be necessary when the model is nonparametric and well-specified. We examine the missing case in the literature where the model is nonparametric and misspecified, and show that the IW correction is needed for obtaining the best approximation of the true unknown function for the test distribution. We do this by analyzing IW-corrected kernel ridge regression, covering a variety of settings, including parametric and nonparametric models, well-specified and misspecified settings, and arbitrary weighting functions.  

%This paper studies regularized least-squares algorithms under {\em covariate shift}, where the test input distribution is different from the training one, and analyzes how importance weighting (IW) can correct for such a shift.   For {\em parametric} models,  it is known the IW correction is necessary for statistical consistency if the model is misspecified, while the IW correction is not needed if the model is well specified. This work focuses on {\em nonparametric} models, specifically those defined in reproducing kernel Hilbert spaces,  and studies the consistency and convergence rates under different weighting strategies of the IW correction. We extend existing generalization bounds to arbitrarily weighting strategies, and clarify several factors that should be considered to make the IW correction successful.   

%This paper studies the effect of importance weighting (IW) correction in learning scenarios affected by covariate shift.  In the case of misspecified \emph{parametric} models and asymptotically large training sets, IW correction is known to be the optimal learning strategy; while, in the case of well-specified models, it can be proven that assigning equal importance to each training sample is the best approach. In this work, we investigate \emph{nonparametric} models and we study how different weighting strategies affect the learning rates. Furthermore, we extend the available generalization bounds to arbitrarily weighted algorithms and we highlight several factors that should be considered to obtain successful re-weighting procedures.
\end{abstract}

\begin{keywords}
  Covariate shift, importance weighting, supervised learning, model misspecification, kernel ridge regression
\end{keywords}

\tableofcontents

\section{Introduction}
In real-world applications of supervised learning methods, training and test data do not necessarily follow the same probability distribution \citep{quinonero2008dataset}. One of the most common situations is the so-called {\em covariate shift} \citep{shimodaira2000}, where the distributions of inputs (covariates) are different for training and test data. Covariate shift naturally occurs when training data are obtained for generic use, and when training data has a selection bias \citep{heckman1979sample,quinonero2008dataset}.      
It arises in a variety of learning problems, including active learning \citep{pukelsheim2006optimal, mackay1992information, cortes2008sample}, domain adaptation \citep{ben2007analysis, MansourMR09, jiang2007instance, cortes2014domain, zhang2012generalization}, and off-policy reinforcement learning \citep{precup2000eligibility,thomas2015high}.

%In many real-world applications of supervised learning, training and testing distributions are different.  The most common setting in the literature is the one in which the conditional distributions of labels given inputs are the same, but the marginal distributions over the inputs differ across training and testing instances. This situation is referred to as {\em covariate shift} \citep{shimodaira2000}, which is a special case of sample selection bias \citep{heckman1979sample}. Covariate shift naturally arises in many common learning scenarios. In active learning problems, the training data points are sampled by the learner at will, while the test data points are bounded to be sampled from the environment distribution \citep{pukelsheim2006optimal, mackay1992information, cortes2008sample}. In domain adaptation, the training data is drawn from a source domain that differs from the target domain, to which the learner is required to transfer its knowledge \citep{ben2007analysis, MansourMR09, jiang2007instance, cortes2014domain, zhang2012generalization}. Covariance shift also occurs in off-policy reinforcement learning, when a learner is required to evaluate a policy using data generated by interacting with the environment using a different policy \citep{precup2000eligibility,thomas2015high}.

The so-called {\em importance weighting (IW)} is a common approach to addressing covariate shift. Let $\rho_X^{\rm tr}(x)$ and $\rho_X^{\rm te}(x)$ be the training and test input distributions, respectively.  
The IW approach uses the ratio $w(x) = d\rho_X^{\rm te} (x) / d\rho_X^{\rm tr}(x)$ of their densities (or the Radon-Nikodym derivative), which is called the {\em importance weighting (IW) function}, to weight the loss function so that the learning objective becomes an unbiased estimator of the expected loss under the test distribution.     
The IW correction has been widely used and studied in the machine learning literature  \citep[e.g.,][]{huang2006correcting,cortes2010,sugiyama2012density,fang2020rethinking}.

\subsection{Importance Weighting for Covariate Shift Adaptation}
\cite{shimodaira2000}  studied IW-corrected maximum likelihood estimation for parametric regression under covariate shift. He showed that the IW correction is relevant when the model is {\em misspecified}, i.e., the true model (e.g., a quadratic function) does not belong to the model class (e.g., linear functions). In this case, the optimal parametric model is the one that minimizes the Kullback–Leibler (KL) divergence to the true model for the {\em test} input distribution. Without the IW correction, maximum likelihood estimation leads to the model that minimizes the KL divergence for the {\em training} distribution, which can drastically differ from the optimal model. The IW correction enables obtaining the optimal model that minimizes the KL divergence for the test input distribution, thus yielding a good predictor in the test phase. On the other hand, in the {\em well-specified} case where the true model is contained in the model class, it is known that the standard maximum likelihood estimation leads to the optimal model, and the IW correction is not necessary. For related results on the IW correction for covariate shifts in parametric models, see \cite{white1981consequences,yamazaki2007asymptotic,wen2014robust,lei2021near}.

%\cite{shimodaira2000} studies IW-corrected maximum likelihood estimation for parametric models under covariate shift.   He considers both {\em well-specified} and {\em misspecified settings}. Under covariate shift, a parametric model is called well-specified if the model class contains the target distribution in the test phase (which we call {\em test distribution} here), while it is misspecified if the model does not. The model misspecification naturally occurs for a parametric model, since the model only has a finite degree of freedom.   In the misspecified setting, the optimal parameters are those minimizing the KL divergence between the model and the test distribution; thus, the resulting model is geometrically interpreted as the {\em projection} of the test distribution onto the model class (or the hypothesis class). It is shown that IW-corrected maximum likelihood estimation is consistent in estimating the optimal parameters in the misspecified setting, while it is no longer asymptotically efficient in terms of the estimator's variance.   In the well-specified setting, the IW correction is not necessary, and the use of uniform weights leads to a consistent estimator. \cite{wen2014robust} provide a further investigation of the misspecified setting. 

Recent works have studied covariate shifts in high-capacity models, such as nonparametric and over-parameterized models. Most of them focused on the {\em well-specified} case where the model class contains the true function, and suggest that the IW correction may {\em not} be necessary to correct for covariate shifts. \citet{kpotufe2021marginal} show that the k-nearest neighbours classifier without the IW correction can achieve minimax optimal convergence rates characterized by the {\em transfer exponent} that quantifies the severity of a covariate shift. They consider the well-specified case for the k-nearest neighbours, as the true regression function is assumed to belong to the H\"older class. Similar results have been obtained by \citet{pathak2022new,ma2022optimally,schmidt2022local,wang2023pseudo} on nonparametric models, assuming the well-specified case.

In case of over-parameterized models, particularly neural networks, recent empirical and theoretical results suggest that over-parameterization helps to improve the robustness against covariate shifts. Arguments exist about whether the IW correction is needed to address covariate shifts. \citet{tripuraneni2021overparameterization} studied the robustness of an over-parameterized model to a covariate shift, by analyzing high-dimensional asymptotics of kernel ridge regression with random features without the IW correction, assuming the true function is a linear function. They observed that over-parameterization could improve the robustness against a covariate shift, which agrees with empirical observations by \citet{hendrycks2018benchmarking} and \cite{hendrycks2021many}. 
\citet{byrd2019effect} empirically studied a deep neural net classifier trained with stochastic gradient descent under a covariate shift. They observed that the effects of the IW correction (where class-conditional weighting is used) only appear in the early stage of training and diminish after the neural net separates positive and negative samples. \citet{xu2021understanding} provides theoretical insights about the observation of \citet{byrd2019effect}, but also suggests the benefits of the IW correction by establishing a generalization bound for IW-corrected empirical risk minimization with a neural net. See \citet{wang2022is,zhai2023understanding} for related discussions.

As reviewed above, the previous works on parametric models suggest that IW correction is needed when the model is misspecified. On the other hand, most prior works on over-parameterized and nonparametric models consider the well-specified case; thus, the observation that the IW correction is unnecessary is consistent with the results on parametric models. Therefore, there is a gap in the literature on the IW correction for covariate shifts; systematic studies are missing in the model-misspecified case for over-parameterized and nonparametric models. Model misspecification occurs also for such models. For example, when the model class consists of smooth functions, the true function may not be smooth; when the model consists of continuous and bounded functions, the true function may be discontinuous or unbounded. Such model misspecification occurs in practice, so it is important to understand how a covariate shift affects the predictive performance of a learning algorithm and whether the IW correction can address it.

\subsection{Contributions}
Motivated by the above gap in the literature, this paper studies the IW correction for a regularized least squares algorithm under a covariate shift. In particular, we consider regularized least squares in a reproducing kernel Hilbert space (RKHS), which results in kernel ridge regression (KRR). This choice enables studying different learning paradigms, from parametric to over-parameterized to nonparametric models, since different choices of the reproducing kernel lead to different RKHSs and thus different model classes. For example, the linear kernel leads to linear models, the neural tangent kernel leads to over-parameterized models \citep{jacot2018neural}, and the Gaussian, Mat\'ern and Laplace kernels lead to nonparametric models \citep{scholkopf2002learning,steinwart2008support}. Thus, the analysis of kernel ridge regression provides a unifying framework for understanding the effects of covariate shifts and the IW correction in different learning approaches.

We study the influences of model misspecification by allowing the true regression function not to be included in the RKHS, and by considering the {\em projection} of the regression function onto the RKHS. The projection is the function in the RKHS that best approximates the regression function in terms of the L2 distance for the {\em test} input distribution. This projection is generally different from that defined for the {\em training} input distribution, as the latter is the best approximation of the regression function for the training distribution. 

Our main contribution is to show that the KRR predictor converges to the projection of the true regression function as the sample size increases, but this projection depends on the weights used in the learning objective (Theorem \ref{main_imperfect} in Section \ref{sec:arbitrary-weights}; see Figure \ref{fig:projections} for an illustration). If the weights are uniform as in the standard KRR, the predictor converges to the projection for the {\em training} input distribution. Therefore, under a covariate shift, the standard KRR is inconsistent as an estimator of the projection for the {\em test} distribution. Using the IW correction, the KRR predictor becomes a consistent estimator of the projection for the test distribution. This result can be understood as an extension of the classic result of \citet{shimodaira2000} on parametric models.  

The above result recovers a recent result of \citet{ma2022optimally} on KRR under a covariate shift as a special case. They show that, assuming that the RKHS contains the regression function (i.e., the well-specified case), the KRR predictor {\em without} the IW correction converges to the regression function as the sample size increases. In the well-specified case, the projection is identical to the regression function and thus does not depend on the input distribution. Therefore, in this case, our result suggests that the KRR converges to the same projection (i.e., the regression function) for {\em any} possible weighting function, recovering the result of \citet{ma2022optimally} as a special case where the weights are uniform.

This observation also agrees with the previous works on nonparametric models (mentioned above), which show that, assuming that the model is well-specified for the regression function, the uniform weighting yields a consistent estimator under a covariate shift. However, in the misspecified case, our result suggests that the IW correction may be needed even for nonparametric models, to obtain the best approximation to the regression function for the test input distribution. This finding thus encourages further research in this setting. 

Moreover, our result is consistent with the previous findings on over-parameterized models, which suggest that over-parameterization improves the robustness against a covariate shift. Over-parameterization increases the capacity of the model and its ability to approximate the true regression function. Therefore, over-parameterization makes the misspecified scenario close to the well-specified one and it makes the model robust to a covariate shift.

We describe the structure of the paper and our additional contributions. Section \ref{sec:learning-under-cov-shift} briefly recalls supervised learning under a covariate shift and the IW correction approach. Section \ref{sec:IW-KRR} describes the IW-corrected kernel ridge regression. 
Section \ref{sec:perfect_weigths} presents our first contribution. We consider the Importance-Weighted Kernel Ridge Regression (IW-KRR) using the true IW function, and examine the various factors that affect the convergence rates (Theorem \ref{main_theorem}). In particular, we quantify the hardness of the covariate shift by a moment condition on the IW function (Assumption \ref{IW_assumption}),  and study how it influences the convergence rates. 
Section \ref{sec:arbitrary-weights} generalizes the result of Section \ref{sec:perfect_weigths} to IW-KRR using an {\em arbitrary} weighting function (Theorem \ref{main_imperfect}). We discuss how the choice of the weighting function affects the convergence of the IW-KRR predictor in the misspecified case, as summarized above. Moreover, we analyze the IW-KRR using a {\em clipped} IW function, showing that it can improve the convergence rates of the IW-KRR using the true IW function if the clipping threshold is chosen appropriately (Theorem \ref{theo:clipped-KRR-re}). 
Section \ref{sec:binary_classification} describes how the above results can be extended to the classification setting.  
We report small simulation experiments in Section \ref{sec:simulations} and conclude in Section \ref{sec:conclusion}. The proofs of the main theoretical results are presented in Appendix.

%This paper is organized as follows. Section 2 briefly recalls the setting of supervised learning under covariate shift and describes the IW correction approach. Section 3 shows our main result on the convergence rates of IW-KRR. In Section 4, we study the IW correction with an arbitrary weight function and derive the convergence properties of the resulting IW-KRR. Section 5 shows the implications of our results to the classification setting. In Section 6, we provide simple numerical experiments to illustrate our results. 

\subsection{Basic Notation}
For a measure $\nu$ on a measurable set $X$ and $p \in \mathbb{N} \cup \{ \infty \}$, let $L^p(X, \nu)$ be the Lebesgue space of $p$-integrable functions with respect to $\nu$:
$$
L^p(X, \nu) := \left\{ f: X \mapsto \mathbb{R} \mid  \|f\|_{p, \nu}  : = \left( \int_{X}f^p(x)d\nu (x) \right)^{1/p}  < \infty \right\}
$$
For $p = 2$, in which we case $L^2(X, \nu)$ is a Hilbert space,  we write the norm as $\left\|  f\right\|_{\nu} := \left\|  f\right\|_{2, \nu}$ to simplify the notation.  
For any $f, g \in L^2(X, \nu)$, let $\left< f, g \right>_{\nu} := \int f(x) g(x) d\nu(x)$ be its inner product.

%The remainder of this paper is structured as follows. In Section 2 we briefly recall the learning problem under covariate shift and we introduce some auxiliary notations and the importance weighted algorithms. Section 3 introduces the assumptions and our main result on the generalization of IW-KRR, followed by some remarks.  sIn Section 4 we study the generalization properties of alternative re-weighting algorithms, which allow us to provide insights into practically relevant weighting schemes. In Section 5 we consider the implication of our results in the context of classification. The final section provides computer simulations supporting our theoretical conclusions.

\section{Learning under Covariate Shift}
\label{sec:learning-under-cov-shift}

\subsection{Expected Prediction Error under the Test Distribution}

We first consider the regression setting, and discuss the classification one in Section \ref{sec:binary_classification}. Let $X$ be a measurable space that serves as a space of inputs (covariates), and $Y = \mathbb{R}$ be the output space.  
Suppose that input-output pairs $(x_i, y_i)_{i=1}^n \in (X \times Y)^n$ are given as training data from a joint probability distribution $\rho^{\rm tr}(x,y)$  in an i.i.d. (independent and identically distributed) manner:  
\begin{align*}
& (x_1,y_1), \dots,  (x_n, y_n) \stackrel{i.i.d.}{\sim}  \rhotr(x,y).
\end{align*}
For conciseness,  we may write 
$$
Z := X \times Y, \quad z_i := (x_i, y_i), \quad {\bf z} := \{z_1, \dots, z_n \} \in Z^n. 
$$
Suppose that the joint distribution decomposes as $\rho^{\rm tr}(x,y) = \rho(y|x) \rho_X^{\rm tr}(x)$ with a conditional distribution $ \rho(y|x)$ on $Y$ given $x \in X$ and a marginal distribution $\rho_X^{\rm tr}(x)$ on $X$.  

Let  $\rho^{\rm te}(x,y)$ be a joint distribution on $X \times Y$ in the {\em test phase} from which test data are generated. Let $(x^{\rm te}, y^{\rm te}) \sim \rho^{\rm te}$ be random variables that represent test data. The task of regression, or {\em prediction}, is to construct a function $f_{\bf z}: X \mapsto Y$ such that its output $f_{\rm z} (x^{\rm te})$ for a test input $x^{\rm te}$ is close to the corresponding test output $y^{\rm te}$. To state this more formally,  let $\mathcal{E}_{\rho^{\rm tr}}(f_{\bf z})$ be the  {\em expected square error}, or the {\em risk}, of the predictor  $f_{\bf z}: X \mapsto Y$  in the test phase:
\begin{equation} \label{eq:test-risk}
    \mathcal{E}_{\rhote}(f_{\bf z}) = \mathbb{E} [ ( f_{\bf z}(x^{\rm te} - y^{\rm te}   )^2 ],
\end{equation}
where the expectation is with respect to $(x^{\rm te}, y^{\rm te}) \sim \rho^{\rm te}$.  The goal is to construct  $f_{\bf z}$ such that this risk becomes as small as possible.

\subsection{Covariate Shift and Importance-Weighting (IW) Correction}

In practice, the test distribution $\rho^{\rm te}(x,y)$ may not be the same as the training distribution $\rho^{\rm tr}(x, y)$, i.e., a dataset shift may occur. {\em Covariate shift} \citep{shimodaira2000} is a specific situation of dataset shift where the test input distribution $\rho_X^{\rm te}(x)$ differs from the training input distribution $\rho_X^{\rm tr}(x)$, while the conditional distribution $\rho(y|x)$ is the same for the test and training data. That is, the training and test distributions are given as:
$$
     \rhotr(x,y) = \rho(y| x)\rhotr_X(x), \quad \rhote(x,y) = \rho(y|x)\rhote_X(x) .
$$

If the test and training input distributions are the same, $\rho_X^{\rm te} = \rho_X^{\rm tr}$, then the training data ${\bf z} = (x_i, y_i)_{i=1}^n$ are i.i.d.~with  $\rho^{\rm te}(x,y)$. Thus, the risk \eqref{eq:test-risk} can be estimated as the empirical risk:
\begin{equation} \label{eq:unif-empirical-risk}
    \mathcal{E}_{\mathbf{z}}(f_{\bf z})=\frac{1}{n}\sum_{i=1}^{n}\left(f_{\bf z}(x_i)-y_i\right)^2,
\end{equation}
In this case, one can construct $f_{\bf z}$ so that this empirical risk becomes small; this is the principle of empirical risk minimization \citep{Vapnik1998}. 

%We consider the {\em covariate shift} setting of \cite{shimodaira2000}, where $\rhotr(x,y)$ and $\rhote(x,y)$ share the same conditional distribution of the output $y$ given the input $x$, but they differ in their marginal distributions on the input space $X$.  More precisely, let $\rho(y|x)$ be the shared conditional distribution, and let $\rhotr_X(x)$ and $\rhote_X(x)$ be the marginal distributions of $\rhotr(x,y)$ and $\rhote(x,y)$ on $X$, respectively
%\begin{equation}
%     \rhotr(x,y) = \rho(y| x)\rhotr_X(x), \quad \rhote(x,y) = \rho(y|x)\rhote_X(x) .     \label{cov_shift_assumption}
%\end{equation}
%Covariate shift refers to the setting in which $\rhotr_X(x)$ and $\rhote_X(x)$ differ. 
However, under covariate shift where inputs $x_1, \dots, x_n$ are generated from a training distribution $\rho_X^{\rm tr}(x)$ different from the test distribution $\rho_X^{\rm te}(x)$, the empirical risk \eqref{eq:unif-empirical-risk} is a {\em biased} estimator of the risk \eqref{eq:test-risk} under the test distribution. Therefore, the minimization of \eqref{eq:unif-empirical-risk} does not necessarily lead to a predictor that makes the risk \eqref{eq:test-risk}  small. One approach to address this issue is to define an {\em unbiased} estimator of the risk \eqref{eq:test-risk} .

To this end, suppose that the test input distribution $\rho_X^{\rm te}$ is absolutely continuous with respect to the training distribution $\rho_X^{\rm tr}$, and let $w(x)$ be the Radon-Nikodym derivative of  $\rho_X^{\rm te}$ with respect to  $\rho_X^{\rm tr}$:
\begin{equation}
    \label{IW_function}
    w(x)=\frac{d\rhote_X}{d\rhotr_X}(x)
\end{equation}
This is called {\em importance-weighting (IW) function}. If $\rho_X^{\rm te}(x)$ and $\rho_X^{\rm tr}(x)$ have probability density functions with respect to a reference measure, the IW function is the ratio of the two density functions. 
 
Then, assuming that the IW function is known, one can define an unbiased estimator of the risk \eqref{eq:test-risk} as an {\em importance weighted empirical risk}:  
\begin{equation}\label{EMP_IW_ LSE}
    \mathcal{E}_{\mathbf{z}}(f_{\bf z})=\frac{1}{n}\sum_{i=1}^nw(x_i)(f_{\bf z}(x_i)-y_i)^{2}.
\end{equation}
We will study learning approaches that use this empirical risk to obtain a predictor. 

While we assume here that the IW function $w(x)$ is known, it is generally unknown and needs to be estimated from available data \citep{sugiyama2012density}. To analyze this case, we will study the use of an {\em arbitrary} weight function in Section \ref{sec:arbitrary-weights}. Moreover, even when the IW function $w(x)$ is known exactly, the use of it may not be optimal; we will also discuss the use of a {\em truncated} IW function in Section \ref{sec:arbitrary-weights}.

\section{Importance-Weighted Kernel Ridge Regression (IW-KRR)}
\label{sec:IW-KRR}

We now introduce Importance-Weighted Kernel Ridge Regression (IW-KRR), which constructs the predictor $f_{\bf z}$ as a function in a {\em reproducing kernel Hilbert space (RKHS)} that minimizes the importance-weighted empirical risk \eqref{EMP_IW_ LSE} plus a regularization term. We first provide preliminary concepts in Section \ref{sec:prelim-IW-KRR} and
we then describe the IW-KRR predictor in Section \ref{sec:IW-KRR-Algorithm}.

\subsection{Preliminaries on Kernels, RKHSs, and Operators} \label{sec:prelim-IW-KRR}
 
 \paragraph{Kernels and RKHSs.}
 
Let $K: X \times X \mapsto \mathbb{R}$ be a continuous, symmetric, and positive semidefinite kernel. That is, for any $n \in \mathbb{N}$ and for any $x_1, \dots, x_n \in X$, the kernel matrix $(K(x_i, x_j))_{i,j = 1}^n \in \mathbb{R}^{n \times n}$ is positive semidefinite. Examples of such kernels include polynomial kernels, $k(x,x') = ( x^\top x' + c )^m$ for $c \geq 0 $ and $m \in \mathbb{N}$, Gaussian kernels, $k(x,x') = \exp(  - \left\|  x - x' \right\|^2 / \gamma^2)$ with $\gamma > 0$, Mat\'ern kernels \cite[Eq.(4.14)]{williams2006gaussian}, and Neural Tangent Kernels \citep{jacot2018neural}. 

Any such kernel $K$ is uniquely associated with a Hilbert space $\mathcal{H}$ of functions on $X$ called RKHS. Denote the inner product and norm of $\mathcal{H}$ by $\left<\cdot, \cdot \right>_{\mathcal{H}}$ and $\left\| \cdot \right\|_{\mathcal{H}}$, respectively. The RKHS $\mathcal{H}$ of kernel $K$ satisfies the following defining properties: 
\begin{enumerate}
    \item For all $x \in X,$ we have $K_x := K(\cdot, x) \in \mathcal{H} ;$
    \item For all $x \in X$ and for all $f \in \mathcal{H}$,
$f(x)=\langle f, K(\cdot, x)\rangle_{\mathcal{H}} \quad$ (Reproducing property).
\end{enumerate}

%In the following, we will study the properties of estimators that belong to the model class $\mathcal{H}$ associated with a reproducing kernel Hilbert space.
\begin{comment}

\begin{definition}
Let $K: X \times X \rightarrow \mathbb{R}$ be a continuous, symmetric, and positive semidefinite kernel, i.e., for any finite set of distinct points $\left\{x_{1}, \ldots, x_{n}\right\} \subset X,$ the matrix $\left(K\left(x_{i}, x_{j}\right)\right)_{i, j=1}^{n}$ is positive semidefinite.  A Hilbert space $\mathcal{H}$ of functions on $X$ equipped with an inner product $\langle\cdot, \cdot\rangle_{\mathcal{H}}$ is called a reproducing kernel Hilbert space (RKHS) with reproducing kernel $K,$ if the following are satisfied:
\begin{enumerate}
    \item For all $x \in X,$ we have $K_x=K(\cdot, x) \in \mathcal{H} ;$
    \item For all $x \in X$ and for all $f \in \mathcal{H}$,
$f(x)=\langle f, K(\cdot, x)\rangle_{\mathcal{H}} \quad$ (Reproducing property).
\end{enumerate}
\end{definition}
\end{comment}

In the following, we assume that $K$ is bounded, i.e., there exists a constant $0 < \kappa < \infty$ such that
%In the following, we assume that $K$ is bounded, meaning that 
\begin{align}
\sup_{x \in X} K(x,x) \leq \kappa.
\label{ass:bounded}
\end{align}
Without loss of generality, we assume $\kappa = 1$ for simplifying the presentation. This condition can always be satisfied by scaling the kernel.  
%To avoid superfluous notations, we further assume $\kappa \leq 1$. This condition can always be achieved by properly scaling the kernel function.

%\subsection{Notations and Auxiliary Operators}
%For two real numbers $a, b$, we write $a \vee b=\max (a, b).$ Let $\nu$ be a measure under consideration (in the consequent it will be training, testing or other).

%By $L^2(X,\nu)$ we denote the Lebesgue spaces of square-integrable functions with respect to measure $\nu$, with the norm given by 
%\[
%    \|f\|_{\nu}=\left(\int_{X}f^2(x)d\nu (x))\right)^{\frac{1}{2}}.
%\]
%For any function $f$  for which the integral is finite, we set $\|f\|_{k,\nu} = \left(\int_{X}f^k(x)d\nu (x))\right)^{1/k}$.

%By $L^2(X,\rhote_X)$ and $L^2(X,\rhotr_X)$ we denote the Lebesgue spaces of square-integrable functions with respect to $\rhote_X$ and $\rhotr_X$. The respective norms are then given by 
%\[
%   \|f\|_{\8rhote_X}=\left(\int_{X}f^2(x)d\rhote_X(x)\right)^{\frac{1}{2}} \quad \text{and} \quad \|f\|_{\rhotr_X}=\left(\int_{X}f^2(x)d\rhotr_X(x)\right)^{\frac{1}{2}}.
%\]
%By $\|f\|_{k,\rhote_X}$ we denote $L^k(X,\rhote_X)$-norm of the function.

\paragraph{Covariance and integral operators.}
For the test input distribution $\rho_X^{\rm te}$, let $T: \mathcal{H} \mapsto \mathcal{H}$ be the {\em covariance operator}
%Our analysis relies upon operators related to the RKHS that we introduce below. Crucial in our analysis are the \textit{covariance} operator $T_{\nu}:\mathcal{H} \rightarrow \mathcal{H}$,
\begin{align} \label{eq:cov-op}
     T f  :=\int K\left(\cdot, x^{\prime} \right) f(x^{\prime}) d \rho_X^{\rm te}(x^{\prime}), \quad f \in \mathcal{H}. 
\end{align}
Similarly, let $L :L^2(X,\rho_X^{\rm te}) \mapsto \mathcal{H}$ be the {\em integral operator}:
\begin{align} \label{eq:integra-operator}
     L  f =\int K\left(\cdot, x^{\prime} \right) f(x^{\prime}) d \rho_X^{\rm te}(x^{\prime}), \quad f \in L^2(X, \rho_X^{\rm te}).  
\end{align}
Note that the domains of these operators are different. 

For $x \in X$, let $T_x : \mathcal{H}  \mapsto \mathcal{H}$ be the covariance operator with $\nu = \delta_{x}$ (the Dirac measure at $x$). It can be compactly written as 
$$
T_x f = K_x f(x) =  K_x \left< K_x, f \right>_{\mathcal{H}} \quad f \in \mathcal{H}.
$$

Under the boundedness condition \eqref{ass:bounded}, the covariance operator $T$ is a positive trace class operator (and hence compact), and thus 
%In both definitions the measure $\nu$ depends on the context. These operators have a similar definition, however they differ in their domains. Under the boundedness assumption (\ref{ass:bounded}), the covariance operator $T_{\nu}$ can be proved to be a positive trace class operator (and hence compact) for any measure $\nu$; namely,
\begin{equation}\label{op_norm_bound}
    \|T \| \leq \operatorname{Tr} ( T ) =\int_{X} \operatorname{Tr} (T_{x}) d \nu(x) \leq 1,
\end{equation}
where $\| T \|$ and ${\rm Tr}(T)$ denote the operator norm and trace of $T$, respectively. 
%where $\|\cdot\|$  denotes the operator norm from $\mathcal{H}$ to $\mathcal{H}$ and $T_x = K_x \langle K_x,\cdot \rangle_{\mathcal{H}}$.

\paragraph{Eigenvalue Decomposition of the Operators.}
Positive trace class operators have at most countably infinitely many non-zero eigenvalues, which are positive. Let $ \mu_1  \geq \mu_2  \geq \cdots \geq 0 $ be the ordered sequence of the eigenvalues of $T$ (with geometric multiplicities), which may be extended by appending zeros if the number of non-zero eigenvalues is finite. From \eqref{op_norm_bound}, we have $\sum_{i=1}^\infty \mu_i  = {\rm Tr}(T) \leq 1$. 

Moreover, the spectral theorem (e.g., \citealt[Theorem A.5.13]{steinwart2008support}) implies that there exists $(e_i)_{i = 1}^\infty \subset \mathcal{H}$ such that (i) $(\mu_i^{1/2} e_i )_{i = 1}^\infty $ is an orthonormal system (ONS) in $\mathcal{H}$, (ii) $(e_i )_{i = 1}^\infty$ is an ONS in $L^2(X, \rho_X^{\rm te})$, and (iii) the covariance and integral operators can be expanded as
%
%Positive trace class operators are known to have at most countably many non-zero eigenvalues, all being non-negative. We denote by $\left(\mu_{i}\left(T_{\nu}\right)\right)_{i \geq 1}$ the ordered sequence of eigenvalues (with geometric multiplicities), possibly extended appending zeros in case of finitely many non-zero eigenvalues. From (\ref{op_norm_bound}) we can conclude that the resulting sequence $\left(\mu_{i}\left(T_{\nu}\right)\right)_{i \geq 1}$ is summable, since $\sum_{i=1}^{\infty} \mu_{i}\left(T_{\nu}\right)= \operatorname{Tr} (T_{\nu}) \leq 1.$ Moreover, the spectral theorem gives
\begin{equation} \label{eq:spectral-decomp}
Tf = \sum_{i \geq 1} \mu_i\left\langle f, \mu_i^{1 / 2} e_i\right\rangle_{\mathcal{H}} \mu_i^{1 / 2} e_i, 
\quad \text { and } \quad L =\sum_{i \geq 1} \mu_i\left\langle\cdot,e_i\right\rangle_{\rho_{X}^{\rm te}} e_i.
\end{equation}
In other words, these two operators share the same eigenvalues $\mu_i$ and eigenfunctions $e_i$.

%\textcolor{blue}{Changing the measure $\nu$ in the covariance operator leads sometimes leads to the the simple behavior of the corresponing eigenvaluas. For example, if $\nu'$ is such that $d\nu'/d\nu$ is bounded by $C,$ then the following bound for the spectrum $\{\lambda'_i\}$ of the covavtiance operator $T_{\nu'}$ holds:
%\begin{equation}{\label{change-of-measure}}
%    \lambda'_i \leq C \lambda_i, \quad \forall i \in \mathbb{N}.
%\end{equation}
%This is a consequence of the Courant-Fischer minimax theorem and can be found for instance in \cite{horn2012matrix}.
%}

%where $\{\mu_i^{1/2} e_i\}_{i=1}^{\infty}$ is an orthonormal basis (ONB) of $\operatorname{Ker} T^{\perp}$ and $\{e_i\}_{i=1}^{\infty}$ is an ONB of $\left(\operatorname{ker} L_K\right)^{\perp}.$

\paragraph{Empirical operators.}
Following \cite{smale2007learning}, we define the sampling operator $S_{\mathbf{x}}: \mathcal{H} \mapsto \mathbb{R}^{n}$ associated with a set $\mathbf{x}=\left\{x_1,\dots,x_n \right\} \in X^n$ as 
$$
\left(S_{\mathbf{x}} f\right)_{i} := f\left(x_{i}\right)=\left\langle f, K_{x_{i}}\right\rangle_{\mathcal{H}}, \quad f \in \mathcal{H},
\quad  (i = 1, \dots, n).
$$
Its adjoint operator $S_{\mathbf{x}}^{\top}$: $\mathbb{R}^{n} \mapsto \mathcal{H}$ is given by 
$$
S_{\mathbf{x}}^{\top} (\mathbf{y}) := \frac{1}{n} \sum_{i=1}^{n} y_{i} K_{x_{i}},
  \quad {\bf y} = (y_1, \dots, y_n)^\top \in \mathbb{R}^n.
$$

%Following \cite{smale2007learning}, we define the \textit{sampling operator} $S_{\mathbf{x}}: \mathcal{H} \rightarrow \mathbb{R}^{n}$ associated with a set $\mathbf{x}=\left\{x_1,\dots,x_n \right\} \subset X$ as
%\[ \left(S_{\mathbf{x}} f\right)_{i}=f\left(x_{i}\right)=\left\langle f, K_{x_{i}}\right\rangle_{\mathcal{H}}, \]
%and its adjoint operator $S_{\mathbf{x}}^{\top}$: $\mathbb{R}^{n} \rightarrow \mathcal{H}$ given by
%\[ S_{\mathbf{x}}^{\top} (\mathbf{y})=\frac{1}{n} \sum_{i=1}^{n} y_{i} K_{x_{i}}. \]
\noindent
We define the empirical covariance operator $T_{\bf x}: \mathcal{H} \mapsto \mathcal{H}$ for a set ${\bf x} = \{x_1, \dots, x_n\} \in X^n$ as 
$$
T_{\bf x} f = \frac{1}{n} \sum_{i=1}^n f(x_i) K(\cdot,x_i), \quad f \in \mathcal{H}.
$$
Using the sampling operator $S_{\bf x}$, it can be written $T_{\bf x} = S_{\mathbf{x}}^{\top} S_{\mathbf{x}}$.  
As $T_{\bf x}$ is the covariance operator (\ref{eq:cov-op}) with the empirical distribution $\nu = \frac{1}{n} \sum_{i=1}^n \delta_{x_i}$, Eq.~\eqref{op_norm_bound} holds for $T_{\bf x}$.

%Finally, we define the \textit{empirical covariance operator} $T_{\mathbf{x}}: \mathcal{H} \rightarrow \mathcal{H}$ such that $T_{\mathbf{x}}=S_{\mathbf{x}}^{\top} S_{\mathbf{x}}.$ It can be shown that $T_{\mathbf{x}}=\frac{1}{n} \sum_{i=1}^{n} K_{x_{i}}\left\langle K_{x_{i}}, \cdot \right\rangle_{\mathcal{H}}$ and, similar to (\ref{op_norm_bound}), we then have
%\begin{align}
%\left\|T_{\mathbf{x}}\right\| \leq \operatorname{Tr}\left(T_{\mathbf{x}}\right) = \frac{1}{n}\sum_{i=1}^{n} K_{x_{i}}\left\langle K_{x_{i}}, \cdot \right\rangle_{\mathcal{H}} \leq 1.
%\end{align}

\subsection{Importance-Weighted Regularized Least-Squares} \label{sec:IW-KRR-Algorithm}

Given training data ${\bf z} = \{ (x_i, y_i) \}_{i=1}^n$, the IW-KRR predictor is defined as the solution to the following importance-weighted regularized least-squares problem: 
\begin{equation}\label{iw_emp_risk}
   f_{\mathbf{z}, \lambda}^{\rm IW}:=\argmin _{f \in \mathcal{H}}\left\{\frac{1}{n} \sum_{i=1}^{n}w(x_i)\left(f\left(x_{i}\right)-y_{i}\right)^{2}+\lambda\|f\|_{\mathcal{H}}^{2}\right\}
\end{equation}
where $\lambda > 0$ is a regularization parameter. If the weights are uniform, $w(x) = 1$, this predictor is identical to the standard KRR.

%We now introduce the importance weighted  regularized least-squares algorithm (IW-RLS). %The IW-RLS solution associated with the kernel $K$ is the minimizer of the following weighted least-square optimization problem defined over the training set of samples $\mathbf{z}=\left\{\left(x_{i}, y_{i}\right)\right\}_{i=1}^{n}$ independently drawn according to $\rhotr$
%\begin{equation}\label{iw_emp_risk}
%   f_{\mathbf{z}, \lambda}^{\rm IW}:=\argmin _{f \in \mathcal{H}}\left\{\frac{1}{n} \sum_{i=1}^{n}w(x_i)\left(f\left(x_{i}\right)-y_{i}\right)^{2}+\lambda\|f\|_{\mathcal{H}}^{2}\right\}
%\end{equation}
%where $\lambda=\lambda_n$ is any positive function of the number of examples $n$ known as a \textit{regularization parameter}.

%In this section, we assume that the importance weighting function $w=d\rhote_X/d\rhotr_X$ is known and we are mainly concerned to study the effects of the weights on generalization properties of IW-RLS. When the importance weights are not known, they can be estimated from the training and testing input data. An analysis of the effect of an error in the estimation of the reweighting function on the accuracy of the learning algorithm is given in \cite{cortes2008sample}.

Note that, because the IW empirical risk \eqref{EMP_IW_ LSE} is an unbiased estimator of the risk  \eqref{eq:test-risk}, the expectation of the objective function in \eqref{iw_emp_risk} is 
%We begin by describing the the data-free limit of (\ref{iw_emp_risk}). As we increase the number of training examples, $n\to\infty$, the functional in the minimization problem (\ref{iw_emp_risk}) becomes
\begin{equation} \label{eq:data-free-object}
    \mathcal{E}_{\rhote}(f)+\lambda \|f\|_{\mathcal{H}}^{2}.
\end{equation}
It is well known that (e.g., \citealt{caponnetto2007optimal}) the risk $\mathcal{E}_{\rhote}(f)$ can be decomposed as 
\[
    \mathcal{E}_{\rhote}(f) = \left\|f-f_{\rho}\right\|^{2}_{\rhote_X}+\mathcal{E}_{\rhote}(f_{\rho})
\]
where $f_\rho: X \mapsto \mathbb{R}$ is the {\em regression function} defined as 
\begin{equation} \label{eq:regress}
    f_{\rho}(x) := \int_{Y}y ~d\rho(y|x), \quad x \in X. 
\end{equation}
Since the last term $\mathcal{E}_{\rhote}(f_{\rho})$ is independent of $f$, the minimizer of \eqref{eq:data-free-object} is thus given by
\begin{equation}\label{iw_risk}
    f_{\lambda}:=\argmin _{f \in \mathcal{H}}\left\{\left\|f-f_{\rho}\right\|^{2}_{\rhote_X}+\lambda\|f\|_{\mathcal{H}}^{2}\right\}.
\end{equation}
One can interpret this $f_\lambda$ as the data-free limit $n \to \infty$ solution to IW-KRR problem \eqref{iw_emp_risk} for fixed $\lambda$. 

%By the standard decomposition of $\mathcal{E}_{\rho}(f)$ we have
%\[
%    \mathcal{E}_{\rhote}(f) = \left\|f-f_{\rho}\right\|^{2}_{\rhote_X}+\mathcal{E}_{\rhote}(f_{\rho}).
%\]
%As the last term is independent from $f$, in the data-free limit  we have that (\ref{iw_emp_risk}) becomes
%\begin{equation}\label{iw_risk}
%    f_{\lambda}:=\argmin _{f \in \mathcal{H}}\left\{\left\|f-f_{\rho}\right\|^{2}_{\rhote_X}+\lambda\|f\|_{\mathcal{H}}^{2}\right\}.
%\end{equation}

The following lemma provides operator-based expressions for the IW-KRR predictor \eqref{iw_emp_risk} and its data-free limit \eqref{iw_risk}, which will be useful in our analysis. 
%The following lemma describes the solution of the minimization problems (\ref{iw_emp_risk}) and (\ref{iw_risk}).
\begin{lemma} \label{lemma:expressions-solutions}
For any $\lambda > 0$, the solutions $f_{\mathbf{z},\lambda}$ in \eqref{iw_emp_risk} and $f_{\lambda}$ in \eqref{iw_risk} exist and are unique. Moreover, we have
%For any $\lambda>0$, the solutions $f_{\mathbf{z},\lambda}$ and $f_{\lambda}$ exist and are unique. Moreover,
\begin{equation} \label{ERM_IW}
    f_{\mathbf{z},\lambda}^{\rm IW}=\left(S_{\mathbf{x}}^{\top}  M_{\mathbf{w}} S_{\mathbf{x}}+\lambda \right)^{-1} S_{\mathbf{x}}^{\top}  M_{\mathbf{w}}\mathbf{y}
\end{equation}
where $\mathbf{y}=(y_1,\dots, y_n)^\top$,  $\mathbf{w}=\left(w(x_1),\dots, w(x_n)\right)^\top$, %with $w(x)=d\rhotr_X/d\rhote_X(x)$ 
and $M_{\mathbf{w}}$ is the diagonal matrix with  diagonal entries $w(x_1), \dots, w(x_n)$.
Furthermore, we have
%For the infinite data case, the the solution of the minimization problem (\ref{iw_risk}) is
\begin{equation} \label{data_free_solution_iw}
    f_{\lambda}=\left(T+\lambda\right)^{-1}L f_{\rho}.
\end{equation}
%where $T = T_{\rhote_X}$ and $L = L_{\rhote_X}.$
\end{lemma}
\begin{proof}
The proof of \eqref{ERM_IW} can be found in \citet[Theorem.~2]{smale2004shannon}, and that of \eqref{data_free_solution_iw} in \citet[Proposition.~7]{cucker2002mathematical}.
%The proof of (\ref{ERM_IW}) can be found in \cite[Theorem.~2]{smale2004shannon}. For (\ref{data_free_solution_iw}) see \cite[Proposition.~7]{cucker2002mathematical}.
\end{proof}
\begin{remark}
If the weights $w(x_1), \dots, w(x_n)$ are all positive, in which case the matrix $M_{\bf w}$ has full rank, the IW-KRR predictor \eqref{ERM_IW} can be equivalently written as  
%Assuming that the matrix $M_{\mathbf{w}}$ has full rank, the solution \eqref{ERM_IW} can be equivalently written as
\begin{align} \label{eq:IW-KRR-matrix}
    f_{\mathbf{z},\lambda}^{\rm IW}=\sum_{i=1}^{n} \alpha_{i} K\left(\cdot, x_{i}\right), \quad \alpha := (\alpha_1, \dots, \alpha_n)^\top := \left(K_{\mathbf{x} \mathbf{x}}+n \lambda M_{1/\mathbf{w}}\right)^{-1} \mathbf{y}
\end{align}
where $K_{\mathbf{xx}} = (K(x_i,x_j))_{i,j = 1}^n   \in \mathbb{R}^{n \times n}$ is the kernel matrix  and $M_{1/\mathbf{w}} \in \mathbb{R}^{n \times n}$ is the diagonal matrix with  diagonal entries $1/w(x_1), \dots, 1/w(x_n)$. 
From the expression \eqref{eq:IW-KRR-matrix}, one can interpret the IW-KRR predictor \eqref{iw_emp_risk} as KRR with data-dependent regularizations, i.e., for a training pair $(x_i, y_i)$, we regularize with the parameter  $n \lambda / w(x_i)$. Thus, if the weight $w(x_i)$ is small, we use a stronger regularizer, and vice versa. 
%Depending on the observation weight we rescale the regularizer accordingly: the higher the weight of an observation, the less we regularize.
\end{remark}

\section{Convergence of IW-KRR with Importance Weights}
\label{sec:perfect_weigths}

In general, the regression function $f_\rho$ may not belong to the RKHS $\mathcal{H}$, i.e., the model may be misspecified.  Therefore it is necessary to define the {\em best approximation} $f_\mathcal{H} \in \mathcal{H}$ of $f_\rho$ in $\mathcal{H}$, where the approximation quality is measured by the distance of $L^2(X, \rho_X^{te})$; namely,
\begin{equation} \label{eq:target-func}
f_\mathcal{H} := \arg\min_{f \in \mathcal{H}} \left\|  f - f_\rho \right\|_{ \rho_X^{\rm te} }^2 = \arg\min_{f \in \mathcal{H}} \int_X \left( f(x) - f_\rho(x) \right)^2 d\rho_X^{\rm te}(x),
\end{equation}
assuming that the minimum exists in $\mathcal{H}$ and is unique.  
Following previous theoretical studies on KRR \cite[e.g., ][]{caponnetto2007optimal,rudi2017generalization}, we consider $f_\mathcal{H}$ as the target function to estimate.

The target function $f_{\mathcal{H}}$ can be interpreted as the {\em projection} of the regression function $f_\rho \in L^2(X, \rho_X^{\rm te})$  onto the closure of the RKHS $\mathcal{H} \subset L^2(X, \rho_X^{\rm te})$. 
This setup is conceptually similar to the parametric setting where the target distribution is the projection (i.e., the best approximation) of the true distribution onto the parametric model class, where the KL divergence measures the approximation quality.

This section aims to understand how the IW-KRR predictor in \eqref{iw_emp_risk} approximates the target function $f_\mathcal{H}$ as the sample size $n$ goes to infinity.
In particular, we quantify the performance of the IW-KRR estimator $f_{ {\bf z}, \lambda }^{\rm IW}$ using the L2 distance with respect to the test input measure $\rho_{X}^{\rm te}$.   

\begin{remark}
It is known that we have $L f_\rho = T f_\mathcal{H}$ \citep[Proposition 1 (ii)]{caponnetto2007optimal}.
Thus, the solution $f_\lambda$ in the data-free limit
\eqref{data_free_solution_iw} can be written as $f_{\lambda}=\left(T+\lambda\right)^{-1}L f_{\rho} = \left(T+\lambda\right)^{-1}T f_{\mathcal{H}}$.
\end{remark}

%Let $f_{\mathcal{H}}$ be the projection of the regression function $f_{\rho}$ onto the closure of $\mathcal{H}$ in $L^2(X,\rhote_X):$
%\[
%\left\|f_{\rho}-f_{\mathcal{H}}\right\|_{\rhote_X}=\operatorname{dist}(f_{\rho}, \mathcal{H}) :=\inf _{f \in \mathcal{H}}\|f_{\rho}-f\|_{\rhote_X}.
%\]
%Clearly none of the learning procedures over $\mathcal{H}$ can achieve better performance than $f_{\mathcal{H}}.$ The goal of this section is to understand: (i)  how well the weighted empirical risk minimizer $f_{\mathbf{z},\lambda}$ approximates $f_{\mathcal{H}},$ (ii) what is the role of the importance weights in this approximation, (iii) how the decay of the regularization parameter $\lambda$ affects the convergence rates. 

%There are various ways to measure the approximation error of $f_{\rho}$ with respect to $f_{\mathbf{z},\lambda}.$ First of all, let us notice that measuring the approximation quality using $L^2(X,\rhotr_X)$ norm does not lead to anything interesting as our goal is not to minimize the risk with respect to $\rhotr.$  In this paper, we shall measure the performance of the model with respect to the $L^2(X,\rhote_X)$  norm. 

\subsection{Assumptions}

We first present key assumptions required for the convergence analysis.  

\paragraph{Existence of the target function.}
We first make the following basic assumption about the target function $f_\mathcal{H}$ in \eqref{eq:target-func}.

\begin{assumption} \label{ass:target-exist}
The target function $f_\mathcal{H} \in \mathcal{H}$ in \eqref{eq:target-func} exists and is unique. 
\end{assumption}

The existence of $f_{\mathcal{H}}$ with finite RKHS norm $ \left\|f_{\mathcal{H}}\right\|_{\mathcal{H}} < \infty$ implies that $f_\rho =  f_{\mathcal{H}} \in \mathcal{H}$, if the RKHS $\mathcal{H}$ is {\em univeral} in $L^2(X, \rho_X^{\rm te})$, i.e., for all $g \in L^2(X, \rho_X^{\rm te})$ and $\varepsilon > 0$, there exists a $f \in \mathcal{H}$ such that $\left\| g - f \right\|_{\rho_X^{\rm te }} < \varepsilon$. This consequence follows from, if  $f_\rho \in L^2(X, \rho_X^{\rm te}) \backslash \mathcal{H}$, the universality of $\mathcal{H}$ implies the existence of a sequence $f_1, f_2, \dots$ such that $\lim_{i\to\infty} \left\|  f_\rho - f_i \right\|_{\rho_X^{\rm te}} = 0$ but we have $\lim_{i\to \infty} \left\| f_i \right\|_{ \mathcal{H} } = \infty$.  
For example, a Gaussian kernel's RKHS is universal in $L^2(X, \rho_X^{\rm te})$ \citep[Theorem 4.63]{steinwart2008support}, and thus so are RKHSs larger than the Gaussian kernel's RKHS.

Therefore, the case where $f_{\mathcal{H}} \not= f_\rho$ with finite RKHS norm $\left\| f_{\mathcal{H}} \right\|_{\mathcal{H}}  < \infty$ occurs if the RKHS $\mathcal{H}$ is {\em not} universal in $L^2(X, \rho_X^{\rm te})$. Examples of kernels inducing non-universal RKHSs include the following: (i) {\em Linear and polynomial kernels}, as their RKHSs are finite-dimensional; (ii) Approximate kernels based on a fixed number of {\em random features} \citep{rahimi2007random}; (iii) {\em Neural Tangent Kernels} with finite network widths \citep{jacot2018neural}; (iv) Structured kernels such as {\em additive kernels} \citep{raskutti2012minimax}.

\paragraph{The smoothness of the target function.}
The next assumption involves a {\em power} of the integral operator $L$ in \eqref{eq:integra-operator}. For a constant $r > 0$, the $r$-th power of $L$ is defined via the spectral decomposition \eqref{eq:spectral-decomp}
$$
L^r f := \sum_{i \geq 1} \mu_i^r\left\langle f ,e_i\right\rangle_{\rho_{X}^{\rm te}} e_i, \quad f \in L^2(X, \rho_X^{\rm te}).
$$
We then make the following assumption for the target function $f_\mathcal{H}$.

%We first introduce some basic assumptions and we then present the convergence results for importance weighted algorithms. 

\begin{assumption}
\label{ass:1}
There exist  $1/2 \leq r \leq 1$ and $g \in L^2(X, \rho_X^{\rm te})$ with $\left\| g \right\|_{\rho^{\rm te}} \leq R$ for some $R > 0$ such that  $f_\mathcal{H} = L^r g$ for the target function $f_\mathcal{H}$ in \eqref{eq:target-func}.
%There exist $r \geq 1/2$ and $R>0$ such that $\left\|L^{-r} f_{\mathcal{H}}\right\|_{\rhote_X} \leq R.$
\end{assumption}

%\begin{remark}
%The condition above can be equivalently stated as follows. There exist $r \geq 1 / 2,$ $R>0$ and $g \in L^2\left(X, \rhote_X\right)$ such that $\|g\|_{\rhote_X} \leq R$ and
%$f_{\mathcal{H}}(x)=\left(L^r g\right)(x),$ where $L^r$ is defined by
%$$
%L^r\left(\sum a_i e_i\right)=\sum \mu_i^r a_i e_i.
%$$
%\end{remark}

Assumption \ref{ass:1} is a common assumption in the literature known as {\em source condition} \citep{smale2004shannon,smale2007learning,de2005model, caponnetto2007optimal}. The constant $r$ quantifies the smoothness (or the regularity) of the target function $f_{\mathcal{H}}$ relative to the least smooth functions in the RKHS $\mathcal{H}$ (for which we have $r=1/2$). Intuitively, a larger $r$ implies $f_{\mathcal{H}}$ being smoother.

%Assuming the finiteness of $\left\|L^{-r} f_{\mathcal{H}}\right\|_{\rhote_X}$ is a common \textit{source condition} in the inverse problem literature \citep{smale2004shannon,smale2007learning,de2005model, caponnetto2006optimal} and it characterizes the regularity of the target function $f_{\mathcal{H}}.$ A bigger $r$ corresponds to higher regularity and it can lead to faster convergence rates. In particular, the case $r=0$ is equivalent to making no assumption, while when $r=1/2$, we require $f_{\mathcal{H}} \in \mathcal{H}$, since $\|L^{1/2}f\|_{\mathcal{H}} =\|f \|_{\rhote_X}.$ For $r\geq 1/2$ the image of the integral operator $L^r\left(L^2(X,\rhote_X) \right)$ becomes a subset of $\mathcal{H},$ which implies that the minimization of risk functional (\ref{LSE}) over $\mathcal{H}$ has at least one solution in $\mathcal{H}.$ This is referred to as the attainable case.

\paragraph{Importance-weighting function.}
We next make an assumption on the IW function $w(x)$, or equivalently, the training and test input distributions $\rho_X^{\rm tr}(x)$ and $\rho_X^{\rm te}(x)$.  In particular, Assumption \ref{IW_assumption} below assumes that the IW function is bounded or all of its moments are bounded.

\begin{assumption}\label{IW_assumption}
Let $w = d\rhote_X/d\rhotr_X$ be the IW function in \eqref{IW_function}. There exist constants $q \in [0,1]$, $W > 0$ and $\sigma > 0$ such that, for all $m \in \mathbb{N}$ with $m \geq 2$, it holds that  
\begin{align}\label{condition_on_iw}
  &  \left(\int_X w(x)^{\frac{m-1}{q}} d\rhote_X(x)\right)^q \leq \frac{1}{2}m!W^{m-2}\sigma^2,
\end{align}
where the left-hand side for $q = 0$ is defined  as $\left\| w^{m-1} \right\|_{\infty, \rhote_X}$, the essential supremum of $w^{m-1}$ with respect to $\rhote_X$.%, denoted by ${\rm ess} {\rm sup}_{\rho_X^{\rm te}} w(x)$. 

%Let $w = d\rhote_X/d\rhotr_X.$  For some $q \in [0,1]$ there exist positive constants $W$ and $\sigma$ depending on $q$ such that for all integer $m \geq 2$
%\begin{equation}\label{condition_on_iw}
%    \left(\int_X w(x)^{\frac{m-1}{q}} d\rhote_X(x)\right)^q \leq \frac{1}{2}m!W^{m-2}\sigma^2.
%\end{equation}
\end{assumption}

If the IW function $w(x)$ is uniformly bounded on $X$, then Assumption \ref{IW_assumption} holds for $q = 0$ and $W = \sigma^2 = {\rm sup}_{x \in X} w(x)$.  If the IW function is not uniformly bounded, Assumption 2 may still hold for $q > 0$. In particular, for $q = 1$, Assumption 2 holds if the moments of the IW function $w(x)$ with respect to the test distribution $\rho_X^{\rm te}(x)$ are bounded for all the orders $m \geq 2$.

%\begin{remark}
%When $q = 0$ it corresponds to the case when $w(x)$ is uniformly bounded over $X.$ In this case $W =\sup_{x \in X} w(x)$ and $\sigma^2 = \int w(x)d\rhote_X = \int w^2(x)d\rhotr_X.$
%\end{remark}

Intuitively, Assumption \ref{IW_assumption} requires that the training distribution $\rho_X^{\rm tr}(x)$ covers the support of the test distribution $\rho_X^{\rm te}(x)$, as the IW function $w(x) = d\rho_X^{\rm te} / d\rho_X^{\rm tr}(x)$ is the Radon-Nikodym derivative of $\rho_X^{\rm te}(x)$ with respect to  $\rho_X^{\rm tr}(x)$. For example, Assumption \ref{IW_assumption} is satisfied for $q \in (0,1]$ if $W \geq 1$, $\sigma^2 \geq 1$ and 
\begin{align*}  
2  \rhote_X\left( \left\{ x \in X:    \frac{d\rho_X^{\rm te}}{d\rho_X^{\rm tr}}(x) \geq t \right\} \right) \leq \sigma^2 \exp\left( - W^{-1} t^{1/q}   \right) \quad \text{for all } \ t > 0.
\end{align*}
See Proposition \ref{prop:assump-IW-sufficient} in Appendix \ref{sec:appendix-auxiliary-results} for a formal result.
 
%\textcolor{blue}{Provided that $W = \sigma^2 \geq 1,$ one can show that Assumption 2 holds if }
%$$2\rhotr_X\{x:w(x)\geq t\} \leq \exp\left(-\frac{t^{1/q}}{W}\right),  \quad t > 0. $$

%Obviously, if the training measure does not properly cover the support of the testing one, learning is impossible. It is not hard to check that the Assumption \ref{IW_assumption} is satisfied when $2\rhotr_X\{x:w(x)\geq t\} \leq W\sigma^2\exp(-\frac{t^{1/q}}{W}),$ restricting the behavior of large values of the Radon–Nikodym derivative. 

 Assumption \ref{IW_assumption} can be equivalently stated as a condition on the {\em R\'{e}nyi divergence} 
between $\rho_X^{\rm te}$ and $\rho_X^{\rm tr}$ \citep{10.5555/1795114.1795157,cortes2010}.  The R\'{e}nyi divergence between $\rho_X^{\rm te}$ and $\rho_X^{\rm tr}$  with parameter $\alpha \in (0, \infty]$ is defined as 
$$
H_{\alpha}(\rhote_X\|\rhotr_X) := 
\begin{cases}
\alpha^{-1} \log \int_X w(x)^{\alpha} d\rhote_X(x) \quad & (\alpha > 0) \\
\log(\|w\|_{\infty, \rhote_X}) \quad & (\alpha = \infty)
\end{cases}
$$
Then  Assumption \ref{IW_assumption} requires that for all integers $m \geq 2$, the Renyi divergence is upper bounded as 
\[
H_{(m-1)/q}(\rhote_X\|\rhotr_X) \leq \frac{1}{m-1}\left(\log m!+ \log\left(\frac{W^{m-2}\sigma^2}{2}\right) \right).
\]
Therefore Assumption \ref{IW_assumption} can be intuitively understood as requiring that the testing distribution $\rho_X^{\rm tr}(x)$ does not deviate too much from the training distribution $\rho_X^{\rm te}(x)$, and the constant $q \in [0,1]$ quantifies the degree of the deviation. 

%The condition (\ref{condition_on_iw}) can be equivalently written as a condition on the R\'{e}nyi Divergence \citep{10.5555/1795114.1795157,cortes2010} as follows
%\[
%H_{(m-1)/q}(\rhote_X\|\rhotr_X) \leq \frac{1}{m-1}\left(\log m!+ \log\left(\frac{W^{m-2}\sigma^2}{2}\right) \right)
%\]
%where 
%\[
%H_{\alpha}(\rhote_X\|\rhotr_X) = \frac{1}{\alpha} \log \int_X w(x)^{\alpha} d\rhote_X(x)
%\]
%is the R\'{e}nyi Divergence with parameter $\alpha.$ Notice that for each fixed $q>0,$ we impose the growth condition on the R\'{e}nyi Divergence w.r.t. the parameter $m.$ 

\paragraph{Effective dimension.}
Lastly, we make an assumption on the {\em effective dimension}  \citep[][]{caponnetto2007optimal} defined as 
$$
\mathcal{N}(\lambda) := \operatorname{Tr}\left( T(T+\lambda)^{-1} \right)  = \sum_{i=1}^\infty \frac{ \mu_i  }{ \mu_i  + \lambda }, \quad \lambda > 0,
$$
where $T : \mathcal{H} \mapsto \mathcal{H}$ is the covariance operator defined in \eqref{eq:cov-op} with the kernel $K$ and the test distribution $\rho_X^{\rm te}(x)$.  
Intuitively, the effective dimension quantifies the {\em degree of freedom} (or the capacity) of the KRR model with the regularization constant $\lambda > 0$ \citep{zhang2005learning}, as it roughly measures the number of eigenvalues greater than the regularization constant $\lambda$. 
As such, the effective dimension $\mathcal{N}(\lambda)$ grows as $\lambda$ decreases (if there are infinitely many positive eigenvalues $\mu_i$), and the growth rate is determined by the decay rate of the eigenvalues $\mu_1 \geq \mu_2 \geq \cdots$. The following assumption characterizes this growth rate of the effective dimension $\mathcal{N}(\lambda)$.

\begin{assumption} \label{ass:effective-dim}
There exists a constant $s \in [0,1]$ such that
%For some $s \in (0,1]$ we assume that
\begin{equation}  \label{eq:const-Es}
  E_{s}:= \max \left(1, \sup _{\lambda \in(0,1]} \sqrt{\mathcal{N}(\lambda) \lambda^{s}} \right) < \infty .
\end{equation}
%where $\mathcal{N}(\lambda)=\operatorname{Tr}\left[T(T+\lambda)^{-1} \right].$
\end{assumption}

Assumption \ref{ass:effective-dim} is satisfied, for example, if the eigenvalues $\mu_i$ decay at the asymptotic order $\mathcal{O}(i^{-1/s})$. As such, a smaller $s$ implies that the eigenvalues decay more quickly, and thus one can understand that the capacity of the KRR model is smaller.  
Note that Assumption \ref{ass:effective-dim} always holds with $s = 1$, as we have $\mathcal{N}(\lambda) \lambda = \sum_{i=1}^\infty \frac{ \mu_i \lambda }{ \mu_i  + \lambda } \leq \sum_{i=1}^\infty  \mu_i = {\rm Tr}( T ) < \infty$.
In general, the effective dimension can characterize more precisely the capacity of the kernel model compare to the more classical covering or entropy numbers \citep{steinwart2009optimal}.  \citet[Definition 1, (iii)]{caponnetto2007optimal} implicitly assume the finiteness of $E_s$.

%The constant $E_{s}$ characterizes the marginal testing distribution $\rhote_X$ through $\mathcal{N}(\lambda)$, also termed as \textit{degrees of freedom} \citep{zhang2005learning} or \textit{effective dimension} \citep{caponnetto2007optimal}. The boundedness of $E_{s}$ was implicitly assumed in \cite[Definition 1, (iii)]{caponnetto2007optimal} and it is satisfied, for instance, when the eigenvalues of $T$, $\mu_i(T)$, have an asymptotic order $\mathcal{O}\left(i^{-1/s} \right).$ In general, the eigenvalue assumption is a tighter measure for the complexity of the RKHS than more classical covering or entropy number assumptions \citep{steinwart2009optimal}. For the case $s=1$, referred to as the \textit{capacity independent} setting, $E_1$ is always bounded as $\mathcal{N}(\lambda) \lambda^{s}=\mathcal{N}(\lambda) \lambda \leq \kappa=1.$ 

\subsection{Convergence Rates of the IW-KRR Predictor}

Before presenting the generalization bounds, let us explain intuitively how the IW-KRR predictor converges to the target function $f_\mathcal{H}$ as the sample size $n$ increases. First, define $\xi(z_i) := y_i K_{x_i} w(x_i) \in \mathcal{H}$ with $z_1, \dots, z_n =(x_1, y_1), \dots, (x_n, y_n) \stackrel{i.i.d.}{\sim} \rho_{\rm tr}(x,y)$, which are i.i.d.~$\mathcal{H}$-valued random variables. The expression $S_{\mathbf{x}}^{T}  M_{\mathbf{w}}\mathbf{y}$ in \eqref{ERM_IW} can be written as the empirical average of $\xi_1, \dots, \xi_n$. Thus, by the law of large numbers,  we have 
\[
\frac{1}{n}\sum_{i=i}^n \xi(z_i)  \longrightarrow \int_{X \times Y} y w(x) K_x  d \rho^{\rm tr}(x,y) = \int f_\rho(x) K_x d\rho_X^{\rm te} (x) = L f_{\rho}
\]
as $n \to \infty$, where $L$ is the integral operator in \eqref{eq:integra-operator}.  
Therefore the term $S_{\mathbf{x}}^{T}  M_{\mathbf{w}}\mathbf{y}$ converges to $L f_{\rho}$ as $n \to \infty$. 

Second, for an arbitrary function $f \in \mathcal{H}$,  define $\xi(x_i) := w(x_i) f(x_i) K_{x_i} \in \mathcal{H}$. Then $\xi(x_1), \dots, \xi(x_n)$ are i.i.d.~$\mathcal{H}$-valued random variables, and the term $S_{\mathbf{x}}^{\top} M_{\mathbf{w}} S_{\mathbf{x}}$ in \eqref{ERM_IW} is their empirical average. Thus, as $n \to \infty$, we have  
$$
S_{\mathbf{x}}^{\top} M_{\mathbf{w}} S_{\mathbf{x}} f = \frac{1}{n} \sum_{i=1}^n \xi\left(x_i\right) \longrightarrow \int w(x) f(x) K_x d\rho_X^{\rm tr}(x) = 
\int  f(x) K_x d\rho_X^{\rm te}(x) = Tf.
$$
Therefore, the term $S_{\mathbf{x}}^{\top} M_{\mathbf{w}} S_{\mathbf{x}}$ converges to the covariance operator $T$.   

To put it all together, the IW-KRR predictor $f_{\mathbf{z}, \lambda}^{\rm IW}$ in \eqref{ERM_IW} converges to $f_{\lambda}=\left(T+\lambda\right)^{-1}L f_{\rho}$ in \eqref{data_free_solution_iw} as $n \to \infty$ if $\lambda$ is fixed. 
By combining an approximation analysis of $f_\lambda$ converging to $f_\mathcal{H}$ as $\lambda \to 0$, we obtain a  generalization bound of the IW-KRR predictor, as summarized in Theorem \ref{main_theorem} below. Since it is a special case of a more generic result stated in Theorem \ref{main_imperfect} in Section \ref{sec:arbitrary-weights}, we omit its proof. 
%The proof can be found in Appendix \ref{sec:proof-main-IW-KRR}.

\begin{comment}
{\color{blue}
Before providing the generalization bounds let us explain why IW correction is a reasonable approach. First consider the random variable $\xi := y K_x w(x)$ on $(X, \rhotr)$ with the values in the Hilbert space $\mathcal{H}.$ By the low of large numbers, the term $S_{\mathbf{x}}^{T}  M_{\mathbf{w}}\mathbf{y}$ in \eqref{ERM_IW} converges to
\[
\frac{1}{n}\sum_{i=i}^n \xi(x_i) \longrightarrow \int_X y w(x) K_x  d \rhotr_X(x) = L f_{\rhotr}
\]
as the number of samples goes to infinity. This shows that $S_{\mathbf{x}}^{T}  M_{\mathbf{w}}\mathbf{y}$ is a good approximation of $L f_{\rhotr}$. Second, for the function $f \in \mathcal{H}$, consider the random variable $\xi:=w(x) f(x) K_x$ on $\left(Z, \rhotr \right)$ with values in $\mathcal{H}$. Again we have
$$
S_{\mathbf{x}}^{\top} M_{\mathbf{w}} S_{\mathbf{x}} = \frac{1}{n} \sum_{i=1}^n \xi\left(z_i\right) \longrightarrow T f
$$
meaning that $S_{\mathbf{x}}^{\top} M_{\mathbf{w}} S_{\mathbf{x}}$ is a good approximation of the covariance operator $T$. Thus $f_{\mathbf{z},\lambda}^{\rm IW}$ should approximate $f_{\lambda} := (T+\lambda I)^{-1}Lf_{\mathcal{H}}$  and one would expect small {\em testing} error of $f_{\mathbf{z}, \lambda}^{\rm IW}-f_\lambda$. This observation is made precise by the following theorem.
}
\end{comment}

\begin{theorem}    \label{main_theorem}
Let $\rho^{\rm te}$ and $\rho^{\rm tr}$ be probability distributions on $X \times [-M, M]$, where $M > 0$ is a constant, and $K: X \times X \mapsto \mathbb{R}$ be a kernel.  Suppose $\rho^{\rm te}$, $\rho^{\rm tr}$ and $K$ satisfy  Assumptions \ref{ass:target-exist}, \ref{ass:1}, \ref{IW_assumption} and \ref{ass:effective-dim} with constants $r \in [1/2, 1]$, $R \in (0, \infty)$,  $q \in [0,1]$, $W \in (0, \infty)$, $\sigma \in (0, \infty)$,  $s \in [0,1]$ and $E_s \in [1, \infty)$. 
Let $\delta \in (0,1)$ be an arbitrary constant. Let 
\begin{equation}  \label{opt_reg}
\lambda = c n^{- \beta},
\end{equation}
where $\beta > 0$ is defined by 
$$
\beta := \frac{1}{2r +  s (1-q)  + q} = \frac{1}{2r + A}, \quad \text{where} \quad \quad A := s (1-q)  + q,
$$
and $c > 0$ is such that 
$$
c \geq \left( 64 (W+\sigma^2)  (E_{s})^{2(1-q)}  \log^2\left( 6/\delta \right)  \right)^{1/(1+A)}.
$$
Suppose $n$ is large enough so that $\lambda \leq 1$. Then, with probability greater than $1-\delta$, it holds that 
\begin{align}   \label{IW_generalization_bound}
\left\| f_{\mathbf{z},\lambda}-f_{\mathcal{H}} \right\|_{\rho_X^{\rm te}} & \leq n^{- r  \beta}  \left\{ 16  \left(      M   +  \left\| f_\mathcal{H} \right\|_{\mathcal{H}}    \right)  \left( W   +   \sigma  (E_{s})^{1-q}   \right) c^{-A/2} \log \left( 6/\delta \right)     + c^{r}    R \right\}.  
\end{align}
\end{theorem}

Theorem \ref{main_theorem} provides a probabilistic error bound for the IW-KRR predictor in estimating the target function $f_{\mathcal{H}}$. We can make the following observations:
\begin{itemize}
    \item The constant $r$ in Assumption \ref{ass:1} quantifies the smoothness of the target function $f_\mathcal{H}$. Therefore, as $r$ increases, the problem becomes easier, and the rate \eqref{IW_generalization_bound} becomes faster.  
    
    \item The constant $s$ in Assumption \ref{ass:effective-dim} quantifies the capacity of the RKHS $\mathcal{H}$, and a larger $s$ implies that the RKHS has a higher capacity. The limit $s \to 0$ is the case where the RKHS is finite-dimensional, and because $f_\mathcal{H} \in \mathcal{H}$, the learning problem is easier than larger $s$. Indeed, the rate approaches the parametric rate $\mathcal{O}(n^{-1/2})$ as $s \to 0$ if $q = 0$. 
    
    \item 
The rate \eqref{IW_generalization_bound} captures the influence of covariate shift on the hardness of the learning problem. To discuss this, consider the two extreme cases of the constant $q$ in Assumption \ref{IW_assumption}, $q = 0$ and $q = 1$. 
If $q = 0$, in which case the IW function $w(x)$ is bounded, the convergence rate is $\mathcal{O}\left(n^{-\frac{r}{2r+s}} \right)$, which matches the optimal rate of standard KRR in \citet[Theorem 3]{caponnetto2007optimal}.\footnote{\citet{caponnetto2007optimal} use constants $1 < b \leq \infty$ and $1 \leq c \leq 2$ to characterize the learning problem. By setting $b = 1/s$ and $c = 2r$, our setting with $q = 0$ is recovered.}
For $q = 1$, where the IW function $w(x)$ may be unbounded, the rate of the IW-KRR is $\mathcal{O}\left(n^{-\frac{r}{2r+1}} \right)$, which is independent of the constant $s \in [0,1]$ and slower than the rate for $q = 0$.  Thus, the rate suggests that learning becomes harder as the covariate shift worsens.  

\item This last point agrees with the earlier observation of \citet{cortes2010} that the IW correction can succeed when the IW function is bounded, while it leads to slower rates when the IW function is not bounded. \citet{kpotufe2021marginal} point out that such slow rates are not only due to the IW correction itself; for any learning approach, the rates become slower in a minimax sense due to the hardness of the learning problem caused by covariate shift. 
\end{itemize}

\subsection{Examples}
Here we discuss two examples of RKHSs to illustrate Theorem \ref{main_theorem}. One is where the RKHS is finite-dimensional, and the other is where the RKHS is norm-equivalent to a Sobolev space.

\subsubsection{Finite Dimensional RKHSs}
We first consider the case where the kernel has a finite rank $N$, i.e., the case where the eigenvalues of the covariance operator satisfy $\mu_j = 0$ for all $j > N$. Examples of such kernels include the linear kernel $K\left(x, x^{\prime}\right)=\left\langle x, x^{\prime}\right\rangle_{\mathbb{R}^{d}}$, polynomial kernels $K(x, x')=\left(c+\left\langle x, x^{\prime}\right\rangle_{\mathbb{R}^{d}}\right)^{m}$ with $c \geq 0$ and $m \in \mathbb{N}$, approximate kernels with random features with a fixed number of features \citep{rahimi2007random}, approximate kernels given by the Nystr\"om method with a fixed number of inducing inputs \citep{williams2000using}, and the Neural Tangent Kernels with finite network widths \citep{jacot2018neural,arora2019exact}. In these cases,  Assumption \ref{ass:effective-dim} holds with $s = 0$ and we have $E_s \leq \sqrt{N}$. Therefore we directly obtain the following corollary from Theorem \ref{main_theorem}.

%\paragraph{Finite dimensional RKHS.} We now turn to deriving some explicit consequences of our main theorems for specific classes of reproducing kernel Hilbert spaces. Our first corollary applies to problems for which the kernel has finite rank $N$, meaning that its eigenvalues satisfy $\mu_{j}(T)=0$ for all $j>N$. Examples of such finite rank kernels include the linear kernel $K\left(x, x^{\prime}\right)=\left\langle x, x^{\prime}\right\rangle_{\mathbb{R}^{d}}$ and the kernel $K(x, x')=\left(1+\left\langle x, x^{\prime}\right\rangle_{\mathbb{R}^{d}}\right)^{k}$ generating polynomials of degree $k.$

\begin{corollary}   
Let $\rho^{\rm te}$ and $\rho^{\rm tr}$ be probability distributions on $X \times [-M, M]$ with $0 < M < \infty$ and $K: X \times X \mapsto \mathbb{R}$ be a kernel. Suppose that Assumptions \ref{ass:target-exist}, \ref{ass:1} and \ref{IW_assumption}  are satisfied with constants $r \in [1/2, 1]$, $R \in (0,\infty)$, $q \in [0,1]$, $W \in (0,\infty)$ and $\sigma \in (0, \infty)$. Suppose further that such that  there exists $N \in \mathbb{N}$ such that $\mathcal{N}(\lambda) \leq N$ for all $\lambda > 0$. 
Let $\delta \in (0,1)$ be an arbitrary constant. Let 
$$
\lambda = c n^{- \frac{1}{2r + q}}, \quad \text{where}\quad c \geq \left( 64 (W+\sigma^2)  N^{1-q}  \log^2\left( 6/\delta \right)  \right)^{1/(1+q)}.
$$
Suppose $n$ is large enough so that $\lambda \leq 1$. Then, with probability greater than $1-\delta$, it holds that 
\begin{align}   \label{IW_generalization_bound-finite-rank}
\left\| f_{\mathbf{z},\lambda}-f_{\mathcal{H}} \right\|_{\rho_X^{\rm te}} & \leq  n^{- \frac{r}{2r + q}}  \left\{ 16  \left(      M   +  \left\| f_\mathcal{H} \right\|_{\mathcal{H}}    \right)  \left( W   +   \sigma N^{(1-q)/2}   \right) c^{-q/2} \log \left( 6/\delta \right)     + c^{r}    R \right\}.  
\end{align}
\end{corollary}

If $q = 0$, the rate becomes  $\mathcal{O}\left(\sqrt{N/n}\right)$ (which can be observed by setting $r = 1/2$ in \eqref{IW_generalization_bound-finite-rank}), which matches the optimal rate for ridge regression without covariate shift  \cite[e.g.,][Theorem 2 (a)]{raskutti2012minimax}.

%The rate $\mathcal{O}\left(\sqrt{N/n}\right)$ is known to be optimal for the ridge regression without  covariate shift.

\subsubsection{Finite Smoothness RKHSs}
As mentioned earlier, Assumption \ref{ass:effective-dim} holds with $s \in (0,1)$ if the eigenvalues of the covariance operator $T$ decay at the rate  
%Assumption 3 holds for the covariance operator $T$ with the eigenvalues of the following order 
\begin{equation}\label{eigen_decay}
    \mu_i(T) = \mathcal{O}\left(i^{-\frac{1}{s}} \right).
\end{equation}
For example, if $X$ is a Euclidean space, $\rho_X^{\rm te}$ is the uniform distribution and the RKHS $\mathcal{H}$ is norm-equivalent to the Sobolev space of order $\eta > d/2$ (e.g., if $K$ is a Mat\'ern kernel with smoothness parameter $\eta - d/2$ \citep[p.86]{williams2006gaussian}), then \eqref{eigen_decay} holds with $s = d / (2\eta)$ (see e.g. \cite{birman1967piecewise}). In this case, the RKHS consists of functions whose $\eta$-times weak derivatives exist and are square-integrable; thus, $\eta$ represents the smoothness of functions in the RKHS. 
We have the following corollary in this case.

%If $X$ is an Euclidean ball in $\mathbb{R}^{d}, \beta>d / 2$ is some integer and $\rhote_X$ is the uniform distribution on $X$, then the Sobolev space $\mathcal{H}:=W^{\beta}(X)$ is an RKHS that satisfies (\ref{eigen_decay}) for $s:=\frac{d}{2 \beta}.$ 

%Let us restate the Theorem \eqref{main_theorem} for the Sobolev hypothesis classes.

\begin{corollary}    
Let $\rho^{\rm te}$ and $\rho^{\rm tr}$ be probability distributions on $X \times [-M, M]$ with $0 < M < \infty$ and $K: X \times X \mapsto \mathbb{R}$ be a kernel. Suppose that Assumptions \ref{ass:target-exist}, \ref{ass:1} and \ref{IW_assumption}  are satisfied with constants $r \in [1/2, 1]$, $R \in (0,\infty)$, $q \in [0,1]$, $W \in (0,\infty)$ and $\sigma \in (0, \infty)$. Suppose further that \eqref{eigen_decay} is satisfied with $s = d/(2\eta)$ with $\eta > d/2$ and let $E_s \in [1, \infty)$ be defined in \eqref{eq:const-Es}.  
Let $\delta \in (0,1)$ be an arbitrary constant. 
Let 
$$
\lambda = c n^{- \frac{2\eta}{  2\eta( 2r + q )  +  d (1-q)   }},   
$$
 where $c > 0$ is such that 
$$
c \geq \left( 64 (W+\sigma^2)  (E_{s})^{2(1-q)}  \log^2\left( 6/\delta \right)  \right)^{\frac{2\eta}{ 2\eta (1+q) + d (1-q) } }.
$$
Suppose $n$ is large enough so that $\lambda \leq 1$. Then, with probability greater than $1-\delta$, it holds that 
\begin{align}    \label{IW_for_Sobolev}
\left\| f_{\mathbf{z},\lambda}-f_{\mathcal{H}} \right\|_{\rho_X^{\rm te}} & \leq n^{- \frac{2\eta r}{  2\eta( 2r + q )  +  d (1-q)   }}  \left\{ 16  \left(      M   +  \left\| f_\mathcal{H} \right\|_\mathcal{H}    \right)  \left( W   +   \sigma  (E_{s})^{1-q}   \right) c^{- \frac{d (1-q)  + 2 \eta q}{4\eta } } \log \left( 6/\delta \right)     + c^{r}    R \right\}.  
\end{align}
\end{corollary}

For the case where $q = 0$ and $r = 1/2$, the rate \eqref{IW_for_Sobolev} becomes $\mathcal{O}(n^{-\frac{\eta}{2\eta+d}})$ and matches the minimax optimal rate for regression in the Sobolev space of order $\eta$, which is also optimal under covariate shift when the IW function is bounded \citep{ma2022optimally}.  

%For the special case when $q=0$ and $r=1/2$, we have the optimal rates for Sobolev spaces $\mathcal{O}((W/n)^{-\frac{\beta}{2\beta+d}})$ under covariate shift with bounded importance weights \citep{ma2022optimally}. Note that the rates reduce to the known optimal rates  \citep{caponnetto2007optimal} in the case of no covariate shift.

Sobolev RKHSs are just one example that satisfies condition \eqref{eigen_decay}.
For example, if the kernel is Gaussian and the support of the test input distribution $\rho_{X}^{\rm te}$ is compact, then the eigenvalues $\mu_i$ decay exponentially fast (see e.g., \citet[Appendix C.2]{bach2002kernel} and references therein), and thus \eqref{eigen_decay} is satisfied for an arbitrarily small $s > 0$. Therefore, the resulting rate \eqref{IW_for_Sobolev} holds for an arbitrarily large $\eta$, and the rate approaches $O(n^{ -\frac{r}{2r + q} } )$. 

%The Sobolev space is not the only example when the condition (\ref{eigen_decay}) is satisfied.  Another example is the RKHS with the Gaussian radial basis function (RBF) reproducing kernel and the distribution $\rhote_X$ satisfying the following condition
%\[
%    \rhote_X\left(\mathbb{R}^{d} \backslash  r_1B\right) \leq r_1^{-\tau}, \quad r_1>0
%\]
%where $B$ is the unit ball in $\mathbb{R}^{d}.$ A bound of the form (\ref{eigen_decay}) can be established \citep[Theorem ~7.34]{steinwart2008support}.   

\section{Convergence of IW-KRR using a Generic Weighting Function} 
\label{sec:arbitrary-weights}

In Section \ref{sec:perfect_weigths}, we have analyzed the convergence properties of the IW-KRR predictor using the IW function $w(x) = d\rho_X^{\rm te}/ d\rho_X^{\rm tr}(x)$ in \eqref{IW_function}. This section extends the analysis to the IW-KRR predictor using a {\em generic} weighting function $v(x)$, which may be different from the IW function $w(x)$. This extension helps us to understand the effect of an ``incorrect'' weight function on the convergence of the IW-KRR predictor. Note that \cite{cortes2008sample} analyze how the use of an estimated weight function influences the accuracy of a learning algorithm. 
%\textcolor{blue}{explain what is the novelty of our work compared with Cortes 08.}

%In Section \ref{sec:perfect_weigths} we have analyzed the generalization properties of the importance weighted KRR and concluded that importance weighting adaptation is an effective strategy whenever the regularizer is properly tuned (Theorem \ref{main_theorem}). It is important to notice these results have been derived under the assumption that the importance weights $w(x)$ can be perfectly estimated, see (\ref{IW_function}). In the following we relax this assumption and we study the performance of the weighted KRR in the case of a weighing function $v(x)$ that does not match the ratio between test and train marginal distributions.  \cite{cortes2008sample} analyze how the use of an estimated weight function influences the accuracy of a learning algorithm. 

\subsection{Generalization Bound}
 We consider a weighting function $v(x) := d\rho'_X/ d \rho_X^{\rm tr}(x)$ that can be expressed as the Radon-Nikodym derivative of {\em some} probability distribution  $\rho'_X(x)$ on $X$ that is absolutely continuous with respect to the training input distribution $\rho_X^{\rm tr}$. For example, if $\rho'_X(x)$ is the test input distribution $\rho_X^{\rm te}(x)$, then the weight function $v(x)$ is the IW function $w(x)$. If $\rho'_X(x)$ is the training input distribution $\rho_X^{\rm tr}(x)$, then the weighting function is uniform, $v(x) = 1$; thus, this is the case of standard KRR without any correction. Different choices of $\rho'_X(x)$ lead to different weighting functions. 

%In the case of ``imperfect'' weights $v(x)$, the optimization problem (\ref{iw_emp_risk}) becomes
With the weighting function $v(x)$, the regularized least squares problem \eqref{iw_emp_risk} becomes
\begin{equation}
\label{imperfect_iw_emp_risk}
f'_{\mathbf{z}, \lambda}:=\argmin _{f \in \mathcal{H}}\left\{\frac{1}{n} \sum_{i=1}^{n}v(x_i)\left(f\left(x_{i}\right)-y_{i}\right)^{2}+\lambda\|f\|_{\mathcal{H}}^{2}\right\}.
\end{equation}
The solution $f'_{\bf z, \lambda}$ is given as \eqref{ERM_IW} or \eqref{eq:IW-KRR-matrix} with the weight function $w(x)$ replaced by $v(x)$.  
In the data-free limit $n \to \infty$, the optimization problem \eqref{imperfect_iw_emp_risk} becomes
\begin{equation} \label{eq:distribution-free-generic}
f'_{\lambda}:=\argmin_{f \in \mathcal{H}}\left\{\left\|f-f_{\rho}\right\|_{\rho'_X}^{2}+\lambda\|f\|_{\mathcal{H}}^{2}\right\},
\end{equation}
where $\left\| \cdot \right\|_{\rho'_X}$ is the norm of $L^2(X,\rho'_X)$ defined by the input distribution $\rho'_X(x)$.

Similar to the projection $f_\mathcal{H}$ of the regression function $f_\rho$ defined with respect to the test input distribution $\rho_X^{\rm te}$ in \eqref{eq:target-func}, we define the projection $f'_{\mathcal{H}}$ of $f_\rho$ with respect to the distribution $\rho'_X$: 
\begin{equation} \label{eq:target-func-gen}
f'_\mathcal{H} := \arg\min_{f \in \mathcal{H}} \left\|  f - f_\rho \right\|_{ \rho'_X }^2 = \arg\min_{f \in \mathcal{H}} \int_X \left( f(x) - f_\rho(x) \right)^2 d\rho'_X(x),
\end{equation}
assuming its existence and uniqueness. 
We also define the covariance operator $T'$ and integral operator $L'$ with respect to $\rho'_X$:
$$
T' f := \int k(\cdot,x) f(x)d\rho'_X(x) \   \text{for } f \in \mathcal{H}, \quad L' f := \int k(\cdot,x) f(x)d\rho'_X(x) \ \text{for }  f \in L^2(X, \rho'_X) .
$$ 
Then the solution $f'_\lambda$ in the data-free limit \eqref{eq:target-func-gen} is given as 
\begin{equation}\label{imp_weights}
f'_{\lambda}=\left(T'+\lambda I\right)^{-1} L' f_\rho = \left(T'+\lambda I\right)^{-1} T' f'_\mathcal{H}.
\end{equation}

%where now $v(x)=\frac{d\rho'_X(x)}{d\rhotr_X(x)}$ for some measure $\rho'_X \ll \rhotr_X.$
%In the data free scenario, the optimization problem (\ref{imperfect_iw_emp_risk}) is equivalent to
%\begin{equation*}
%\label{imperfect_iw_emp_risk_data_free}
%f'_{\lambda}:=\argmin_{f \in \mathcal{H}}\left\{\left\|f-f_{\rho}\right\|_{\rho'_X}^{2}+\lambda\|f\|_{\mathcal{H}}^{2}\right\},
%\end{equation*}
%whose solution can be expressed as
%\begin{equation}\label{imp_weights}
%f'_{\lambda}=\left(T'+\lambda I\right)^{-1} L' f'_{\mathcal{H}}
%\end{equation}
%where $T' = T_{\rho'_X},$ $L' = L_{\rho'_X}$ and $f'_{\mathcal{H}}$ being the projection of the regression function $f_{\rho}$ onto the closure of $\mathcal{H}$ in $L^2(X,\rho'_X):$
%\[
%\left\|f_{\rho}-f'_{\mathcal{H}}\right\|_{\rho'_X}=\inf _{f \in \mathcal{H}}\|f_{\rho}-f\|_{\rho'_X}.
%\]

The following assumptions, about the projection $f'_{\mathcal{H}}$, the weighting function $v(x)$ and the effective dimension, mirror Assumptions \ref{ass:target-exist}, \ref{ass:1}, \ref{IW_assumption} and \ref{ass:effective-dim} of Section \ref{sec:perfect_weigths}.

%In order to provide guarantees for the finite data scenario we need a set of assumptions on $v$ and $\rho'_X$ that are similar to Assumptions 2 and 3. 
\begin{assumption} \label{ass:target-exist-gen}
The projection $f'_\mathcal{H} \in \mathcal{H}$ in \eqref{eq:target-func-gen} exists and is unique.
Moreover, there exist  $1/2 \leq r \leq 1$ and   $g \in L^2(X, \rho_X^{\rm te})$ with $\left\| g \right\|_{\rhote_X} \leq R'$ for some $R' > 0$ such that  $f'_\mathcal{H} = (L')^{r'} g$ for the target function $f'_\mathcal{H}$ in \eqref{eq:target-func-gen}. 
%Moreover, there exist  $1/2 \leq r \leq 1$ and   $g \in L^2(X, \rho_X^{\rm te})$ with $\left\| g \right\|_{\rhote_X} \leq R'$ for some $R' > 0$ such that  $f'_\mathcal{H} = L^{r'} g$ for the target function $f'_\mathcal{H}$ in \eqref{eq:target-func-gen}. 
\end{assumption}

\begin{assumption}\label{ass:4}
Let $v = d\rho'_X/d\rhotr_X$ be a weighting function. There exist constants $q' \in [0,1]$, $V > 0$ and $\gamma > 0$ such that, for all $m \in \mathbb{N}$ with $m \geq 2$, it holds that  
\begin{equation}\label{condition_on_iw_imperfect}
    \left(\int_X v(x)^{\frac{m-1}{q'}} d\rho'_X(x)\right)^{q'} \leq \frac{1}{2}m!V^{m-2}\gamma^2,
\end{equation}
where the left-hand side for $q = 0$ is defined  as $\left\| v^{m-1} \right\|_{\infty, \rhote_X}$, the essential supremum of $v^{m-1}$ with respect to $\rho'_X$.

\end{assumption}

%\begin{assumption}
%\label{ass:4}
%For some $q' \in [0,1]$ there exist positive constants $V$ and $\gamma$ depending on $q$ such that for all $m \geq 2$
%\begin{equation}
%    \left(\int_X v(x)^{\frac{m-1}{q'}} d\rho'_X(x)\right)^{q'} \leq \frac{1}{2}m!V^{m-2}\gamma^2.
%\end{equation}
%\end{assumption}

\begin{assumption}   \label{ass:5}
There exists a constant $s' \in [0,1]$ such that
\begin{equation}  \label{eq:const-Es_imperfect}
  E'_{s'}:= \max \left(1, \sup _{\lambda \in(0,1]} \sqrt{\mathcal{N}'(\lambda) \lambda^{s'}} \right) < \infty, \quad \text{where}\ \ \mathcal{N}'(\lambda) :=  \operatorname{Tr}\left( T'(T'+\lambda)^{-1} \right).
\end{equation}
%where $\mathcal{N}(\lambda)=\operatorname{Tr}\left[T(T+\lambda)^{-1} \right].$
\end{assumption}

%\begin{assumption}
%\label{ass:5}
%For some $s' \in (0,1]$ we assume that
%\begin{equation} 
%  E'_{s'}:=1 \vee \sup _{\lambda \in(0,1]} \sqrt{\mathcal{N}'(\lambda) \lambda^{s'}} < \infty
%\end{equation}
%where $\mathcal{N}'(\lambda)=\operatorname{Tr}\left[T'(T'+\lambda)^{-1} \right].$
%\end{assumption}

We also make the following assumption. 
\begin{assumption}\label{ass:connection_te_'}
For  $T = \int T_x d\rho_X^{\rm te}(x)$ and $T' = \int T_x d\rho'_X(x)$, there exists a constant $G > 0$ such that
\begin{equation}
\|T(T'+\lambda)^{-1}\| \leq G < \infty \quad \text{for all } \lambda>0.    
\end{equation}
\end{assumption}
%For $\rho'_X = \rhote_X$ the condition above is satisfied with $G = 1.$

In Assumption \ref{ass:connection_te_'}, the constant $G$ can be interpreted as quantifying the discrepancy between the two distributions $\rhote_X$ and $\rho'_X$. In particular, it is satisfied when $\rho'_X = \rhote_X$ with $G = 1$. It is also satisfied if the Radon-Nikodym derivative $d\rhote_X / d\rho'_X$ is bounded, with $G = \left\| d\rhote_X / d\rho'_X \right\|_{\infty}$; see Proposition \ref{prop:asumption-G-radon} in Appendix \ref{sec:appendix-auxiliary-results}.

\begin{figure}  
\centering 
  \begin{subfigure}[b]{0.4\linewidth}
    \begin{tikzpicture}[x=0.75pt,y=0.75pt,yscale=-0.8,xscale=0.8]
%uncomment if require: \path (0,301); %set diagram left start at 0, and has height of 301

%Curve Lines [id:da20135154893253793] 
\draw    (97,201.23) .. controls (158.75,105.08) and (397.62,132.17) .. (395.99,247.27) ;
%Straight Lines [id:da182767063720356] 
\draw [color={rgb, 255:red, 0; green, 0; blue, 0 }  ,draw opacity=1 ]   (224,47) -- (225.37,141.64) ;
\draw [shift={(225.37,141.64)}, rotate = 89.17] [color={rgb, 255:red, 0; green, 0; blue, 0 }  ,draw opacity=1 ][fill={rgb, 255:red, 0; green, 0; blue, 0 }  ,fill opacity=1 ][line width=0.75]      (0, 0) circle [x radius= 3.35, y radius= 3.35]   ;
%Shape: Boxed Line [id:dp4994461995550985] 
\draw [color={rgb, 255:red, 0; green, 0; blue, 0 }  ,draw opacity=1 ][fill={rgb, 255:red, 245; green, 166; blue, 35 }  ,fill opacity=1 ]   (224,47) -- (304,153) ;
\draw [shift={(304,153)}, rotate = 52.96] [color={rgb, 255:red, 0; green, 0; blue, 0 }  ,draw opacity=1 ][fill={rgb, 255:red, 0; green, 0; blue, 0 }  ,fill opacity=1 ][line width=0.75]      (0, 0) circle [x radius= 3.35, y radius= 3.35]   ;
\draw [shift={(224,47)}, rotate = 52.96] [color={rgb, 255:red, 0; green, 0; blue, 0 }  ,draw opacity=1 ][fill={rgb, 255:red, 0; green, 0; blue, 0 }  ,fill opacity=1 ][line width=0.75]      (0, 0) circle [x radius= 3.35, y radius= 3.35]   ;
%Shape: Circle [id:dp9215261614798296] 
\draw  [dash pattern={on 0.84pt off 2.51pt}] (157.4,184.22) .. controls (157.81,156.35) and (180.74,134.08) .. (208.61,134.49) .. controls (236.49,134.9) and (258.75,157.82) .. (258.34,185.7) .. controls (257.93,213.57) and (235.01,235.84) .. (207.13,235.43) .. controls (179.26,235.02) and (156.99,212.09) .. (157.4,184.22) -- cycle ;
%Straight Lines [id:da5741709502435952] 
\draw    (207.87,184.96) ;
\draw [shift={(207.87,184.96)}, rotate = 0] [color={rgb, 255:red, 0; green, 0; blue, 0 }  ][fill={rgb, 255:red, 0; green, 0; blue, 0 }  ][line width=0.75]      (0, 0) circle [x radius= 3.35, y radius= 3.35]   ;
%Straight Lines [id:da6587143920846787] 
\draw    (215,225) ;
\draw [shift={(215,225)}, rotate = 0] [color={rgb, 255:red, 0; green, 0; blue, 0 }  ][fill={rgb, 255:red, 0; green, 0; blue, 0 }  ][line width=0.75]      (0, 0) circle [x radius= 3.35, y radius= 3.35]   ;
%Straight Lines [id:da6095071141137556] 
\draw    (297.87,193.96) ;
\draw [shift={(297.87,193.96)}, rotate = 0] [color={rgb, 255:red, 0; green, 0; blue, 0 }  ][fill={rgb, 255:red, 0; green, 0; blue, 0 }  ][line width=0.75]      (0, 0) circle [x radius= 3.35, y radius= 3.35]   ;
%Shape: Circle [id:dp9344865173416628] 
\draw  [dash pattern={on 0.84pt off 2.51pt}] (268.54,193.53) .. controls (268.77,177.33) and (282.1,164.38) .. (298.3,164.62) .. controls (314.51,164.86) and (327.45,178.19) .. (327.21,194.39) .. controls (326.97,210.59) and (313.65,223.53) .. (297.44,223.3) .. controls (281.24,223.06) and (268.3,209.73) .. (268.54,193.53) -- cycle ;
%Straight Lines [id:da14596332051278615] 
\draw    (308.87,212.96) ;
\draw [shift={(308.87,212.96)}, rotate = 0] [color={rgb, 255:red, 0; green, 0; blue, 0 }  ][fill={rgb, 255:red, 0; green, 0; blue, 0 }  ][line width=0.75]      (0, 0) circle [x radius= 3.35, y radius= 3.35]   ;

% Text Node
\draw (117.49,188.02) node [anchor=north west][inner sep=0.75pt]   [align=left] {$\displaystyle \mathcal{H}$};
% Text Node
\draw (196.75,44.01) node [anchor=north west][inner sep=0.75pt]    {$f_{\rho }$};
% Text Node
\draw (206.12,144.71) node [anchor=north west][inner sep=0.75pt]    {$f_{\mathcal{H}}$};
% Text Node
\draw (272.74,146.72) node [anchor=north west][inner sep=0.75pt]    {$f_{\mathcal{H}}^{'}$};
% Text Node
\draw (213,233.4) node [anchor=north west][inner sep=0.75pt]    {$f_{\mathbf{z},\lambda}$};
% Text Node
\draw (199.23,207.85) node [anchor=south] [inner sep=0.75pt]    {$f_{\lambda }$};
% Text Node
\draw (312.23,200.85) node [anchor=south] [inner sep=0.75pt]    {$f_{\lambda }^{'}$};
% Text Node
\draw (310.87,216.36) node [anchor=north west][inner sep=0.75pt]    {$f_{\mathbf{z},\lambda}^{'}$};
\end{tikzpicture}
    \caption{Misspecified} \label{fig:Misspecified}  
  \end{subfigure}
  \hspace{1cm}
\begin{subfigure}[b]{0.4\linewidth}
  \begin{tikzpicture}[x=0.75pt,y=0.75pt,yscale=-0.8,xscale=0.8]
%uncomment if require: \path (0,301); %set diagram left start at 0, and has height of 301

%Curve Lines [id:da20135154893253793] 
\draw    (97,201.23) .. controls (158.75,105.08) and (397.62,132.17) .. (395.99,247.27) ;
%Shape: Circle [id:dp9215261614798296] 
\draw  [dash pattern={on 0.84pt off 2.51pt}] (182.4,190.22) .. controls (182.81,162.35) and (205.74,140.08) .. (233.61,140.49) .. controls (261.49,140.9) and (283.75,163.82) .. (283.34,191.7) .. controls (282.93,219.57) and (260.01,241.84) .. (232.13,241.43) .. controls (204.26,241.02) and (181.99,218.09) .. (182.4,190.22) -- cycle ;
%Straight Lines [id:da5741709502435952] 
\draw    (232.87,190.96) ;
\draw [shift={(232.87,190.96)}, rotate = 0] [color={rgb, 255:red, 0; green, 0; blue, 0 }  ][fill={rgb, 255:red, 0; green, 0; blue, 0 }  ][line width=0.75]      (0, 0) circle [x radius= 3.35, y radius= 3.35]   ;
%Straight Lines [id:da6587143920846787] 
\draw    (215,225) ;
\draw [shift={(215,225)}, rotate = 0] [color={rgb, 255:red, 0; green, 0; blue, 0 }  ][fill={rgb, 255:red, 0; green, 0; blue, 0 }  ][line width=0.75]      (0, 0) circle [x radius= 3.35, y radius= 3.35]   ;
%Straight Lines [id:da6095071141137556] 
\draw    (272.87,187.96) ;
\draw [shift={(272.87,187.96)}, rotate = 0] [color={rgb, 255:red, 0; green, 0; blue, 0 }  ][fill={rgb, 255:red, 0; green, 0; blue, 0 }  ][line width=0.75]      (0, 0) circle [x radius= 3.35, y radius= 3.35]   ;
%Shape: Circle [id:dp9344865173416628] 
\draw  [dash pattern={on 0.84pt off 2.51pt}] (236.11,187.42) .. controls (236.4,167.11) and (253.11,150.89) .. (273.41,151.19) .. controls (293.72,151.49) and (309.94,168.19) .. (309.64,188.5) .. controls (309.34,208.8) and (292.64,225.02) .. (272.33,224.73) .. controls (252.03,224.43) and (235.81,207.73) .. (236.11,187.42) -- cycle ;
%Straight Lines [id:da14596332051278615] 
\draw    (284.87,169.96) ;
\draw [shift={(284.87,169.96)}, rotate = 0] [color={rgb, 255:red, 0; green, 0; blue, 0 }  ][fill={rgb, 255:red, 0; green, 0; blue, 0 }  ][line width=0.75]      (0, 0) circle [x radius= 3.35, y radius= 3.35]   ;
%Straight Lines [id:da5700750095528977] 
\draw    (244.87,166.96) ;
\draw [shift={(244.87,166.96)}, rotate = 0] [color={rgb, 255:red, 0; green, 0; blue, 0 }  ][fill={rgb, 255:red, 0; green, 0; blue, 0 }  ][line width=0.75]      (0, 0) circle [x radius= 3.35, y radius= 3.35]   ;

% Text Node
\draw (117.49,188.02) node [anchor=north west][inner sep=0.75pt]   [align=left] {$\displaystyle \mathcal{H}$};
% Text Node
\draw (222.75,144.01) node [anchor=north west][inner sep=0.75pt]    {$f_{\rho }$};
% Text Node
\draw (220,220.4) node [anchor=north west][inner sep=0.75pt]    {$f_{\mathbf{z},\lambda}$};
% Text Node
\draw (221.23,201.85) node [anchor=south] [inner sep=0.75pt]    {$f_{\lambda}$};
% Text Node
\draw (259.23,200.85) node [anchor=south] [inner sep=0.75pt]    {$f_{\lambda }^{'}$};
% Text Node
\draw (295.3,158.02) node [anchor=north west][inner sep=0.75pt]    {$f_{\mathbf{z},\lambda}^{'}$};

\end{tikzpicture}
\caption{Well-specified} \label{fig:well-specified}  
\end{subfigure}
\caption{
Illustrations of the misspecified and well-specified cases and the difference in the effects of using the correct IW function $w(x)$ and a generic weighting function $v(x)$. (a) The misspecified case where the regression function $f_\rho$ does not belong to the RKHS $\mathcal{H}$ and where $f_\rho \not= f_\mathcal{H} \not= f'_\mathcal{H}$.  The IW-KRR predictor $f_{z, \lambda}$ using the correct IW function lies near its data-free limit $f_\lambda$, which approximates the projection $f_\mathcal{H}$ of $f_\rho$ under $\rho_X^{\rm te}$.  On the other hand, the IW-KRR predictor $f'_{\lambda}$ using the ``imperfect'' weight function $v(x) = d\rho'_X(x) / \rho_X^{\rm tr}(x)$ is close to its data-free limit $f'_\lambda$, which approximates the projection $f'_\mathcal{H}$ of $f_\rho$ under  the distribution $\rho'_X$ associated with $v(x)$. The $f_{ {\bf z}, \lambda}$ has a smaller bias in estimating the target function $f_\mathcal{H}$ but can have a higher variance (represented by the diameter of the dotted circle) than $f'_{z, \lambda}$, if the IW function $w(x)$ has larger moment constants $W$ and $\sigma$ than those of $v(x)$, i.e., $V$ and $\gamma$.   (b) The well-specified case where $f_\rho$ belongs to the RKHS $\mathcal{H}$ and thus $f_\rho = f_\mathcal{H} = f'_{\mathcal{H}}$. In this case, $f'_{z, \lambda}$ may have a smaller error in estimating $f_\rho$ if the weighting function $v(x)$ makes  the variance of $f'_{z, \lambda}$ smaller than $f_{ {\bf z}, \lambda }$.  
%Difference between misspecified and well-specified scenarios. (a) IW adaptation could significantly reduce the approximation error with a price of high variance. (b) On the contrary, the performance of the model can be significantly improved by using the weighting function with better control of large values.
}
\label{fig:projections}
\end{figure}
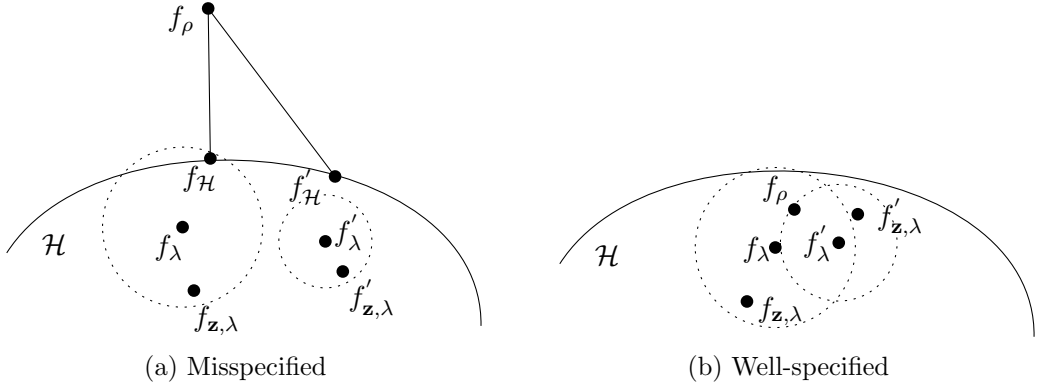

We are now ready to state our result on the convergence of the IW-KRR predictor using the generic weighting function $v(x)$. The proof is given in Appendix \ref{sec:proof-main-gen}.

%Equipped with the necessary assumptions, we are now ready to state the main result for an arbitrarily weighted KRR.

\begin{theorem}   \label{main_imperfect}
Let $\rho^{\rm te}$ and $\rho^{\rm tr}$ be  distributions on $X \times [-M, M]$, where $M > 0$ is a constant, $\rho'_X$ be a distribution on $X$, and $K: X \times X \mapsto \mathbb{R}$ be a kernel.  Suppose $\rho^{\rm te}$, $\rho^{\rm tr}$, $\rho'_X$ and $K$ satisfy  Assumptions \ref{ass:target-exist},  \ref{ass:target-exist-gen}, \ref{ass:4}, \ref{ass:5}, \ref{ass:connection_te_'} with constants $r' \in [1/2, 1]$, $R' \in (0, \infty)$,  $q' \in [0,1]$, $V \in (0, \infty)$, $\gamma \in (0, \infty)$,  $s' \in [0,1]$, $E'_{s'} \in [1, \infty)$ and $G \in [1, \infty)$.  
Let $\delta \in (0,1)$ be an arbitrary constant. Let 
\begin{equation*} 
\lambda = c n^{- \beta},
\end{equation*}
where $\beta > 0$ is defined by 
$$
\beta := \frac{1}{2r' +  s' (1-q')  + q'} = \frac{1}{2r' + A}, \quad \text{where} \quad \quad A := s' (1-q')  + q',
$$
and $c > 0$ is such that 
$$
c \geq \left( 64 (V+\gamma^2)  (E'_{s'})^{2(1-q')}  \log^2\left( 6/\delta \right)  \right)^{1/(1+A)}.
$$
Suppose $n$ is large enough so that $\lambda \leq 1$. Then, with probability greater than $1-\delta$, it holds that 
\begin{align}  \label{generalization_bound_imperfect}
\left\| f_{\mathbf{z},\lambda}'-f_{\mathcal{H}} \right\|_{\rho_X^{\rm te}} & \leq n^{- r' \beta} G^{1/2} \left\{ 16  \left(      M   +  \| f'_\mathcal{H} \|_{\mathcal{H}}   \right)  \left( V   +   \gamma  (E'_{s'})^{1-q'}   \right) c^{-A/2} \log \left( 6/\delta \right)     + c^{r'}    R' \right\} \\
&\quad + \left\| f'_{\mathcal{H}}-f_{\mathcal{H}} \right\|_{\rho_X^{\rm te}} . \nonumber
\end{align}
\end{theorem}

\begin{remark}
For the particular case when $\rho'_X = \rhote_X$, where  $G = 1$ and $ f'_{\mathcal{H}} =f_{\mathcal{H}} $, we recover Theorem \ref{main_theorem}.
%recovering the rate of \eqref{IW_generalization_bound}.
%The scaling factor $\mu_1$ is well-defined as we have $\mu < 1$ \citep[Proposition 8]{rudi2017generalization}.  For the particular case when $\rho'_X = \rhote_X$, both $\mu_1$ and $\|f'_{\mathcal{H}}-f_{\mathcal{H}}\|_{\rhote_X}$ vanish, and thus the rate \eqref{generalization_bound_imperfect} recovers that of \eqref{IW_generalization_bound}. 
%The scaling factor $\mu_1$ is well defined as $\mu<1$. For the special case when $\rho'_X = \rhote_X,$ both $\mu_1$ and $\|f'_{\mathcal{H}}-f_{\mathcal{H}}\|_{\rhote_X}$ vanish, recovering the learning rates of Theorem \ref{main_theorem}.
\end{remark}

%{\color{blue} Generalization bounds for arbitrarily weighted algorithms were also considered in \cite{cortes2010} for the general loss function and the function class with finite pseudodimension. While our setting differs from the one of \cite{cortes2010} in several technical aspects, we believe that the core difference between the two results lies in the treatment of the target functions $f'_{\mathcal{H}}$ and $f_{\mathcal{H}}.$ Simplifying, selection of the successful weighting function $v$ according to \cite{cortes2010} is based on the trade-off between $\|w-v\|_{1,\rhotr_X}$ and $\|v\|_{\rhotr_X}.$ In our analysis similar trade-off applies only when the model class is well-specified, with a somehow similar trade-off between $\mu_1$ and moment constants $V$ and $\gamma$. In the misspecified case, however, the additional bias term $\|f'_{\mathcal{H}}-f_{\mathcal{H}}\|_{\rhote_X}$  should also be taken into account as our error analysis suggests.}

Theorem \ref{main_imperfect} highlights the effects of using an ``imperfect'' weighting function on the convergence of the IW-KRR predictor, as summarized below. 

%Theorem \ref{main_imperfect} highlights some important aspects associated to ``imperfectly'' re-weighting the risk and its effects on the learning rates. 
%As it is clear from the upper bound, in order to choose reasonable re-weighting procedures several important factors should be considered.

\begin{itemize}
    \item A good choice of the weighting function depends heavily on the approximation properties of the RKHS $\mathcal{H}$. Consider the misspecified case where $f_\rho \in \mathcal{H}$ and $f_\rho \not= f_\mathcal{H} \not= f'_\mathcal{H}$ (see Figure \ref{fig:Misspecified} for an illustration). The IW-KRR predictor $f_{ {\bf z}, \lambda }$ using the correct IW function $w(x)$ lies near its data-free limit $f_\lambda$, which is a good approximation of the projection $f_\mathcal{H}$ of the regression function $f_\rho$ under the test input distribution $\rho_X^{\rm te}$. On the other hand, the IW-KRR predictor $f'_{ {\rm z}, \lambda }$ using the ``imperfect'' weight function $v(x)$ lies near its data-free limit $f'_\lambda$ that approximates the projection $f'_\mathcal{H} (\not= f_\mathcal{H})$ of $f_\rho$ under the input distribution $\rho'_X$ corresponding to $v(x) = d\rho'_X(x) / d\rho_X^{\rm te}(x)$. Therefore, if the model class $\mathcal{H}$ is misspecified and $f_\rho \not= f_\mathcal{H} \not= f'_\mathcal{H}$, the use of the ``imperfect'' weight function $v(x)$ leads to the estimation of the ``wrong'' projection $f'_{\mathcal{H}}$. This observation agrees with \citet{shimodaira2000} for weighted maximum likelihood estimation in parametric models. 
    
    \item The situation is less dramatic when the model class $\mathcal{H}$ is well-specified in that $f_\rho \in \mathcal{H}$ so that $f_\rho = f_\mathcal{H} = f'_\mathcal{H}$ (Figure \ref{fig:well-specified}). With an appropriate regularization constant $\lambda$, the data-free limits $f_\lambda$ with $w(x)$ and $f'_\lambda$ with $v(x)$ are both close to the regression function $f_\rho$. Therefore it is preferable to select a weighting function $v(x)$ that makes the variance of the predictor $f'_{ {\rm z}, \lambda }$ small;  $v(x)$ does not need to match the correct IW function $w(x)$. Therefore the use of the uniform weighting function $v(x) = 1$ is justified in the well-specified case. This observation agrees with a convergence result of Ma et al. (Theorem 1), which shows the minimax optimality of the KRR predictor with uniform weighting in the well-specified case (assuming that the IW function $w(x)$ is bounded). 
    
    \item The scaling factor $G$ measures the distortion between the testing distribution $\rho_X^{\rm te}$ and the distribution $\rho'_X$ associated with the weighting function $v(x) = d\rho'_X(x) / d\rho_X^{\rm tr}(x)$. However, setting $\rho_X' = \rho_X^{\rm te}$, which leads to $v(x) = w (x)$, does not necessarily improve the generalization bound, because this may make the constant $V = W$ in Assumptions \ref{ass:4} large. 
\end{itemize}

\citet[Theorem 4]{cortes2010} provide a generalization bound for learning with a generic weighting function, a generic loss function, and a hypothesis class with a finite pseudo-dimension. While our setting differs from \citet{cortes2010} in several technical aspects, the core difference lies in consequence regarding a good choice of a weighting function. Briefly, \citet{cortes2010} argue that a good weighting function $v(x)$ should balance the tradeoff between the approximation error $ \int | w(x) - v(x) | d\rho_X^{\rm tr} (x)$ and the second moment $\int v^2(x) d\rho_X^{\rm tr}(x)$. On the other hand, Theorem \ref{main_imperfect} suggests that a similar tradeoff appears between $G$ and moment constants $V$ and $\gamma$ in the well-specified case where $f_\rho \in \mathcal{H}$ so that $f_\mathcal{H} = f'_\mathcal{H} =  f_\rho$. However, in the misspecified case where $f_\rho \not \not= f_\mathcal{H} \not= f'_\mathcal{H}$, the bias term $\|f'_{\mathcal{H}}-f_{\mathcal{H}}\|_{\rhote_X}$  in \eqref{generalization_bound_imperfect} remains, and thus our result suggests one should take into account this bias when selecting a weight function $v(x)$.

\subsection{Convergence Rates for Specific Weighting Functions}

Below we consider specific weighting functions commonly used in practice.

%Below we consider some weighting procedures which are commonly used in practice.

\subsubsection{Uniform Weights} 
The uniform weighting function $v(x) = 1$ is where $\rho'_X = \rho_X^{\rm tr}$, and yields the standard KRR predictor without any weighting correction.  In this case, Assumption \ref{ass:4} holds with $q = 0$, $V = \gamma = 1$.  Moreover, if the RKHS $\mathcal{H}$ contains $f_\rho$ or if $\mathcal{H}$ is dense in $L^2(X, \rho_X^{\rm te})$, then the second term in \eqref{generalization_bound_imperfect} vanishes, as we have $f_\rho = f_\mathcal{H} = f'_\mathcal{H}$ in either case.  Therefore Theorem \ref{main_imperfect} yields the following generalization bound for unweighted KRR under covariate shift.

%The uniform weights scenario corresponds to the case where $\rho'_X = \rhotr_X.$ This choice corresponds to the solution of the following optimization problem
%\[
%\rhotr_X = \argmin _{\rho': \rho'_X \ll \rhotr_X} \left\|\frac{d\rho'_X}{d\rhotr_X}\right\|_{\infty},
%\]
%and it yields the smallest constant $V$ in the generalization bound (\ref{generalization_bound_imperfect}). In other words, the uniform weights leads to the smallest generalization bound whenever the difference between the projections of the regression function $f_{\rho}$ on $\mathcal{H}$ w.r.t the training and testing input measures is small. This is the case in the well-specified scenario, when $f_{\rho} \in \mathcal{H}$, or when the kernel $K$ is universal, meaning that the corresponding Hilbert space $\mathcal{H}$ is dense in $L^2(X,\rhote_X).$ In both of these cases $f_{\rho} = f_{\mathcal{H}} = f'_{\mathcal{H}}$ and the term $\|f'_{\mathcal{H}}-f_{\mathcal{H}}\|_{\rhote_X}$ in the bound (\ref{generalization_bound_imperfect}) vanishes. The above discussion is formalized in the following corollary.

\begin{corollary} 
Suppose that the conditions in Theorem \ref{main_imperfect} are satisfied with $\rho'_X = \rho_X^{\rm tr}$. Moreover, assume either that $f_\rho \in \mathcal{H}$ or that $\mathcal{H}$ is dense in $L^2(X,\rhote_X)$. 
%Let $\rho^{\rm te}$ and $\rho^{\rm tr}$ be  distributions on $X \times [-M, M]$, where $M > 0$ is a constant, $\rho'_X$ be a distribution on $X$, and $K: X \times X \mapsto \mathbb{R}$ be a kernel.  Suppose $\rho^{\rm te}$, $\rho^{\rm tr}$, $\rho'_X$ and $K$ satisfy  Assumptions \ref{ass:target-exist},  \ref{ass:target-exist-gen}, \ref{ass:4}, \ref{ass:5}, \ref{ass:connection_te_'} with constants $r' \in [1/2, 1]$, $R' \in (0, \infty)$,  $q' \in [0,1]$, $V \in (0, \infty)$, $\gamma \in (0, \infty)$,  $s' \in [0,1]$, $E'_{s'} \in [1, \infty)$ and $G \in [1, \infty)$.  
Let $\delta \in (0,1)$ be an arbitrary constant. Let 
\begin{equation*} 
\lambda = c n^{- \frac{1}{2r' +  s'} }, \quad \text{where} \quad c \geq \left( 128  (E'_{s'})^{2}  \log^2\left( 6/\delta \right)  \right)^{1/(1+s')}.
\end{equation*}
Suppose $n$ is large enough so that $\lambda \leq 1$. Then, with probability greater than $1-\delta$, it holds that 
\begin{align}   \label{generalization_bound_imperfect_uniform}
\left\| f_{\mathbf{z},\lambda}'-f_{\mathcal{H}} \right\|_{\rho_X^{\rm te}} & \leq n^{- \frac{r'}{2r' +  s'} } G^{1/2} \left\{ 16  \left(      M   +  R'   \right)  \left( 1   +   E'_{s'}    \right) c^{-s'/2} \log \left( 6/\delta \right)     + c^{r'}    R' \right\} .
\end{align}
\end{corollary}

\begin{comment}

\begin{corollary}
Suppose that the conditions in Theorem \ref{main_imperfect} are satisfied with $\rho'_X = \rho_X^{\rm tr}$. Moreover, assume either that $f_\rho \in \mathcal{H}$ or that $\mathcal{H}$ is dense in $L^2(X,\rhote_X)$. 
\textcolor{red}{MK: revise the following.}
 %Let us assume that the conditions of Theorem \ref{main_imperfect} are satisfied with $\rho'_X = \rhotr_X.$ Furthermore, assume that either $f_{\rho} \in \mathcal{H}$ or $\mathcal{H}$ is dense in $L^2(X,\rhote_X)$. 
 Let $n$ and $\lambda$ satisfy the constraints $\lambda \leq \|T'\|$ and 
 \begin{equation*}
    \lambda = \left(\frac{8E'_{s'}\log\left(\frac{6}{\delta} \right)}{\sqrt{n G^s}G^{r-1/2}}\right)^{\frac{2}{2r'+s'}}
\end{equation*}
for $\delta \in (0,1)$ and $s'>0.$ Then, for $r \geq 0.5,$ with probability greater than $1-\delta,$ it holds
\begin{equation}\label{generalization_bound_imperfect_uniform}
\|f'_{\mathbf{z},\lambda}-f_{\mathcal{H}}\|_{\rhote_X}\leq 8 \left(M+R'\right) \left(\frac{8E'_{s'} \sqrt{G} \log\left(\frac{6}{\delta} \right)}{\sqrt{n}}\right)^{\frac{2r'}{2r'+s'}}.
\end{equation}
\end{corollary}

\end{comment}

\citet[Theorem 1]{ma2022optimally}  provide a similar convergence result for unweighted KRR under covariate shift when the IW function is bounded and when $f_\rho \in \mathcal{H}$. Our bound \eqref{generalization_bound_imperfect_uniform} with $r = 1/2$ corresponds to their result.

\subsubsection{Clipped IW function}

Another popular weighting function is the one given by clipping the IW function $w(x)$ at a specified threshold $D > 0$. Namely, the clipped IW function with a clipping threshold $D$ is given by
%Another popular weighting function is the one obtained after clipping the importance weights that exceed a maximum value $D$.Namely,
\begin{equation*}
w_D(x) := \min \{w(x),D\}.
\end{equation*}
We denote the IW-KRR predictor using the clipped IW function $w_D(x)$ by $f_{ {\bf z}, \lambda }^D$. The following theorem provides a generalization bound of $f_{ {\bf z}, \lambda }^D$. The proof can be found in Appendix \ref{trucated_KRR_proof}.  
%We denote the solution of the weighted KRR based on clipped importance weights by $f^D_{\mathbf{z}, \lambda}$ and provide the following learning guarantee.
\begin{comment}
\[
f^D_{\mathbf{z}, \lambda}:=\argmin _{f \in \mathcal{H}}\left\{\frac{1}{n} \sum_{i=1}^{n}w_D(x_i)\left(f\left(x_{i}\right)-y_{i}\right)^{2}+\lambda\|f\|_{\mathcal{H}}^{2}\right\}.
\]
\end{comment}

\begin{theorem} \label{theo:clipped-KRR-re}
Let $\rho^{\rm te}$ and $\rho^{\rm tr}$ be probability distributions on $X \times [-M, M]$, where $M > 0$ is a constant, and $K: X \times X \mapsto \mathbb{R}$ be a kernel.  Suppose $\rho^{\rm te}$, $\rho^{\rm tr}$ and $K$ satisfy  Assumptions \ref{ass:target-exist}, \ref{ass:1}, \ref{IW_assumption} and \ref{ass:effective-dim} with constants $r \in [1/2, 1]$, $R \in (0, \infty)$,  $q \in (0,1]$, $W \in (0, \infty)$, $\sigma \in (0, \infty)$,  $s \in [0,1]$ and $E_s \in [1, \infty)$. 
%Suppose that Assumptions \ref{IW_assumption} and \ref{ass:effective-dim} hold with constants $q \in [0,1]$, $W \in (0,\infty)$, $\sigma \in (0, \infty)$, $s \in [0,1]$ and $E_s \in (0, \infty)$. 
Let $m \in \mathbb{N}$ with $m \geq 2$ and $\epsilon > 0$ be arbitrary constants.
%such that $(m-1)  (1+s) - 4qr >  0$. 
Define $\lambda > 0$ and  $D > 0$ by
\begin{align} \label{eq:lambda-D-2915}
& \lambda := c_1 n^{- \beta}, \quad  D := c_2 n^{ \tau },  \\
& \text{where} \quad \beta := \frac{m-1}{(s+2r)(m-1) + 4qr + \epsilon}, \quad \tau := \frac{4qr}{(s+2r)(m-1) + 4qr + \epsilon}, \nonumber
\end{align}
and $c_1, c_2 > 0$  are constants such that
\begin{equation} \label{eq:c2-c1-cond-trun}
    c_2   \geq  (2^{2q-1} E_s^{2q} m! W^{m - 2} \sigma^2 )^{ \frac{1}{m-1} } c_1^{- \frac{(1+s)q}{m-1} }.  
\end{equation}
 Let $0 < \delta < 1$ be arbitrary. Suppose that sample size $n$ is large enough so that $\lambda \leq 1$ and 
 \begin{equation} \label{eq:condition-trunc-2921}
E_s c_1^{ -(1+s)/2 } c_2^{1/2} n^{ - \frac{  (m-1) (2r-1) + \epsilon  }{2\left[ (s+2r)(m-1) + 4qr + \epsilon \right] } } \leq \frac{3}{32 \log(6/\delta)}  
 \end{equation}
Then, with probability greater than $1 - \delta$, we have
\begin{align} \label{eq:rate-IWKRR-1392}
 &    \left\| f_{\mathbf{z},\lambda}^D-f_{\mathcal{H}} \right\|_{\rhote_X} \leq n^{-  \beta r} \left( A_1 c_2^{- (m-1) / 2q } + A_2 c_1^{-1/2} c_2 + A_3 c_1^{-s/2} c_2^{1/2} + 2^r R c_1^r \right),
\end{align}
where $A_1, A_2, A_3 > 0$ constants  are defined as
 \begin{align}
 & A_1 := \left(  \| f_\rho \|_{\rhote_X} + \| f_\mathcal{H} \|_{\rhote_X}  \right)    \left( 2^{6q-1} m! W^{m - 2} \sigma^2 \right)^{1/2q} \nonumber, \\
 & A_2 :=  2^{1/2} 16 (M + \left\| f_\mathcal{H} \right\|_\mathcal{H})  \log \left( 6/\delta\right), \quad A_3 := 2^{1/2} 16 (M + \left\| f_\mathcal{H} \right\|_\mathcal{H}) E_s \log \left( 6/\delta\right). \label{eq:A-constants-def-2949}
 \end{align}

\end{theorem}

Theorem \ref{theo:clipped-KRR-re} shows that the IW-KRR predictor using the clipped IW function $w_D$ converges to the target function $f_\mathcal{H}$ as the sample size $n$ increases, if the clipping threshold $D$ increases at an appropriate rate as $n$ increases. We can make the following observations. 
\begin{itemize} 
    \item The convergence rate in \eqref{eq:rate-IWKRR-1392} can be written as 
\begin{equation} \label{eq:rate-IWKRR-1458}
 \left\| f_{\mathbf{z},\lambda}^D-f_{\mathcal{H}} \right\|_{\rhote_X} = \mathcal{O}\left( n^{ - \frac{ r }{ 2r + s + \xi} } \right) \quad (n  \to \infty),  
\end{equation}
where $\xi := (4qr + \epsilon) / (m-1) > 0$ can be arbitrarily small as $m$ can be arbitrarily large. 
 Therefore, the rate can be arbitrarily close to the optimal rate $n^{ - \frac{ r }{ 2r + s } }$ of KRR without covariate shift \citep{caponnetto2007optimal} or the rate \eqref{IW_generalization_bound} of the IW-KRR predictor when the IW function is bounded, i.e., $q = 0$ in Theorem \ref{main_theorem}; see the discussion thereof. Notably, the rate \eqref{eq:rate-IWKRR-1458} holds even when the IW function is unbounded, or $q > 0$. Thus, Theorem \ref{theo:clipped-KRR-re} implies that the clipped IW function $w_D$ with a suitably chosen threshold $D$ can improve the convergence rate of the IW-KRR predictor.

    \item 
    The clipping threshold $D$ introduces a bias for the predictor $f_{ {\bf z}, \lambda }^D$ in estimating the target function $f_{\mathcal{H}}$, but can reduce the variance of the predictor $f_{ {\bf z}, \lambda }^D$.
    The choice of the threshold $D$ and (and the regularization constant $\lambda$) in \eqref{eq:lambda-D-2915} can be understood as the one optimally balancing this bias-variance trade-off; this balancing leads to the faster rate of the IW-KRR predictor using the clipped IW function. 
\end{itemize}

We compare Theorem \ref{theo:clipped-KRR-re} with the related results of  \citet[Theorem 4 and Corollary 2]{ma2022optimally}. They derive convergence rates of the IW-KRR predictor with a clipped IW function, assuming that (i) the IW function satisfies $\int w(x) d\rhote(x) < \infty$, (ii) the eigenfunctions $(e_i)_{i=1}^\infty \subset \mathcal{H}$ of the covariance operator $T$ are uniformly bounded:  $\sup _{i \geq 1}\left\|e_i\right\|_{\infty} \leq 1$; 
 (iii) the variance of the output noise is {\em lower} bounded by the RKHS norm of the regression function $f_\rho$ (assuming that it belongs to the RKHS). Under these assumptions and the threshold chosen as $D \propto \sqrt{n}$, they derive near-optimal convergence rates in Sobolev RKHSs.   

Key differences between our Theorem \ref{theo:clipped-KRR-re} and the results of \citet{ma2022optimally} include the following: 
\begin{itemize}
    \item By assuming that the IW function $w(x)$ satisfies Assumption \ref{IW_assumption}, where $q \in [0,1]$ quantifies the degree of the unboundedness of the IW function, we analyze how this degree $q$ affects the convergence rate and how the clipped IW function $w_D$ with appropriate threshold $D$ can eliminate the effects of $q$. 
    \item  We do not assume the uniform boundedness of the eigenfunctions, which is assumed in \citet{ma2022optimally}. While this condition is sometimes assumed in the literature \citep[e.g.,][]{steinwart2009optimal,mendelson2010regularization},  it is known that it is not always satisfied. Indeed, \citet[Example 1]{zhou2002covering} gives an example of an infinitely smooth kernel on $[0,1]$ whose eigenfunctions (where the integral operator is defined with respect to the Lebesgue measure) are not uniformly bounded.  
    \item  We assume that the range of output $y$ is upper-bounded, while \citet{ma2022optimally} consider a large-noise regime where the variance of $y$ is lower-bounded.
\end{itemize}

\section{Binary Classification}
\label{sec:binary_classification}

This section describes the applicability of the above results  to {\em binary classification}, where the task is to predict a binary label  $y \in Y := \{ - 1, 1 \}$ for a given $x \in X$.   
Let $\rho^{\rm te}(x,y) = \rho(y|x) \rho_X^{\rm te}(x)$ and $\rho^{\rm tr}(x,y)= \rho(y|x) \rho_X^{\rm tr}(x)$ be test and training distributions on $X \times Y$, where $\rho(y|x)$ is the conditional distribution on  $Y = \{ -1, 1\}$ given an input $x \in X$, and  $\rho_X^{\rm tr}(x)$ and $\rho_X^{\rm te}(x)$ are training and test input distributions.  

For a real-valued function $f: X \mapsto \mathbb{R}$, we can consider its {\em sign} as a classifier: ${\rm sgn}(f(x)) = 1$ if $f(x) \geq 0 $ and ${\rm sgn}(f(x)) = - 1$ if $f(x) < 0$. 
Therefore, by defining $f_{  {\bf z}, \lambda }$ as the IW-KRR predictor obtained from training data $(x_i, y_i)_{i=1}^n \in (X \times \{-1, 1\})^n$ and the IW function $w(x)$,  one can construct a classifier as the sign of $f_{  {\bf z}, \lambda }$.

%In the following, we show that the above results can be applied to binary classification algorithms, i.e., when $Y=\{-1,1\}.$ The problem of statistical learning in classification consists of predicting the value $y \in \{-1,1\}$ for a  given $x \in X$. As in the regression setting here we also distinguish training and testing input distributions, while the conditional distribution is the same and supported on $\{-1,1\}.$ We consider  \textit{binary classifiers}, namely functions $f:X\rightarrow \{-1,1\}$ that assign a label to each point $x \in X$. We denote the \textit{classification error} of a classifier as follows

The risk (or the expected misclassification error) $\mathcal{R}(f)$ of $f$ as a classifier is defined as the probability of ${\rm sgn}(f(x))$ being different from $y$, where $(x,y) \sim \rho^{\rm te}$: 
 \begin{equation*}
     \mathcal{R}(f)=\rho^{\rm te}\left\{(x,y) \in X \times \{-1,1\} \,\, |\,\, {\rm sign}(f(x)) \neq y \right\}.
 \end{equation*}
 It is well known that the minimum of the risk is attained by the {\em Bayes classifier}, defined as the sign of the regression function $f_\rho(x) = \int y d\rho(y|x) = \rho(y=1 \, |\,x)-\rho(y=-1\,|\,x)$: 
 %It is well known that 
 \begin{equation*}
     \min_{f:X \mapsto \mathbb{R}} \mathcal{R}(f)=\mathcal{R}(f_{\rho}).
 \end{equation*}
 %where $f_{\rho}(x)=\int_{\mathbb{R}}yd\rho(y\,|\,x)=P(y=1 \, |\,x)-P(y=-1\,|\,x)$ is a regression function. The classifier $\text{sgn}(f_{\rho}(x))$ is called the \emph{Bayes rule}.
 
For any function $f: X \mapsto \mathbb{R}$, it can be shown \citep[e.g.,][]{bartlett2006convexity,bauer2007regularization} that the excess risk $\mathcal{R}(f)-\mathcal{R}(f_{\rho})$ is upper bounded by the $L^2(X, \rho_X^{ \rm te })$-distance between $f$ and $f_\rho$:  
 %It can be shown \citep{bartlett2006convexity,bauer2007regularization} that the excess misclassification error $\mathcal{R}(f)-\mathcal{R}(f_{\rho})$ can be upper bounded as follows
\begin{equation} \label{eq:rough-bound-excess}
    \mathcal{R}(f)-\mathcal{R}(f_{\rho}) \leq \|f-f_{\rho}\|_{\rhote_X}.
\end{equation}
%Therefore,  one can bound the excess risk of the classifier ${\rm sgn}( f_{ {\bf z}, \lambda } )$ by bounding the distance $\|f_{ {\bf z}, \lambda } -f_{\rho}\|_{\rhote_X}$, which can be done by applying Theorem~\ref{main_theorem} and assuming $f_\mathcal{H} = f_\rho$. 
%meaning that the Theorem~\ref{main_theorem} can be directly applied to achieve finite sample guarantees replacing $f$ by $f_{\mathbf{z},\lambda}$. 
Moreover, suppose that the Tsybakov noise condition \citep{mammen1999smooth,tsybakov2004optimal} holds for the noise exponent $l \geq 0$: 
\begin{equation}\label{margin_condition}
\rhote_{X}\left(\left\{x \in X:f_{\rho}(x) \in [-\Delta,\Delta]\right\}\right) \leqslant B_{l} \Delta^{l}, \quad \forall \Delta \in[0,1],
\end{equation} 
Then one can refine the excess risk bound \eqref{eq:rough-bound-excess} as 
\begin{equation} \label{eq:bound-refined-classific}
\mathcal{R}\left(f \right)-\mathcal{R}\left(f_{\rho}\right) \leqslant 4 c_{\alpha}\left\|f-f_{\rho}\right\|_{\rhote_X}^{\frac{2}{2-\alpha}},
\end{equation}
where $\alpha := l/(l+1)$ and $c_{\alpha} := B_{\alpha}+1$ \citep{bauer2007regularization,yao2007early}. Intuitively, a larger $l$ implies that the noise around the decision boundary  $\{x \in X: f_{\rho}(x)=0\}$ is lower and thus the classification problem is easier, leading to a faster convergence rate. The case $l = 0$ imposes no assumption on the decision boundary, thus recovering \eqref{eq:rough-bound-excess}.

Now, one can bound the excess risk of the classifier ${\rm sgn}( f_{ {\bf z}, \lambda } )$, by setting $f = f_{ {\bf z}, \lambda }$ in  \eqref{eq:bound-refined-classific},  using Theorem~\ref{main_theorem} and assuming $f_\mathcal{H} = f_\rho$, as summarized as follows.

%\cite{mammen1999smooth} first showed that one can attain fast rates under mild assumptions on the behavior of $f_{\rho}(x)$ in a neighborhood of the boundary $\{x: f_{\rho}(x)=0\}$. Namely if the Tsybakov noise condition holds \citep{tsybakov2004optimal} for $l\geq 0$,
%\begin{equation}\label{margin_condition}
%\rhote_{X}\left(\left\{x \in X:f_{\rho}(x) \in [-\Delta,\Delta]\right\}\right) \leqslant B_{l} \Delta^{l}, \quad \forall \Delta \in[0,1],
%\end{equation} 
%then 
%\begin{equation*}
%\mathcal{R}\left(f_{\mathbf{z},\lambda}\right)-\mathcal{R}\left(f_{\rho}\right) \leqslant 4 c_{\alpha}\left\|f_{\mathrm{z}}-f_{\rho}\right\|_{\rhote_X}^{\frac{2}{2-\alpha}},
%\end{equation*}
%with $\alpha=l/(l+1)$ and $c_l=B_l+1$(\cite{bauer2007regularization},\cite{yao2007early}). A direct application of Theorem~\ref{main_theorem} gives us

\begin{corollary}

Suppose that the conditions of Theorem \ref{main_theorem} hold with $Y = \{-1, 1\}$ and thus $M=1$.  Moreover, assume that $f_\rho \in \mathcal{H}$ and that the Tsybakov noise condition \eqref{margin_condition} holds. 
Let $\delta \in (0,1)$ be an arbitrary constant. Let 
\begin{equation*} 
\lambda = c n^{- \beta},
\end{equation*}
where $\beta > 0$ is defined by 
$$
\beta := \frac{1}{2r +  s (1-q)  + q} = \frac{1}{2r + A}, \quad \text{where} \quad \quad A := s (1-q)  + q,
$$
and $c > 0$ is such that 
$$
c \geq \left( 64 (W+\sigma^2)  (E_{s})^{2(1-q)}  \log^2\left( 6/\delta \right)  \right)^{1/(1+A)}.
$$
Suppose $n$ is large enough so that $\lambda \leq 1$. Then, with probability greater than $1-\delta$, it holds that 
\begin{align*}    
\mathcal{R}\left(f_{\mathbf{z},\lambda}\right)-\mathcal{R}\left(f_{\rho}\right) & \leq 4 c_\alpha n^{- \frac{2 r \beta}{2 - \alpha} }  \left\{ 16  \left(      1   +  \left\| f_\rho \right\|_{\mathcal{H}}    \right)  \left( W   +   \sigma  (E_{s})^{1-q}   \right) c^{-A/2} \log \left( 6/\delta \right)     + c^{r}    R \right\}.  
\end{align*}
 
\end{corollary}

\begin{comment}
    
\begin{corollary}
Suppose that the conditions of Theorem \ref{main_theorem} hold with $Y = \{-1, 1\}$ and thus $M=1$. (\textcolor{blue}{Is this true? Can we then remove $M$ from the generalization bound?}\textcolor{red}{Yes, we can remove it. The bound (40) doesn't depend on it.}) Moreover, assume that $f_\rho \in \mathcal{H}$ and that the Tsybakov noise condition \eqref{margin_condition} holds. 
Let $\delta \in (0,1)$ be an arbitrary constant, and define 
\begin{equation*}
    \lambda = \left(\frac{8E_{s}^{1-q}(\sqrt{W}+\sigma)\log\left(\frac{6}{\delta} \right)}{\sqrt{n}}\right)^{\frac{2}{2r+s+q(1-s)}}
\end{equation*}
Suppose $n$ is large enough so that $\lambda \leq \| T \| $. Then, with a probability greater than $1-\delta$, it holds that 
%Assume that the assumptions of Theorem \ref{main_theorem} is satisfied together with the margin condition (\ref{margin_condition}) and let $f_{\rho} \in \mathcal{H}.$ If $\lambda$ satisfies the constraints $\lambda \leq \|T\|$ and 
%\begin{equation*}
%    \lambda = \left(\frac{8E^{1-q}_{s}(\sqrt{W}+\sigma)\log\left(\frac{6}{\delta} \right)}{\sqrt{n}}\right)^{\frac{2}{2r+s+q(1-s)}},
%\end{equation*}
%for $\delta \in (0,1),$ $q\in[0,1]$ and $r \geq 0.5.$ Then with probability greater than $1-\delta,$ it holds
\[
\mathcal{R}\left(f_{\mathbf{z},\lambda}\right)-\mathcal{R}\left(f_{\rho}\right) \leq 12c_{\alpha}\left(M+R\right) \left(\frac{8E^{1-q}_{s}(\sqrt{W}+\sigma)\log\left(\frac{6}{\delta} \right)}{\sqrt{n}}\right)^{\frac{4r}{(2r+s+q(1-s))(2-\alpha)}}.
\]
\end{corollary}

\end{comment}

\begin{figure}[ht] 
  \label{ fig7} 
  \begin{minipage}[b]{0.5\linewidth}
    \centering
    \includegraphics[width=\linewidth]{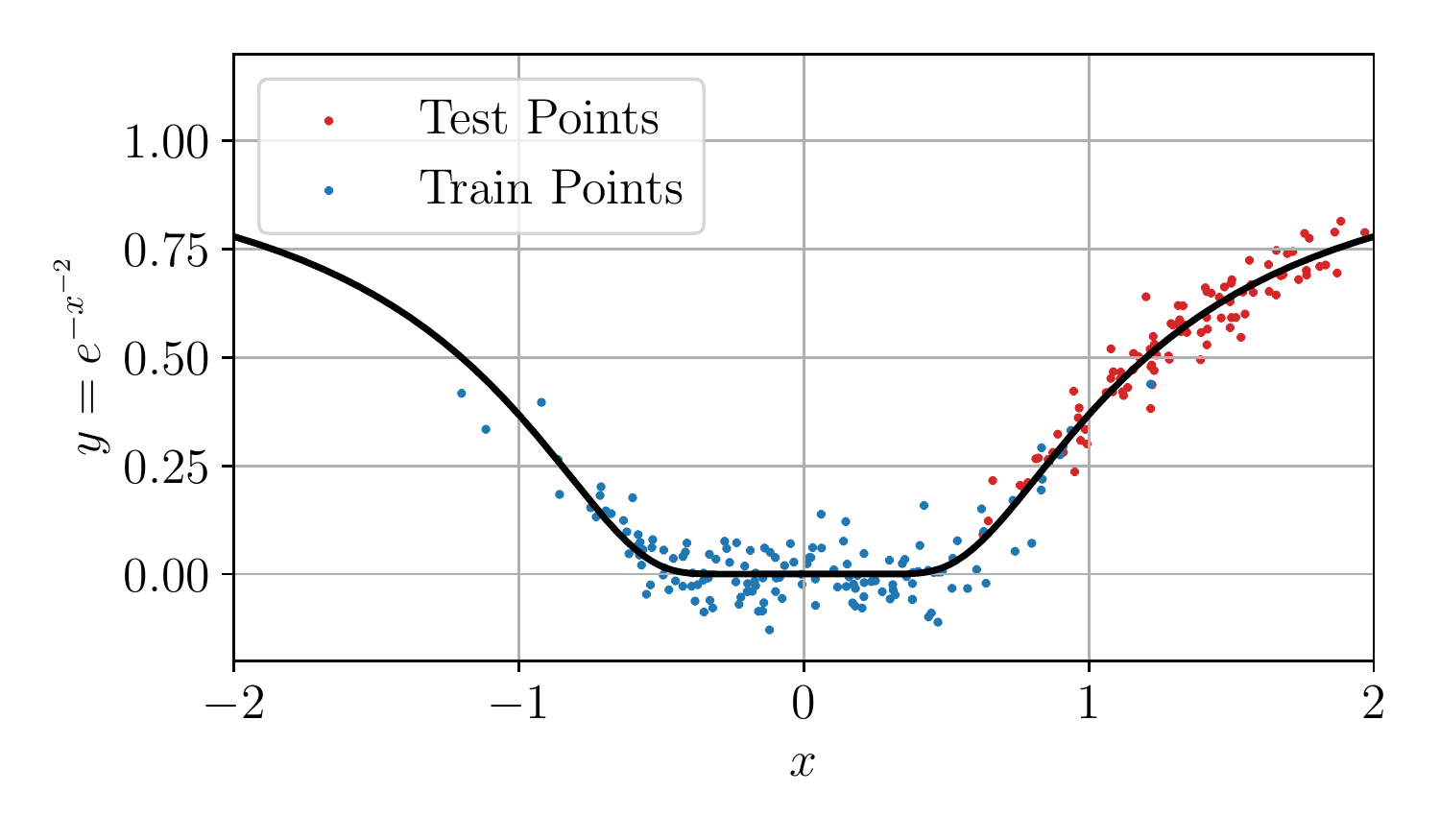} 
  \end{minipage}%%
  \begin{minipage}[b]{0.5\linewidth}
    \centering
    \includegraphics[width=\linewidth]{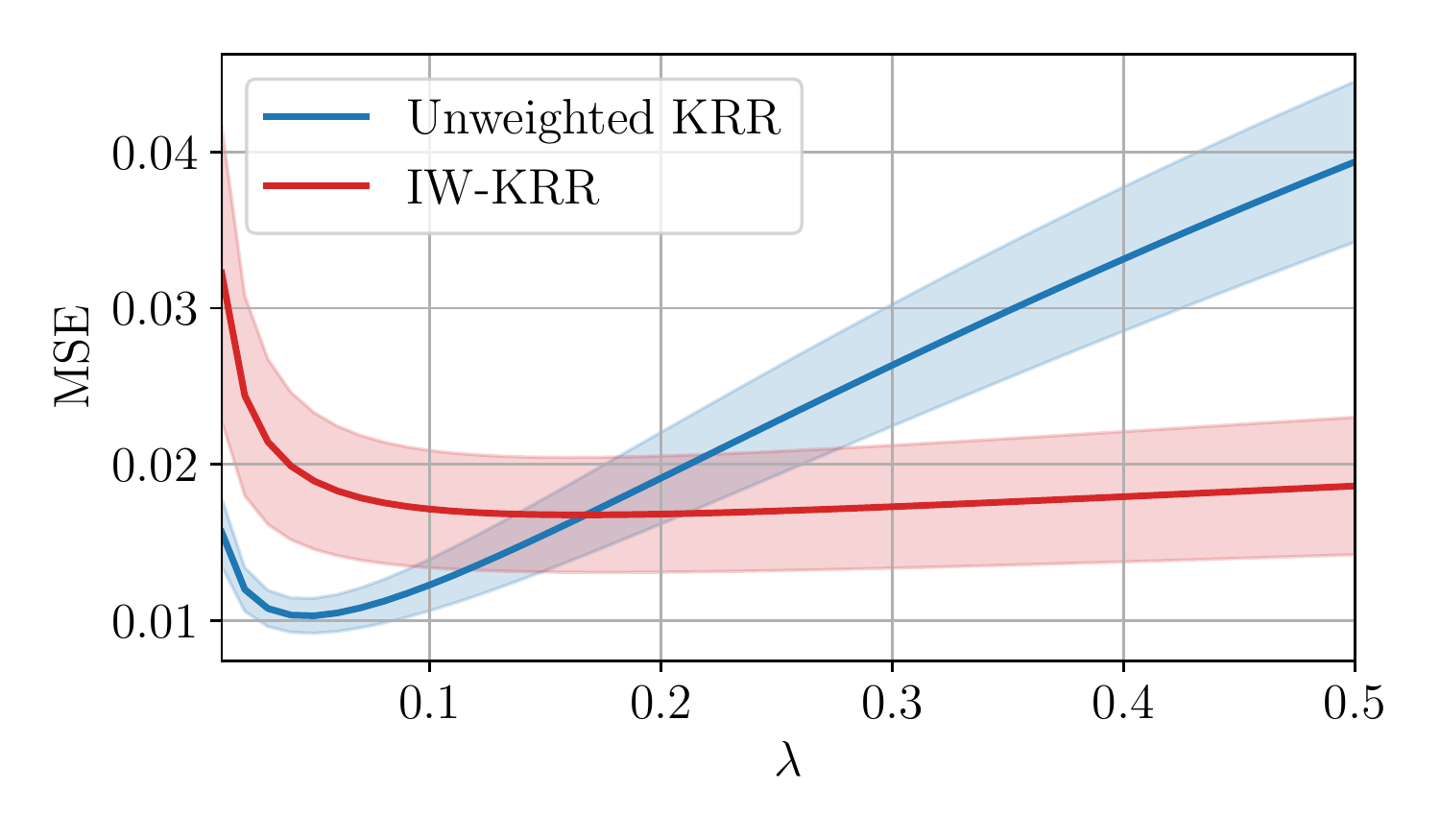} 
  \end{minipage} 
  \begin{minipage}[b]{0.5\linewidth}
    \centering
    \includegraphics[width=\linewidth]{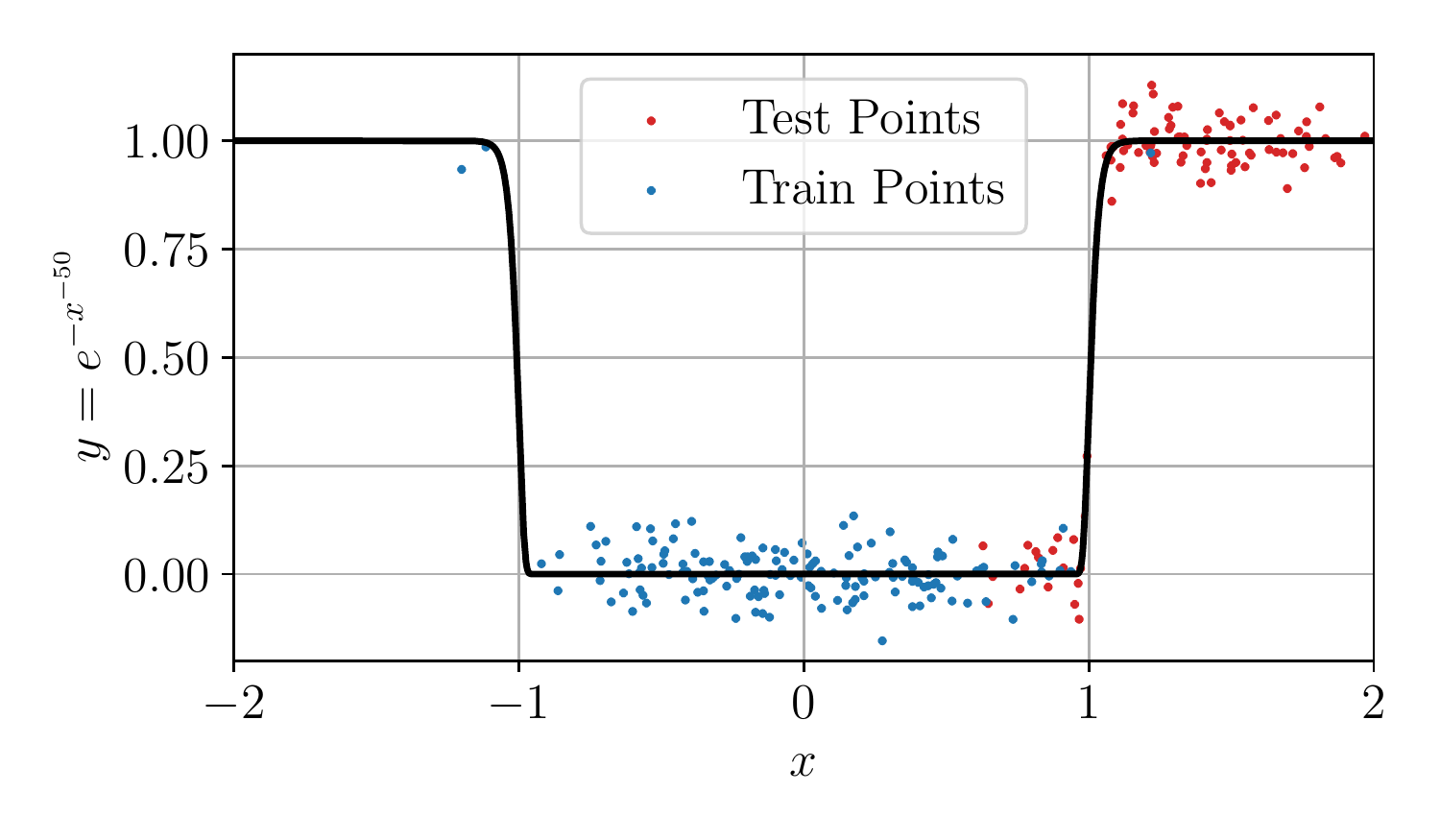} 
  \end{minipage}%% 
  \begin{minipage}[b]{0.5\linewidth}
    \centering
    \includegraphics[width=\linewidth]{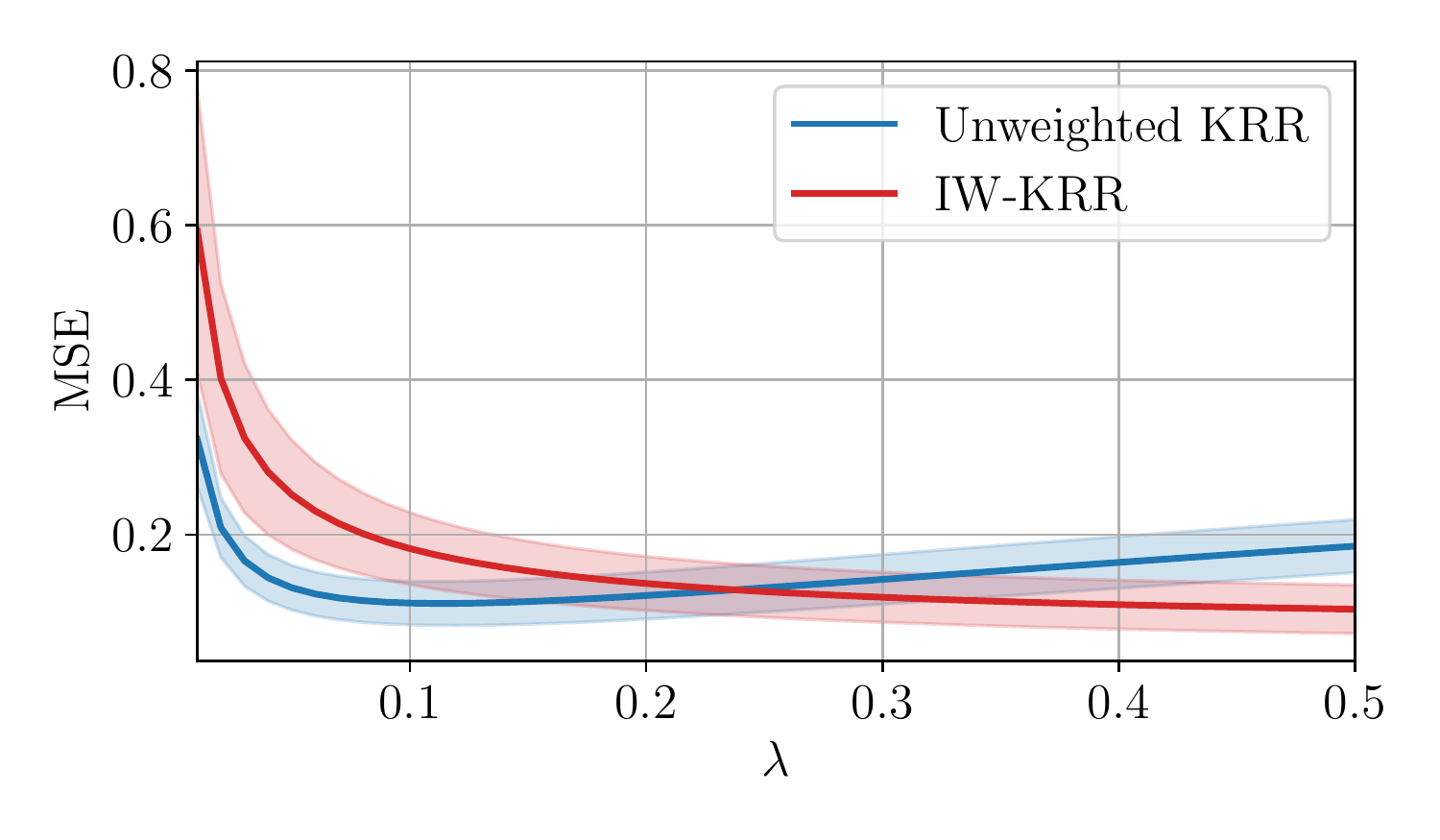} 
  \end{minipage} 
  \caption{
  Comparison between IW-KRR and unweighted KRR for two different regression functions. In the top left panel, the black curve represents the regression function \eqref{regression} with $k = 1$; blue and red points are training and test data points, respectively. The top right panel shows the Mean Square Errors (MSE) of the IW-KRR and unweighted KRR for different values of the regularization constant $\lambda$. The bottom panels show the corresponding results for the regression function \eqref{regression} with $k = 25$. 
  %Comparison between unweighted KRR and IW-KRR for different regression functions. In the top left figure we consider a smooth regression from which training and testing inputs are distributed as $\mathcal{N}(0,0.5)$ and $\mathcal{N}(1.5,0.3)$ respectively. On the top right panel we report the performance of the IW-KRR and unweighted KRR for different regularization parameter values. On the bottom row we repeat the experiment using relatively non-smooth regression function.
  }
  \label{fig:2}
\end{figure}

\section{Simulations}
\label{sec:simulations}

We report here the results of simple simulation experiments. Let $\mathcal{N}(\mu, \sigma^2)$ denote the univariate Gaussian distribution with mean $\mu \in \mathbb{R}$ and variance $\sigma^2 > 0$.   Let $X = \mathbb{R}$. For $k \in \mathbb{N}$, we define the regression function $f_\rho$ as 
\begin{equation} \label{regression}
    f_\rho(x) := e^{-\frac{1}{x^{2k}}}, \quad  k \in \mathbb{N}.
\end{equation}
We assume that an output is given by $y = f_\rho(x) + \varepsilon$, where $\varepsilon \sim \mathcal{N}(0, 0.05^2)$  is an independent noise. We define the training and test input distributions as $\rho_{X}^{\rm tr} = \mathcal{N}(0, 0.5)$ and $\rho_{X}^{\rm te} = \mathcal{N}(1.5, 0.3)$, respectively. 

%We consider a simple one-dimensional regression problem with a regression function  
%\begin{equation} \label{regression}
%    f_\rho(x) := e^{-\frac{1}{x^{2k}}}, \quad  k \in \mathbb{N},
%\end{equation}
%corrupted by homoscedastic Gaussian noise with mean $\mu=0$ and standard deviation $\sigma=0.05$. We assume that $x\sim \mathcal{N}(0,0.5)$ at training time and $x\sim \mathcal{N}(1.5,0.3)$ at testing time.

We compare the performance of IW-KRR using the IW function $w(x) = d\rho_X^{\rm te}(x) / d\rho_X^{\rm tr}(x)$ and standard KRR using the uniform weights.

\paragraph{KRR using a Gaussian kernel.}

The first experiment uses the Gaussian RBF kernel with the unit length scale: $K(x,x') = \exp( - (x - x')^2 )$.
We consider two different values for $k$ in the regression function: $k = 1$ and $k = 25$. In Figure \ref{fig:2}, the left panels describe the corresponding regression functions and training and test data points. The right panels report the Mean Square Errors (MSE) of the IW-KRR and unweighted KRR for different values of the regularization constant $\lambda$ for $K = 1$ (top right) and $K = 25$ (bottom right).  

For $k = 1$, the MSE of unweighted KRR with the optimal regularization constant $\lambda$ is slightly smaller than the MSE of IW-KRR with optimal $\lambda$. This observation can be explained by the fact that the regression function $f_\rho$ is sufficiently smooth and can be well approximated by functions in the RKHS of the Gaussian kernel. 

In contrast, for $k = 25$ the regression function essentially becomes a piece-wise constant function. It is known that neither constant functions nor discontinuous functions belong to the RKHS of the Gaussian kernel \cite[Corollary 4.44]{steinwart2008support}, so one can understand that a larger $k$ increases the level of misspecification. In this case, the IW correction is beneficial, as described in the bottom left panel of Figure \ref{fig:2}.

%In the first experiment we compare the performance of importance weighted and unweighted KRR for two different values of $k.$ We choose a Gaussian RBF kernel with length-scale parameter equal to one.  On the top left panel of Figure \ref{fig:2} the regression function with $k=1$ along with a randomly drawn set of points from training and testing measures are reported. On the top right panel the performance of the weighted and unweighted KRR for different values of the regularization parameter is presented. As one can see, importance weighting adaptation with optimal regularization parameter performs slightly worse than unweighted KRR with optimal $\lambda.$ This can be explained by the fact that the underling regression function is sufficiently smooth and can be well approximated by the functions in RKHS with exponential kernel.

\begin{figure}[t]
    \centering
    \includegraphics[width=0.5\textwidth]{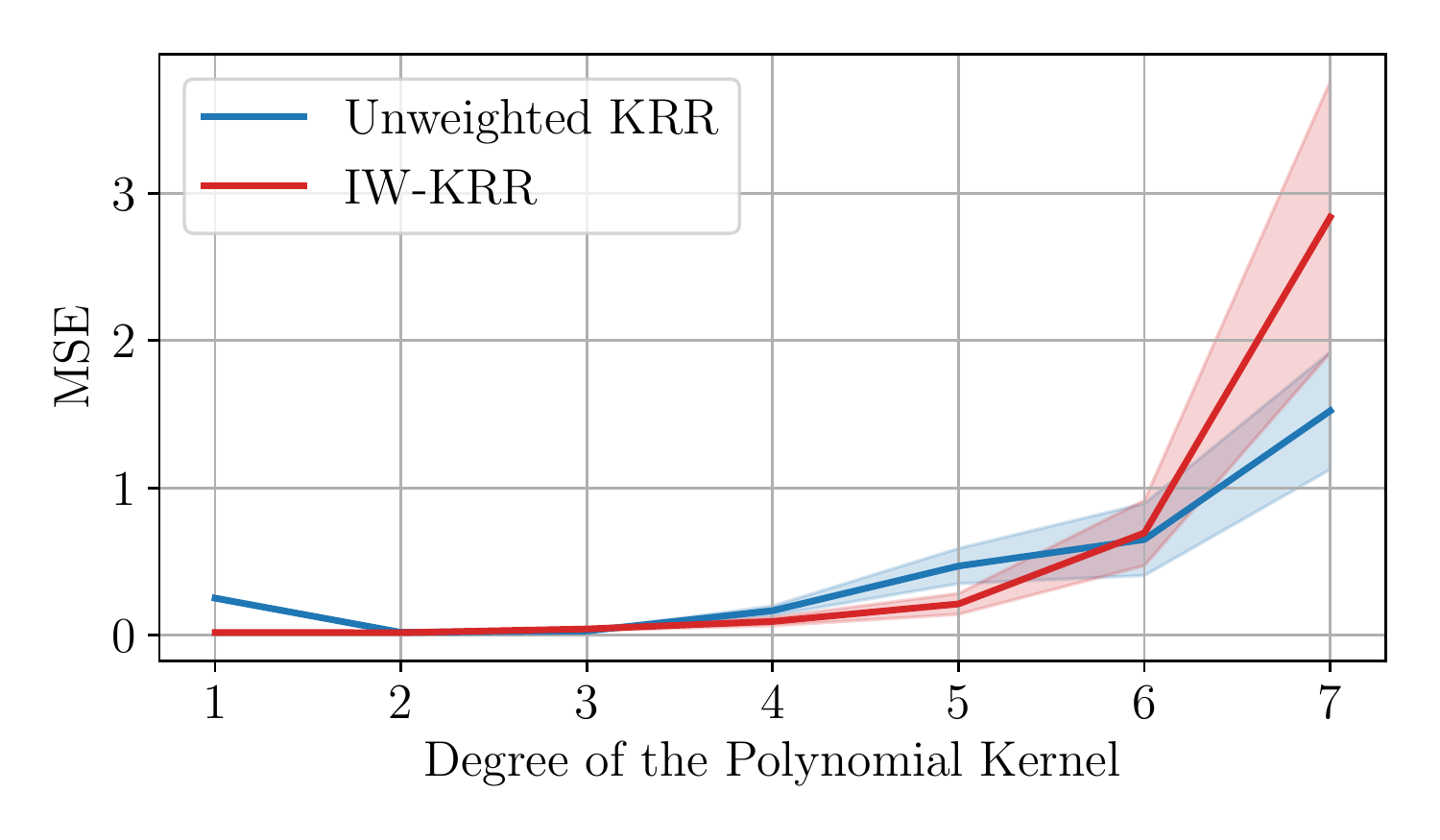}
    \caption
    {
    Mean square errors (MSE) of IW-KRR and unweighted KRR using polynomial kernels of different degrees.
    %Performance of IW-KRR and unweighted KRR using a polynomial kernel with increasing degree. As the degree of the polynomial kernel increases, the role of the importance weights diminishes.
    }
    \label{fig:3}
\end{figure}

%On the other hand when $k$ in (\ref{regression}) is large, the regression function essentially becomes piece-wise constant.  As it is well known, neither constants nor discontinuous functions belong to the RKHS associated with the Gaussian RBF  kernel, so the increased values of $k$ increases the level of misspecification. The regression function with $k=25$ together with the randomly drawn training and testing points is reported in the bottom-left panel of Figure \ref{fig:2}. Unlike the previous example, IW adaptation here can be beneficial as shown in the bottom-right panel of the figure. Notice that in both these examples, the optimal regularization parameter for unweighted KRR is smaller in comparison with the optimal $\lambda$ of IW-KRR. This is well justified by the optimal choice of the regularization parameter in Theorem \ref{main_imperfect}.

\paragraph{KRR using polynomial kernels.}

We next use polynomial kernels of different degrees to illustrate the relation between the capacity of the RKHS and the benefit of the IW correction. We use here the regression function \eqref{regression} with $k=1$. Figure \ref{fig:3} describes the MSEs of the IW-KRR and unweighted KRR using the polynomial kernel $K(x,x')  = ( x^\top x' + 1 )^m$ with $m = 1, \dots, 7$. For each  degree of the polynomial kernel, we repeat the experiment 100 times for a fixed value of the regularization $\lambda=1$.
For $m = 1$, in which case the KRR becomes linear regression,  the benefit of the IW correction is apparent. Unweighted KRR learns a linear function that fits the training data but does not predict well the test data.  

%The relation between the approximation properties of the RKHS associated to a kernel and the performace of weighted KRR can be showcase using a polynomial kernel with an increasing kernel degree. In Figure \ref{fig:3} the performance of KRR with polynomial kernel is given for the function (\ref{regression}) with $k=1.$ For the degree-one kernel, the advantage of IW adaptation is apparent; the regression function can be well approximated by the linear function under the testing distribution. For the degree-two polynomial, the space of quadratic functions can approximate the true regression function in the supremum norm in a suitable chosen domain, and therefore the model trained on the training data with uniform weights gives a globally suitable model. With the degree of polynomial kernel increasing, IW adaptation does not provide a clear benefit.

\section{Conclusion}
\label{sec:conclusion}

Covariate shift naturally occurs in real-world applications of machine learning; thus, understanding its effects and how to address it is fundamental. Importance-weighting (IW) is a standard approach to correct the bias caused by covariate shift, and classical results show that the IW correction is necessary when the learning model is parametric and misspecified. On the other hand, recent studies indicate that IW correction may not be necessary for large-capacity models such as neural networks and nonparametric methods.  

The current work bridges these two lines of research. We have studied how covariate shift affects the convergence of a regularized least-squares algorithm whose hypothesis space is given by a reproducing kernel Hilbert space (RKHS), namely kernel ridge regression (KRR). Different choices of the RKHS (or the kernel) lead to different learning models, and thus our analysis covers a variety of settings. In particular, the model may become parametric when the RKHS is finite-dimensional and become nonparametric when the RKHS is infinite-dimensional. The model may become over-parameterized when a neural tangent kernel defines the RKHS. 

%Covariate shift can naturally arise in real-world applications of supervised learning. Importance-weighting (IW) is a standard approach for correcting for the bias introduced by covariate shift, but a theoretical understanding of its mechanism has yet to be mature,  particularly for large capacity models such as kernel methods and neural networks. The current work contributes to the theoretical understanding of the IW correction for regularized least-square algorithms with the hypothesis space defined by an RKHS, namely kernel ridge regression (KRR). 

%As different choices of the RKHS (or the kernel) lead to different learning models, our theoretical analysis provides insights into such models. For example, one can obtain a parametric regression model if one chooses a kernel inducing a finite-dimensional RKHS, such as linear and polynomial kernels. One obtains a nonparametric model using a kernel inducing an infinite-dimensional RKHS, such as Gaussian and Mat\'ern kernels. Overparameterized models, such as neural networks, may be obtained using Neural Tangent Kernels.  

A key ingredient of our analysis is to consider the {\em projection} of the true regression function onto the model class,  similar to the classical literature on covariate shift in parametric models but different from the recent literature on nonparametric models. We have formulated the projection as the function in the RKHS that is the closest to the regression function in terms of L2 distance for the test input distribution. The projection is identical to the regression function if the RKHS contains the regression function (the well-specified case) or if the RKHS is universal. If the RKHS does not contain the regression function and the RKHS is not universal, then the projection may differ from the regression function and from projections defined for other distributions, such as the training input distribution.  
 
%A key point in our analysis is to consider the {\em projection} of the regression function onto the RKHS with respect to the test input distribution $\rho_X^{\rm te}$. This projection may differ from the projection obtained from the training input distribution  $\rho_X^{\rm tr}$, when the RKHS does not contain the regression function and the RKHS is not universal.  

%The importance of the IW correction can be observed in the setting. Without the IW correction, which means that the weight function is uniform or different from the IW function, then the predictor does not converge to the projection with respect to the test distribution but to the other projection defined with respect to the training distribution. 

One takeaway from our analysis is that, if the projection exists and differs from the regression function, then different weighting functions can cause the IW-KRR predictor to converge to different projections as the sample size increases. In particular, with the correct IW function, the IW-KRR predictor converges to the projection for the {\em test} input distribution. In contrast, the IW-KRR predictor converges to the projection for a distribution different from the test distribution, if the weighting function differs from the IW function. This is the case with the uniform weighting function, in which case the IW-KRR becomes the standard KRR, and it converges to the projection for the {\em training} input distribution;  this projection is not the best approximation of the true regression function for the {\em test} distribution. Thus, our analysis shows the benefit of using the true IW function when the projection exists and differs from the regression function.  
This observation recovers the classical result on covariate shift in parametric models, but extends it to models with higher capacity.  

On the other hand, if the RKHS contains the regression function or if the RKHS is universal, then the projection is identical to the regression function and thus is independent of a (test or training) distribution with which the projection is defined. In this case, our analysis shows that the IW-KRR predictor converges to the regression function for an {\em arbitrary} weighting function, if it satisfies an appropriate moment condition. Therefore, the uniform weighting function also leads to convergence to the regression function, and one may not need the correct IW function. This observation is consistent with the recent literature on covariate shift in nonparametric models, particularly the concurrent work by \citet{ma2022optimally}, which assumes that the RKHS contains the regression function.  

Thus an interesting case is when the projection exists and differs from the regression function while the model is nonparametric. Such a case includes over-parameterized models, which can be analyzed with neural tangent kernels or random feature approximations, and structured models, such as additive models defined by additive kernels. By studying the resulting projection onto the RKHS, one can obtain new insights into the learning behavior of such models under covariate shifts and the effects of different weighting strategies. We leave this topic for future investigation. 

\subsection*{Acknowledgements} 
The work of D.Gogolashvili and M. Zecchin is funded by the Marie Curie action WINDMILL (grant No. 813999).
M. Kanagawa and M. Filippone have been supported by the French government, through the 3IA Cote d’Azur Investment in the Future Project managed by the National Research Agency (ANR) with the reference number ANR-19-P3IA-0002. M. Kountouris has received funding from the European Research Council (ERC) under the European Union’s Horizon 2020 Research and Innovation Programme (Grant agreement No. 101003431). M. Filippone gratefully acknowledges support from the AXA Research Fund and the Agence Nationale de la Recherche (grant ANR-18-CE46-0002).

%\newpage

\appendix

\section{Auxiliary Results}
\label{sec:appendix-auxiliary-results}

We present the auxiliary results mentioned in the main body of the paper.

\begin{proposition} \label{prop:assump-IW-sufficient}
Assume that for constants $q \in (0,1]$, $W \in (0,\infty)$ and $\sigma \in (0,\infty)$ we have
\begin{align} \label{eq:assumption-weight-3254} 
2  \rhote_X\left( \left\{ x \in X:  w(x) \geq t \right\} \right) \leq \sigma^2 \exp\left( - W^{-1} t^{1/q}   \right) \quad \text{for all } \ t > 0.
\end{align}
Then we have, for all $m \in \mathbb{N}$ with $m \geq 2$,
\begin{align} \label{eq:bound1817}
  \int w^\frac{m-1}{q}(x) d\rhote_X(x) \leq \frac{1}{2} (m-1)! \sigma^2 W^{m-2}. 
\end{align}
Moreover, if $\frac{1}{2} m! \sigma^2 W^{m-2} \geq 1$, we have 
\begin{align} \label{eq:bound1821}
 \left( \int w^\frac{m-1}{q}(x) d\rhote_X(x) \right)^q \leq  \frac{1}{2} m! \sigma^2 W^{m-2}.   
\end{align}
\end{proposition}

\begin{proof}
Let $\alpha > 0$ be arbitrary. Then we have 
\begin{align*}
& \int w^\alpha(x) d\rhote_X(x)  = \int_0^\infty \rhote_X\left( \left\{ x \in X:  w^\alpha(x) \geq t \right\} \right) dt \\
& = \alpha \int_0^\infty \rhote_X\left( \left\{ x \in X:  w(x) \geq s \right\} \right) s^{\alpha - 1} ds \stackrel{(*)}{\leq} 2^{-1} \alpha \sigma^2 \int_0^\infty  \exp\left( - W^{-1} s^{1/q}  \right)  s^{\alpha - 1} ds \\
& = 2^{-1} \alpha \sigma^2 \int_0^\infty \exp\left( - \tau  \right) (W \tau)^{\alpha - 1} q(W\tau)^{q-1} d\tau = 2^{-1} \sigma^2 q \alpha W^{q\alpha - 1}   \int_0^\infty \exp\left( - \tau  \right) \tau^{q\alpha - 1} d\tau \\
& = 2^{-1} \sigma^2 q \alpha W^{q\alpha - 1}  \Gamma(q\alpha) = 2^{-1} \sigma^2 W^{q\alpha - 1}  \Gamma(q\alpha + 1),
\end{align*}
where $(*)$ follows from \eqref{eq:assumption-weight-3254}  and  $\Gamma(\cdot)$ denote the Gamma function. 
Now setting $\alpha = (m-1)/q$, we have
\begin{align*}
    \int w^\frac{m-1}{q}(x) d\rhote_X(x) \leq 2^{-1} \sigma^2 W^{m-2}  \Gamma(m) = 2^{-1} \sigma^2 W^{m-2} (m-1)!,  
\end{align*}
which proves the first assertion \eqref{eq:bound1817}. 
The second assertion \eqref{eq:bound1821} follows from \eqref{eq:bound1817}, the assumption $\frac{1}{2} m! \sigma^2 W^{m-2} \geq 1$, and $q \in (0,1]$. 

\end{proof}

\begin{proposition} \label{prop:asumption-G-radon}
Suppose that $\rhote_X$ is absolutely continuous with respect to $\rho'_X$, and the Radon-Nikodym derivative $d\rhote_X / d\rho'_X$ is bounded. Then Assumption \ref{ass:connection_te_'} is satisfied with $G := \left\| d\rhote_X / d\rho'_X \right\|_\infty$.
\end{proposition} 

\begin{proof}
For operators $A$ and $B$  on $\mathcal{H}$,   denote by $A \geq B$  that $A - B$ is a non-negative operator. 
Let $G := \left\| d\rhote_X / d\rho'_X \right\|_\infty$. 
 For all $f \in \mathcal{H}$, we have 
    \begin{align*}
&  \left< f, T f \right>_{\mathcal{H}} = \left< f,  \int K_x f(x) d\rhote_X(x) \right>_{\mathcal{H}} =   \int \left< f,  K_x \right>_{\mathcal{H}} f(x) d\rhote_X(x)  = \int f^2(x) d\rhote_X(x)  \\
& =  \int f^2(x) \frac{d\rhote_X}{d\rho'_X}(x) d\rho'_X(x) \leq G \int f^2(x) d\rho'_X(x) = G \left< f, T' f \right>_{\mathcal{H}}.
    \end{align*} Therefore,  we have   
    \begin{align*}
T \leq G T' \leq G (T' + \lambda) \  \Longrightarrow \ T(T'+\lambda)^{-1} \leq G I,
\end{align*}
   where $I: \mathcal{H} \mapsto \mathcal{H}$ is the identity operator.  This implies $\left\| T(T'+\lambda)^{-1} \right\| \leq G$, which proves the assertion.  
\end{proof}

\section{Preliminaries to the Proofs of Main Results}

We present here auxiliary results needed for proving the main results.

As in the main body, we assume throughout $\| K_x \|_{\mathcal{H}}^2 = K(x,x) \leq 1$ for all $x \in X$. 
For $g \in \mathcal{H}$, let  $g^\top : \mathcal
{H} \to \mathbb{R}$ be the linear functional such that $g^\top f = \left< g, f \right>_{\mathcal{H}}$ for $f \in \mathcal{H}$. In particular,  $K_x^\top : \mathcal{H} \mapsto \mathbb{R}$ for $x \in X$ is defined as $K_x^\top f = \left< K_x, f \right>_{\mathcal{H}} = f(x)$ for $f \in \mathcal
H$.

Proposition \ref{bernstein} below is a version of the Bernstein inequality for Hilbert space-valued random variables  from \citet[Proposition 2]{caponnetto2007optimal}.

%The following version of Bernstein inequality for Hilbert space valued random variables is from \citep[Proposition 2]{caponnetto2007optimal}. 

\begin{proposition} \label{bernstein}
Let $F$ be a real separable Hilbert space and $\xi \in F$ be a random variable. Assume that there exist constants $L, \sigma > 0$ such that  
%Let $(Z, \rho)$ be a probability space and let $\xi$ be a random variable on $Z$ taking value in a real separable Hilbert space $H.$ Assume that there are two positive constants $L$ and $\sigma$ such that
\begin{equation}\label{bernstein_condition}
    \mathbb{E}\left[\|\xi-\mathbb{E}[\xi]\|_{F}^{m}\right] \leq \frac{1}{2} m ! \sigma^{2} L^{m-2}, \quad \forall\  m \in \mathbb{N}, \  m \geq 2.
\end{equation}
Let $\xi_1, \dots, \xi_n \in F$ be i.i.d.~copies of $\xi$. 
Then, for any $\delta \in (0,1)$, we have 
%$$
%\left\|\frac{1}{n} \sum_{i=1}^{n} \xi_i -\mathbb{E}[\xi]\right\|_{F} \leq \frac{2L\log(2/\delta)}{n}+\sqrt{\frac{2\sigma^{2}\log(2/\delta)}{n}} 
%$$\textcolor{blue}{Attention! The following is Prop 2 of \citet{caponnetto2007optimal}. The dependence on $\log(2/\delta)$ is different. Make sure that the proofs still work properly.}
$$
\left\|\frac{1}{n} \sum_{i=1}^{n} \xi_i -\mathbb{E}[\xi]\right\|_{F} \leq 2 \left( \frac{L}{n} + \frac{\sigma}{\sqrt{n}} \right) \log(2/\delta)
$$
with probability at least $1-\delta$.
%In particular, \eqref{bernstein_condition} holds if
%\[ \|\xi\|_{F} \leq \frac{L}{2},  \quad \mathbb{E}\left[\|\xi\|_{F}^2\right] \leq \sigma^2. \]
\end{proposition}

\begin{comment}
    
\textcolor{blue}{Cite a reference for the following proposition.}
\begin{proposition} \label{approximation_error}
Suppose that the projection $f_\mathcal{H} \in \mathcal{H}$ satisfies Assumption \ref{ass:1} for some $0 \leq r \leq  1$. Then we have
%Let $f_{\mathcal{H}}$ satisfies the Assumption \ref{ass:1} for some $r>0$. Then, the following estimate holds
\begin{equation}\label{bias}
    \left\|f_{\lambda}-f_{\mathcal{H}}\right\|_{\rhote_X} \leq \lambda^{r}\left\|L^{-r} f_{\mathcal{H}}\right\|_{\rhote_X} \quad \text { if } r \leq 1.
\end{equation}\textcolor{blue}{Is this Prop 3 of \citet{caponnetto2007optimal}?Then the following is true. (The above one is what you wrote)}
\begin{equation}\label{bias-corrected}
    \left\|f_{\lambda}-f_{\mathcal{H}}\right\|_{\rhote_X} \leq \lambda^{{r + 1/2} }\left\|L^{-r} f_{\mathcal{H}}\right\|_{\rhote_X} \quad \text { if } r \leq 1.
\end{equation}Moreover, for $1/2 r \leq 1$, we have 
\textcolor{blue}{Where is this from? Cite a reference.}
%Furthermore, for $r \geq 0.5$
\begin{equation}\label{f_lambda_bound}
        \left\|f_{\lambda}\right\|_{\mathcal{H}} \leq \kappa^{-\frac{1}{2}+r}\left\|L^{-r} f_{\mathcal{H}}\right\|_{\rhote_X} \leq \left\|L^{-r} f_{\mathcal{H}}\right\|_{\rhote_X}.
\end{equation}
\end{proposition}

\end{comment}

The following result is available from, e.g.,  \citet[Theorem 1 in Section 3.11]{furuta2001invitation}.
\begin{proposition}[Cordes Inquality] \label{Cordes_Inequality}
Let $A, B$ be positive operators on a Hilbert space. Then for any $s \in[0,1]$  we have
$$
\left\|A^s B^s\right\| \leq\|A B\|^s.
$$
\end{proposition}

\begin{comment}
\begin{assumption}\label{ass:4}
Let $v = d\rho'_X/d\rhotr_X$ be a weighting function. There exist constants $q' \in [0,1]$, $V > 0$ and $\gamma > 0$ such that, for all $m \in \mathbb{N}$ with $m \geq 2$, it holds that  
\begin{equation}\label{condition_on_iw_imperfect}
    \left(\int_X v(x)^{\frac{m-1}{q'}} d\rho'_X(x)\right)^{q'} \leq \frac{1}{2}m!V^{m-2}\gamma^2,
\end{equation}
where the left hand side for $q' = 0$ is defined as the essential supremum of $v$ with respect to $\rho'_X$.

\end{assumption}

\begin{assumption}   \label{ass:5}
There exists a constant $s' \in [0,1]$ such that
\begin{equation}  \label{eq:const-Es_imperfect}
  E'_{s'}:= \max \left(1, \sup _{\lambda \in(0,1]} \sqrt{\mathcal{N}'(\lambda) \lambda^{s'}} \right) < \infty, \quad \text{where}\ \ \mathcal{N}'(\lambda) :=  \operatorname{Tr}\left( T'(T'+\lambda)^{-1} \right).
\end{equation}
\end{assumption}
\end{comment}

\begin{lemma} \label{lemma:bound-1870}
Let $\rho'_X$ be a finite positive measure on $X$, and $v = d\rho'_X/d\rhotr_X$ be the Radon-Nikodym derivative of $\rho'_X$ with respect to the training distribution $\rhotr_X$. 
Suppose that  the projection $f'_\mathcal{H} \in \mathcal{H}$  in \eqref{eq:target-func-gen} satisfies Assumption  \ref{ass:target-exist-gen} with constants $1/2 \leq r' \leq 1$ and $R' > 0$. Then for all $\lambda > 0$ we have
\begin{equation*} 
    \left\|f'_{\lambda}-f'_{\mathcal{H}}\right\|_{\rhote_X} \leq \lambda^{r'} \left\| T (T'+\lambda)^{-1} \right\|^{1/2}   R'
\end{equation*} 
\end{lemma}

\begin{proof}
By Assumption  \ref{ass:target-exist-gen}, there exists $ g   \in L_2(\rho'_X)$ such that $f'_\mathcal{H} = (L')^{r'} g = (T')^{r'-1/2} (L')^{1/2} g$ and $\left\| g \right\|_{\rho'_X} \leq R'$.   
Let   $I_k: \mathcal{H} \mapsto L_2(\rho_X)$ be the embedding operator. 
We then have
\begin{align*}
& \left\| f'_{\mathcal{H}} - f'_{\lambda} \right\|_{\rhote_X} = \left\| \left((T'+\lambda)^{-1} (T' + \lambda)  -(T'+\lambda)^{-1} T'\right) f'_{\mathcal{H}} \right\|_{\rhote_X} = \left\| \lambda (T'+\lambda)^{-1} f'_{\mathcal{H}} \right\|_{\rhote_X} \\ 
& \stackrel{(A)}{=} \left\| L^{1/2} I_k \lambda (T'+\lambda)^{-1} f'_{\mathcal{H}} \right\|_{\mathcal{H}} \stackrel{(B)}{=}
   \left\| T^{1/2}\lambda (T'+\lambda)^{-1} f'_{\mathcal{H}} \right\|_{\mathcal{H}} \\
& = \left\| T^{1/2}\lambda (T'+\lambda)^{-1} (T')^{r'-1/2} (L')^{1/2} g \right\|_{\mathcal{H}} \\
& = \lambda^{r'} \left\| T^{1/2} (T'+\lambda)^{-1/2} \lambda^{1-r'}  (T'+\lambda)^{r' - 1} (T'+\lambda)^{-r' + 1/2} (T')^{r'-1/2} (L')^{1/2} g \right\|_{\mathcal{H}} \\
& \leq \lambda^{r'} \left\| T^{1/2} (T'+\lambda)^{-1/2} \right\| \left\| \lambda^{1-r'}  (T'+\lambda)^{r' - 1} \right\| \left\| (T'+\lambda)^{-r' + 1/2} (T')^{r'-1/2} \right\| \left\| (L')^{1/2} g \right\|_{\mathcal{H}} \\
& \stackrel{(C)}{\leq} \lambda^{r'} \left\| T (T'+\lambda)^{-1} \right\|^{1/2}   \left\| (L')^{1/2} g \right\|_{\mathcal{H}}  \stackrel{(D)}{=} \lambda^{r'} \left\| T (T'+\lambda)^{-1} \right\|^{1/2}   \left\|   g \right\|_{\rho'_X} \\
& \leq \lambda^{r'} \left\| T (T'+\lambda)^{-1} \right\|^{1/2}   R'
\end{align*} 
where $(A)$ follows from $L^{1/2}: L_2(\rho_X) \mapsto \mathcal{H}$ being an isometry, $(B)$ $T^{1/2} = L^{1/2} I_k$, $(C)$ from  Proposition \ref{Cordes_Inequality}, and $(D)$ from $(L')^{1/2}: L_2(\rho'_X) \mapsto \mathcal{H}$ being an isometry.

 \vspace{-5mm}
\end{proof}

\begin{lemma} \label{lemma:ABC}
Let $\mathcal{H}$ be a Hilbert space, $A: \mathcal{H} \mapsto \mathcal{H}$,  $B: \mathcal{H} \mapsto \mathcal{H} $, and $C: \mathcal{H} \mapsto \mathcal{H}$ be bounded, positive, self-adjoint operators, and $g, h \in \mathcal{H}$. Then for all $\lambda > 0$, we have 
\begin{align*}
& \left\| C^{1/2} \left( \left( A +\lambda\right)^{-1} g -(B+\lambda)^{-1} h \right) \right\|_{\mathcal{H}} \\
&  \leq \left\|\  C  (B+\lambda)^{-1} \right\|^{1/2}  \left\|\left(I-(B+\lambda)^{-1/2}\left(B-A\right)(B+\lambda)^{-1/2}\right)^{-1} \right\|  \\
&  \quad \times  \left( \left\|  (B+\lambda)^{-1/2}\left(g-h\right) \right\|_{\mathcal{H}} + \left\| (B+\lambda)^{-1/2}\left(B-A\right) (B+\lambda)^{-1} h   \right\|_{ \mathcal{H} }  \right).
\end{align*}
\end{lemma}

\begin{proof} We have
\begin{align*}
& \left( A +\lambda\right)^{-1} g -(B+\lambda)^{-1} h  \nonumber \\
& =\left(A+\lambda\right)^{-1} \left\{ g -\left(A+\lambda\right)  (B+\lambda)^{-1} h  \right\}  \nonumber \\
& = \left(A+\lambda\right)^{-1} \left\{ g - h + (B+ \lambda) (B+ \lambda)^{-1} h -  \left(A+\lambda\right)  (B+\lambda)^{-1} h  \right\}  \nonumber \\
& = \left(A+\lambda\right)^{-1} \left\{ g - h   +\left(B - A \right)  (B+\lambda)^{-1}h \right\}   \nonumber  \\
& = \left(A+\lambda\right)^{-1}(B + \lambda)^{1/2} \left\{ (B + \lambda)^{-1/2} (g - h)   +  (B + \lambda)^{-1/2}\left(B - A \right)  (B+\lambda)^{-1} h  \right\}   \nonumber \\
& = (B+\lambda)^{-1/2}\left(I-(B+\lambda)^{-1/2}\left(B-A\right)(B+\lambda)^{-1/2}\right)^{-1} \\
& \quad \times\left\{(B+\lambda)^{-1/2}\left(g-h\right)+(B+\lambda)^{-1/2}\left(B-A\right) (B+\lambda)^{-1} h  \right\}, 
\end{align*}
where the last identity follows from 
\begin{align*}
& \left(A+\lambda\right)^{-1}(B + \lambda)^{1/2} = (B + \lambda)^{- 1/2} (B + \lambda)^{1/2}  \left(A+\lambda\right)^{-1}(B + \lambda)^{1/2} \\
& =  (B + \lambda)^{- 1/2} (B + \lambda)^{1/2}  \left(B + \lambda +    A - B \right)^{-1}(B + \lambda)^{1/2}  \\
& =  (B + \lambda)^{- 1/2} \left(I + (B+\lambda)^{-\frac{1}{2}}\left( A - B\right)(B+\lambda)^{-\frac{1}{2}}\right)^{-1} .
\end{align*}
Therefore,
\begin{eqnarray*}  
&& \left\| C^{1/2} \left( \left( A +\lambda\right)^{-1} g -(B+\lambda)^{-1} h \right) \right\|_{\mathcal{H}} \nonumber\\
&&   = \left\|\  C^{1/2} (B+\lambda)^{-1/2}\left(I-(B+\lambda)^{-1/2}\left(B-A\right)(B+\lambda)^{-1/2}\right)^{-1} \right. \nonumber \\
&& \quad \left. \times\left\{(B+\lambda)^{-1/2}\left(g - h\right)+(B+\lambda)^{-1/2}\left(B-A\right) (B+\lambda)^{-1} h  \right\}  \right\|_{ \mathcal{H} } \nonumber \\
&&  \leq \left\|\  C^{1/2} (B+\lambda)^{-1/2} \right\|  \left\|\left(I-(B+\lambda)^{-1/2}\left(B-A\right)(B+\lambda)^{-1/2}\right)^{-1} \right\|  \\
&& \quad \times  \left( \left\|  (B+\lambda)^{-1/2}\left(g-h\right) \right\|_{\mathcal{H}} + \left\| (B+\lambda)^{-1/2}\left(B-A\right) (B+\lambda)^{-1} h   \right\|_{ \mathcal{H} }  \right) \nonumber.  \\
&&  \leq \left\|\  C  (B+\lambda)^{-1} \right\|^{1/2}  \left\|\left(I-(B+\lambda)^{-1/2}\left(B-A\right)(B+\lambda)^{-1/2}\right)^{-1} \right\|  \\
&& \quad \times  \left( \left\|  (B+\lambda)^{-1/2}\left(g-h\right) \right\|_{\mathcal{H}} + \left\| (B+\lambda)^{-1/2}\left(B-A\right) (B+\lambda)^{-1} h   \right\|_{ \mathcal{H} }  \right),
\end{eqnarray*}
where the last inequality follows from Proposition \ref{Cordes_Inequality}.
\vspace{-5mm}
\end{proof}

\begin{lemma} \label{lemma:S-1-bound}
Let $\rho'_X$ be a finite positive measure on $X$, and $v = d\rho'_X/d\rhotr_X$ be the Radon-Nikodym derivative of $\rho'_X$ with respect to the training distribution $\rhotr_X$. 
Define  $T_{\mathbf{x},\mathbf{v}}  : \mathcal{H} \mapsto \mathcal{H}$ and $T': \mathcal{H} \mapsto \mathcal{H}$ by 
$$
T_{\mathbf{x}, \mathbf{v} } f := \frac{1}{n} \sum_{i=1}^n v(x_i) f(x_i) K_{x_i}, \quad T' f := \int K_x f(x) d\rho'_X(x) \quad  (\text{for } f \in \mathcal{H}),  
$$
where $x_1, \dots, x_n \stackrel{i.i.d.}{\sim} \rho_X^{\rm tr}$. 
Suppose that Assumption \ref{ass:4} is satisfied for constants $q' \in [0,1]$, $V > 0$ and  $\gamma > 0$, and that $\| T' \| \leq 1$.
Let  $\delta \in(0,1)$ and $\lambda \in (0, \infty)$.  Then we have, with probability greater than $1-\delta/3$,
\begin{equation}  \label{eq:S1-bound-1816}
S_1 := \left\|(T'+\lambda)^{-\frac{1}{2}}\left(T'-T_{\mathbf{x},\mathbf{v}}\right)(T'+\lambda)^{-\frac{1}{2}}\right\|_{\mathrm{HS}} \leq 4 \left(  \frac{V}{\lambda n }   +   \gamma  \sqrt{ \frac{ \mathcal{N}'(\lambda)^{1-q'} }{ \lambda^{1+q'} n } }  \right) \log\left( \frac{6}{\delta} \right),
\end{equation}
where $\|A\|_{\mathrm{HS}}^2 := \operatorname{Tr}\left(A^{\top} A\right)$ denotes the Hilbert-Schmidt norm and  $\mathcal{N}'(\lambda) := {\rm Tr}( (T'+\lambda)^{-1} T' )$.

Moreover,  if $\lambda \leq 1$ and 
\begin{equation} \label{eq:n-lambda-lowerbound} 
n \lambda^{1+q'} \geq 64 (V+\gamma^2) \mathcal{N}'(\lambda)^{1-q'}  \log^2\left(\frac{6}{\delta}\right),
\end{equation}
then we have $S_1 \leq 3/4$ with probability greater than $1-\delta/3$.

\end{lemma}

\begin{proof}
Denote by ${\rm HS}(\mathcal{H})$ the Hilbert space consisting of Hilbert-Schmidt operators on the RKHS $\mathcal{H}$.
Let $\xi, \xi_1, \dots, \xi_n  \in {\rm HS}(\mathcal{H})$ be  random variables defined as
\vspace{-1mm}
\begin{align*}
& \xi := (T'+\lambda)^{-\frac{1}{2}} v(x) K_{x}\left\langle K_{x}, \cdot\right\rangle_{\mathcal{H}}(T'+\lambda)^{-\frac{1}{2}}, \quad x \sim \rho_X^{\rm tr}, \\
& \xi_i := (T'+\lambda)^{-\frac{1}{2}} v(x_i) K_{x_i}\left\langle K_{x_i}, \cdot\right\rangle_{\mathcal{H}}(T'+\lambda)^{-\frac{1}{2}}, \quad i = 1,\dots, n.
\end{align*}
Then $\xi, \xi_1, \dots, \xi_n$ are i.i.d., and  $S_1$ in the assertion can be written as 
$S_1 = \left\|  \frac{1}{n} \sum_{i=1}^n \xi_i - \mathbb{E} [\xi]  \right\|_F$,
where $F:= {\rm HS}(\mathcal{H})$. 
Therefore, one can bound $S_1$ using Proposition \ref{bernstein}, if the condition \eqref{bernstein_condition} is satisfied.  

We will check the condition \eqref{bernstein_condition}.
To this end, let $m \in \mathbb{N}$ with $m\geq 2$ be arbitrary, and $\xi'$ be an independent copy of $\xi$. Then, we have
\vspace{-1mm}
\begin{align*}
&    E\|\xi -E\xi \|_{F}^m \leq E_{\xi }E_{\xi' } \|\xi -\xi' \|_{F}^m \leq 2^{m-1}E_{\xi }E_{\xi' } \left( \|\xi \|_{F}^m+\|\xi' \|_{F}^m \right) \leq 2^m E\|\xi \|_{F}^m\\
&  = 2^m \int \left\| (T'+\lambda)^{-\frac{1}{2}} v(x) K_{x}\left\langle K_{x}, \cdot\right\rangle_{\mathcal{H}}(T'+\lambda)^{-\frac{1}{2}} \right\|^m_F d\rho_X^{\rm tr}(x) \\
& = 2^m \int \left\| (T'+\lambda)^{-\frac{1}{2}}  K_{x}\left\langle K_{x}, \cdot\right\rangle_{\mathcal{H}}(T'+\lambda)^{-\frac{1}{2}} \right\|^m_F v^{m-1}(x) d\rho'_X(x).
\end{align*}
Let $(e_j)_{j \geq 1} \subset \mathcal{H}$ be an orthonormal basis of $\mathcal{H}$. Then, as $F = {\rm HS}(\mathcal{H})$, we have
\vspace{-1mm}
\begin{align*}
& \left\| (T'+\lambda)^{-\frac{1}{2}}  K_{x}\left\langle K_{x}, \cdot\right\rangle_{\mathcal{H}}(T'+\lambda)^{-\frac{1}{2}} \right\|_F^2 = \sum_{j \geq 1} \left\| (T'+\lambda)^{-\frac{1}{2}}  K_{x}\left\langle K_{x}, \cdot\right\rangle_{\mathcal{H}}(T'+\lambda)^{-\frac{1}{2}} e_j \right\|_\mathcal{H}^2 \\
& = \sum_{j \geq 1} \left\| (T'+\lambda)^{-\frac{1}{2}}  K_{x} \left<K_x, (T'+\lambda)^{-\frac{1}{2}} e_j \right>_{\mathcal{H}} \right\|_\mathcal{H}^2 \leq  \left\| (T'+\lambda)^{-\frac{1}{2}}  K_{x}\right\|_\mathcal{H}^2 \sum_{j \geq 1} \left<K_x, (T'+\lambda)^{-\frac{1}{2}} e_j \right>_{\mathcal{H}}^2  \\
& = \left\| (T'+\lambda)^{-\frac{1}{2}}  K_{x}\right\|_\mathcal{H}^2  \sum_{j \geq 1} \left<(T'+\lambda)^{-\frac{1}{2}}  K_x, e_j \right>_{\mathcal{H}}^2  = \left\| (T'+\lambda)^{-\frac{1}{2}}  K_{x}\right\|_\mathcal{H}^2 \left\| (T'+\lambda)^{-\frac{1}{2}}  K_{x}\right\|_\mathcal{H}^2 \\
& = \left< (T'+\lambda)^{-\frac{1}{2}}  K_{x}, (T'+\lambda)^{-\frac{1}{2}}  K_{x} \right>_\mathcal{H}^2 = \left( K_{x}^\top (T'+\lambda)^{-1}  K_{x} \right)^2.
\end{align*}
Therefore, letting $p':= 1 - q'$, we have
\begin{align*}
& E\|\xi -E\xi \|_{F}^m \leq   2^m \int \left( K_{x}^\top (T'+\lambda)^{-1}  K_{x} \right)^{m} v^{m-1}(x) d\rho_X'(x) \\
& = 2^m \int \left( K_{x}^\top (T'+\lambda)^{-1}  K_{x} \right)^{m-p'} \left( K_{x}^\top (T'+\lambda)^{-1}  K_{x} \right)^{p'} v^{m-1}(x) d\rho'_X(x)  \\
& \leq 2^m \lambda^{ - (m-p') } \int \left( K_{x}^\top (T'+\lambda)^{-1}  K_{x} \right)^{p'} v^{m-1}(x)  d\rho'_X(x) \\
  & \stackrel{(A)}{\leq } 2^m \lambda^{ - (m-p') }  \left( \int K_{x}^\top (T'+\lambda)^{-1}  K_{x} d\rho'_X(x) \right)^{p'} \left( \int v^{(m-1)/q'}(x)  d\rho'_X(x) \right)^{q'} \\
      & \stackrel{(B)}{=}  2^m \lambda^{ - (m-p') }  \left( \int {\rm Tr}(   (T'+\lambda)^{-1}   T_x  ) d\rho'_X(x) \right)^{p'} \left( \int v^{(m-1)/q'}(x)  d\rho'_X(x) \right)^{q'} \\
 & \stackrel{(C)}{=} 2^m \lambda^{ - (m-p') }  \mathcal{N}'(\lambda)^{p'} \left( \int v^{(m-1)/q'}(x)  d\rho'_X(x) \right)^{q'} \\
  & \stackrel{(D)}{=} 2^m \lambda^{ - (m-p') }   \mathcal{N}'(\lambda)^{p'}        \frac{1}{2}m!V^{m-2}\gamma^2  =  \frac{1}{2}m!  2^2 \gamma^2  \mathcal{N}'(\lambda)^{p'}  \lambda^{- 2 + p'}  2^{m-2}      V^{m-2} \lambda^{ - m + 2  }  \\
  & = \frac{1}{2} m!   \left( 2 \gamma  \mathcal{N}'(\lambda)^{p'/2} \lambda^{-1 + p'/2 } \right)^2 \left( 2 V \lambda^{-1} \right)^{m-2} =: \frac{1}{2} m!   \sigma_1^2 L_1^{m-2},  
\end{align*}
where $(A)$ follows from H\"older's inequality and $p'+q' = 1$,  and $\left( \int v^{(m-1)/q'}(x)  d\rho'_X(x) \right)^{q'} :=  \left\| v^{m-1} \right\|_{\infty, \rho'_X}$ for $q' = 0$,   $(B)$ from $K_x^\top (T'+\lambda)^{-1}  K_{x} \geq 0$ and thus 
$K_x^\top (T'+\lambda)^{-1}  K_{x}  = {\rm Tr}( K_x^\top (T'+\lambda)^{-1}  K_{x}  )  =  {\rm Tr}(  (T'+\lambda)^{-1}  T_x ) $, $(C)$ from $\int {\rm Tr}(   (T'+\lambda)^{-1}   T_x  ) d\rho'_X(x) =  {\rm Tr}(   (T'+\lambda)^{-1}  \int T_x  d\rho'_X(x))  =  {\rm Tr}(   (T'+\lambda)^{-1} T') = \mathcal{N}'(\lambda)$,  $(D)$ from Assumption \ref{IW_assumption}, and we defined
$$
\sigma_1 := 2 \gamma  \mathcal{N}'(\lambda)^{p'/2} \lambda^{-1 + p'/2 } = 2 \gamma  \mathcal{N}'(\lambda)^{(1-q')/2} \lambda^{- (1+q')/2 }, \quad  L_1 := 2 V \lambda^{-1}.
$$
 
Therefore,  by  Proposition \ref{bernstein}, with probability greater than $1-\delta/3$, we have
\begin{align*} 
    %S_1 \leq \frac{2L_1 \log(6/\delta)}{n}+\sigma_1 \sqrt{\frac{2\log(6/\delta)}{n}}
    S_1 & \leq  2 \left( L_1 n^{-1} +  \sigma_1 n^{-1/2} \right) \log(6/\delta) \\
    & = 4 \left(  V \lambda^{-1} n^{-1} +   \gamma  \mathcal{N}'(\lambda)^{(1-q')/2} \lambda^{- (1+q')/2 }n^{-1/2} \right) \log(6/\delta),
\end{align*}
which proves \eqref{eq:S1-bound-1816}.

We will prove the second assertion. 
By \eqref{eq:n-lambda-lowerbound}, $\lambda \leq 1$ and $q' \geq 0$, we have 
\begin{align*}
& \lambda^{-1}n^{-1} \leq \lambda^{-1-q'} n^{-1} \leq   64^{-1} (V+\gamma^2)^{-1} \mathcal{N}'(\lambda)^{-(1-q')} \log^{-2}\left(\frac{6}{\delta}\right) , \\
 &  \lambda^{- (1+q')/2 } n^{-1/2}  \leq  8^{-1} (V+\gamma^2)^{-1/2} \mathcal{N}'(\lambda)^{ - (1-q')/2 }  \log^{-1}\left(\frac{6}{\delta}\right).
\end{align*}
Thus, since  $\log^{-1}\left( 6 / \delta\right) \leq \log^{-1}\left( 6 \right) < 1.3$ and $\mathcal{N}'(\lambda) \geq 1/(1+\lambda) \geq 1/2$, which follows from $\| T' \| \leq 1$ and $\lambda \leq 1$,  we have
\begin{align*} 
    S_1 &  \leq     16^{-1} V  (V+\gamma^2)^{-1} \mathcal{N}'(\lambda)^{-(1-q')} \log^{-1}\left(\frac{6}{\delta}\right) +  2^{-1}   \gamma   (V+\gamma^2)^{-1/2}   \\
    & \leq 16^{-1} \cdot 1.3 \cdot 2 + 2^{-1}  < 3/4.
\end{align*}
 
\vspace{-7mm}
\end{proof}

%\[
%T_{\mathbf{x},\mathbf{w}} := S^{\top}_{\mathbf{x}}M_{\mathbf{w}}S_{\mathbf{x}}, 
%\quad g_{\mathbf{z},\mathbf{w}} :=  S^{\top}_{\mathbf{x}}M_{\mathbf{w}}\mathbf{y} \quad 
%\text{and} \quad g := Lf_{\mathcal{H}}
%\]

\begin{lemma} \label{lemma:S-bound-gen}
Let $\rho'_X$ be a finite positive measure on $X$, and $v = d\rho'_X/d\rhotr_X$ be the Radon-Nikodym derivative of $\rho'_X$ with respect to the training distribution $\rhotr_X$. 
Letting $x \sim \rho_X^{\rm tr}$ and $u \in \mathbb{R}$ be another random variable such that $|u| \leq U$ almost surely for a constant $U > 0$, and $(x_i, u_i)_{i=1}^n$ be i.i.d.~copies of $(x, u)$. 
Define  $T': \mathcal{H} \mapsto \mathcal{H}$,  $g_{{\bf z}, {\bf v}} \in \mathcal{H}$ and $g' \in \mathcal{H}$ by
$$
T' f := \int K_{\tilde{x}} f(\tilde{x}) d\rho'_X(\tilde{x}) \quad  (\text{for}\ f \in \mathcal{H}),   \quad g_{\mathbf{z},\mathbf{v}} :=   \frac{1}{n} \sum_{i=1}^n v(x_i) u_i K_{x_i}, \quad g' :=  \mathbb{E} [ v(x) u K_x ].
$$
Let $\mathcal{N}'(\lambda) := {\rm Tr}( (T'+\lambda)^{-1} T' )$.
Suppose that Assumption \ref{ass:4} is satisfied for constants $q' \in [0,1]$, $V > 0$ and  $\gamma > 0$. 
Let  $\delta \in(0,1)$ and $\lambda \in (0, \infty)$.
Then we have, with probability greater than $1-\delta/3$,
\begin{align} \label{eq:bound-S-gen}
S_2 &:=\left\|(T'+\lambda)^{-\frac{1}{2}}\left(g_{\mathbf{z},\mathbf{v}}-g'\right)\right\|_{\mathcal{H}}  \leq 4 U \left(  \frac{ V}{n \sqrt{\lambda}} +   \gamma  \sqrt{\frac{\mathcal{N}'(\lambda)^{1-q'}}{n \lambda^{q'}}}\right)  \log \left(\frac{6}{\delta}\right).
\end{align}
 
\end{lemma}

\begin{proof}
Let $\xi, \xi_1, \dots, \xi_n \in \mathcal{H}$ be random variables defined by
\begin{align*}
& \xi := (T'+\lambda)^{-\frac{1}{2}} v(x) K_{x} u \quad (x \sim \rhotr_X), \quad \xi_i :=  (T'+\lambda)^{-\frac{1}{2}} v(x_i) K_{x_i} u_i, \quad i = 1,\dots, n.
\end{align*}
Then $\xi, \xi_1, \dots, \xi_n$ are i.i.d., and $S$ in the assertion can be written as 
$
S_2 = \left\| \frac{1}{n} \sum_{i=1}^n \xi - E \xi   \right\|_{\mathcal{H}}.
$
Therefore, one can bound $S$ using Proposition \ref{bernstein}, if the condition \eqref{bernstein_condition} is satisfied.  

We will check the condition \eqref{bernstein_condition}.
To this end, let $m \in \mathbb{N}$ with $m\geq 2$ be arbitrary, and $\xi'$ be an independent copy of $\xi$, and $p' := 1 - q'$.
We have
\begin{align*}
&    E\|\xi -E\xi \|_{\mathcal{H}}^m \leq E_{\xi}E_{\xi'} \|\xi - \xi'\|_{\mathcal{H}}^m \leq 2^{m-1} E_{\xi}E_{\xi'} \left( \|\xi \|_{\mathcal{H}}^m+\|\xi' \|_{\mathcal{H}}^m \right)   \\
& \leq 2^m E\|\xi \|_{\mathcal{H}}^m = 2^m E \left\| (T' +\lambda)^{-\frac{1}{2}} v(x) K_{x} u \right\|_{\mathcal{H}}^m \\
& \leq  2^m U^m \int \left\| (T' +\lambda)^{-\frac{1}{2}} K_{x}  \right\|_{\mathcal{H}}^m  v^m(x) d\rhotr(x) \\
& \stackrel{(A)}{=}  2^m U^m \int (K^{\top}_x(T'+\lambda)^{-1}K_x )^{m/2} v^{m-1}(x) d\rho'_X(x)\\
& = 2^m U^m \int ( K^{\top}_x(T'+\lambda)^{-1}K_x )^{m/2-p'} ( K^{\top}_x(T'+\lambda)^{-1}K_x )^{p'} v^{m-1}(x) d\rho'_X(x)\\
& \leq 2^m U^m \lambda^{ - (m/2-p')} \int ( K^{\top}_x(T'+\lambda)^{-1}K_x )^{p'} v^{m-1}(x) d\rho'_X(x)\\
 & \stackrel{(B)}{\leq} 2^{m} U^m \lambda^{ - (m/2-p')} \left(\int   (K^{\top}_x(T'+\lambda)^{-1}K_x)   d\rho'_X(x)\right)^{p'} \left(\int v^{(m-1)/q'}(x) d\rho'_X(x)\right)^{q'}\\
 % & \leq  2^m U^m \left(\frac{1}{\lambda}\right)^{m/2-p} \left(\int \operatorname{Tr}((T+\lambda)^{-1}T_{x})d\rho'_X(x)\right)^{p'} \left(\int w^{(m-1)/q'}(x) d\rho'_X(x)\right)^{q'}\\
& \stackrel{(C)}{=} 2^m U^m \lambda^{ - (m/2-p')} \mathcal{N}'(\lambda)^{p'} \left(\int v^{(m-1)/q'}(x) d\rho'_X(x)\right)^{q'}\\
 & \stackrel{(D)}{\leq}  2^m  U^m \lambda^{ - (m/2-p')} \mathcal{N}'(\lambda)^{p'}   \frac{1}{2}m!V^{m-2}\gamma^2 \\ 
 & =   \frac{1}{2}m! 2^{m-2} M^{m-2} V^{m-2} \lambda^{ - m/2 + 1  } 2^2 M^2 \lambda^{p' -1 } \gamma^2\ \mathcal{N}'(\lambda)^{1-q'}  \\
 & = \frac{1}{2} m! \left(2U V \lambda^{-1/2} \right)^{m-2} \left( 2 U \lambda^{-q'/2} \gamma \mathcal{N}'(\lambda)^{(1-q')/2} \right)^2 =: \frac{1}{2} m! L_2^{m-2}  \sigma_2^2,
%& \leq \frac{1}{2} m! \left(MW\sqrt{\frac{1}{\lambda}} \right)^{m-2}\left(M\sqrt{\frac{1}{\lambda}}^{q}\sqrt{\mathcal{N}(\lambda)}^{p} \sigma \right)^{2},
\end{align*}
where $(A)$ follows from $\left\| (T' +\lambda)^{-\frac{1}{2}} K_{x}  \right\|_{\mathcal{H}}^2 = \left< (T' +\lambda)^{-\frac{1}{2}} K_{x}, (T' +\lambda)^{-\frac{1}{2}} K_{x} \right>_{\mathcal{H}} = K_x^\top  (T' +\lambda)^{-\frac{1}{2}} K_{x} \geq 0$, 
$(B)$ from the H{\"o}lder inequality, and  $\left( \int v^{(m-1)/q'}(x)  d\rho'_X(x) \right)^{q'} :=  \left\| v^{m-1} \right\|_{\infty, \rho'_X}$ for $q' = 0$, $(C)$ from   $K_x^\top (T'+\lambda)^{-1}  K_{x}  = {\rm Tr}( K_x^\top (T'+\lambda)^{-1}  K_{x}  )  =  {\rm Tr}(  (T'+\lambda)^{-1}  T_x ) $ and  $\int {\rm Tr}(   (T'+\lambda)^{-1}   T_x  ) d\rho'_X(x) =  {\rm Tr}(   (T'+\lambda)^{-1}  \int T_x  d\rho'_X(x))  =  {\rm Tr}(   (T'+\lambda)^{-1} T') = \mathcal{N}'(\lambda)$, $(D)$ from  Assumption \ref{ass:4},   and we defined
$$
L_2 := 2U V \lambda^{-1/2}, \quad \sigma_2 := 2 U \lambda^{-q'/2} \gamma \mathcal{N}'(\lambda)^{(1-q')/2}.
$$
The assertion then follows from Proposition \ref{bernstein}.

\end{proof}

\section{Convergence Rates of IW-KRR with a Generic Weighting Function}

We first present a key auxiliary result in Appendix \ref{sec:key-auxilirary-result} and then the proof of Theorem~\ref{main_imperfect}
in Appendix \ref{sec:proof-main-gen}.

\subsection{A Key Auxiliary Result} \label{sec:key-auxilirary-result}

\begin{theorem} \label{theo:bound-stoc-error-gen}
Let $\rho'_X$ be a finite positive measure on $X$, and $v = d\rho'_X/d\rhotr_X$ be the Radon-Nikodym derivative of $\rho'_X$ with respect to the training distribution $\rhotr_X$. 
%Assume that the projection $f'_\mathcal{H} \in \mathcal{H}$  in \eqref{eq:target-func-gen} exists and is unique.
Define  $T_{\mathbf{x},\mathbf{v}}  : \mathcal{H} \mapsto \mathcal{H}$,  $T': \mathcal{H} \mapsto \mathcal{H}$,  $g_{ {\bf z}, {\bf v} } \in \mathcal{H}$ and $g' \in \mathcal{H}$ by
\begin{align*}
& T_{\mathbf{x}, \mathbf{v} } f := \frac{1}{n} \sum_{i=1}^n v(x_i) f(x_i) K_{x_i}, \quad T' f := \int K_x f(x) d\rho'_X(x) \quad  (\text{for}\ f \in \mathcal{H}),   \\
& g_{ {\bf z}, {\bf v} } :=  \frac{1}{n} \sum_{i=1}^n v(x_i) y_i k(\cdot, x_i), \quad g' := \int K_x f_\rho(x) v(x) d\rhotr_X(x), %T' f'_\mathcal{H},
\end{align*}
where $(x_1, y_1) \dots, (x_n, y_n) \stackrel{i.i.d.}{\sim} \rho^{\rm tr}$ and $f_\rho(x) := \int y d\rho(y|x) $  is the regression function. 
Define $f'_{{\bf z}, \lambda} \in \mathcal{H}$ and $f'_\lambda \in \mathcal{H}$ by
$$
f'_{{\bf z}, \lambda} :=   (T_{{\bf x}, {\bf v}} + \lambda)^{-1}  g_{ {\bf z}, {\bf v} } , \quad  f'_\lambda := (T' + \lambda)^{-1} g',
$$
Suppose that Assumptions \ref{ass:4} and \ref{ass:connection_te_'} are  satisfied for constants $q' \in [0,1]$, $V > 0$,   $\gamma > 0$ and $G > 0$, and that $\| T' \| \leq 1$.
 Let  $\delta \in(0,1)$. Assume that  $\lambda \leq 1$ and 
\begin{equation} \label{eq:n-lambda-lowerbound-thm1} 
n \lambda^{1+q'} \geq 64 (V+\gamma^2) \mathcal{N}'(\lambda)^{1-q'}  \log^2\left(\frac{6}{\delta}\right),
\end{equation}
where $\mathcal{N}'(\lambda) := {\rm Tr}( (T'+\lambda)^{-1} T' )$.
Then, with probability greater than $1-\delta$, it holds that
\begin{equation}\label{variance-revised}
\left\| f'_{\mathbf{z}, \lambda}-f'_{\lambda} \right\|_{\rho_X^{\rm te}}   
 \leq 16 G^{1/2}   \left(      M   +   \|f'_{\lambda}\|_{\infty}    \right)  \left(  \frac{ V}{n \sqrt{\lambda}} +   \gamma  \sqrt{\frac{\mathcal{N}'(\lambda)^{1-q'}}{n \lambda^{q'}}}\right) \log \left(\frac{6}{\delta}  \right). 
\end{equation}
\end{theorem}

\begin{proof}
By using Lemma \ref{lemma:ABC} with $A := T_{\mathbf{x},\mathbf{v}}$, $B:=  T'$, $C := T$, $g := g_{{\bf z}, {\bf v}}$ and $h := g'$, we have
\begin{eqnarray}  
&& \left\| f'_{\mathbf{z}, \lambda}-f'_{\lambda} \right\|_{\rho_X^{\rm te}} = \left\| T^{1/2} ( f'_{\mathbf{z}, \lambda}-f'_{\lambda} ) \right\|_{\mathcal{H}} =    \left\| T^{1/2} \left( (T_{{\bf x}, {\bf v}} + \lambda)^{-1}  g_{ {\bf z}, {\bf v} } - (T' + \lambda)^{-1} g' \right) \right\|_{\mathcal{H}} \nonumber\\
%&&   = \left\|\  T^{1/2} (T'+\lambda)^{-1/2}\left(I-(T'+\lambda)^{-1/2}\left(T'-T_{\mathbf{x},\mathbf{v}}\right)(T'+\lambda)^{-1/2}\right)^{-1} \right. \nonumber \\
%&& \quad \left. \times\left\{(T'+\lambda)^{-1/2}\left(g_{\mathbf{z},\mathbf{v}}-g'\right)+(T'+\lambda)^{-1/2}\left(T'-T_{\mathbf{x},\mathbf{v}}\right) f'_{\lambda}\right\}  \right\|_{ \mathcal{H} } \nonumber \\
&&  \leq \left\|\  T  (T'+\lambda)^{-1 } \right\|^{1/2}  \left\|\left(I-(T'+\lambda)^{-1/2}\left(T'-T_{\mathbf{x},\mathbf{v}}\right)(T'+\lambda)^{-1/2}\right)^{-1} \right\| =   \label{eq:bound1922} \\
&& \quad \times  \left( \left\|  (T'+\lambda)^{-1/2}\left(g_{\mathbf{z},\mathbf{v}}-g'\right) \right\|_{\mathcal{H}} + \left\| (T'+\lambda)^{-1/2}\left(T'-T_{\mathbf{x},\mathbf{v}}\right) f'_{\lambda}  \right\|_{ \mathcal{H} }  \right) \nonumber. 
\end{eqnarray}
Note that, by Lemma \ref{lemma:S-1-bound} and \eqref{eq:n-lambda-lowerbound-thm1}, we have, with probability geater than $1-\delta/3$,   
\begin{align}\label{eq:S1_bound_by_one}
 & \left\|(T'+\lambda)^{-\frac{1}{2}}\left(T'-T_{\mathbf{x},\mathbf{v}}\right)(T'+\lambda)^{-\frac{1}{2}}\right\| \\
 & \leq   \left\|(T'+\lambda)^{-\frac{1}{2}}\left(T'-T_{\mathbf{x},\mathbf{v}}\right)(T'+\lambda)^{-\frac{1}{2}}\right\|_{\mathrm{HS}} =: S_1 \leq \frac{3}{4} < 1, \nonumber
\end{align}
where $\|A\|_{\mathrm{HS}}^2 = \operatorname{Tr}\left(A^{\top} A\right)$ denotes the Hilbert-Schmidt norm. Then we have
\begin{align*}
& \left\|\left\{I-(T'+\lambda)^{-\frac{1}{2}}\left(T'-T_{\mathbf{x},\mathbf{v}}\right)(T'+\lambda)^{-\frac{1}{2}}\right\}^{-1}\right\| \stackrel{(A)}{=}\left\|\sum_{j=0}^{\infty}\left[(T'+\lambda)^{-\frac{1}{2}}\left(T-T_{\mathbf{x},\mathbf{v}}\right)(T'+\lambda)^{-\frac{1}{2}}\right]^j\right\|\\
& \leq \sum_{j=0}^{\infty}\left\|(T'+\lambda)^{-\frac{1}{2}}\left(T'-T_{\mathbf{x},\mathbf{v}}\right)(T'+\lambda)^{-\frac{1}{2}}\right\|^j \leq \sum_{j=0}^{\infty}\left\|(T'+\lambda)^{-\frac{1}{2}}\left(T'-T_{\mathbf{x},\mathbf{v}}\right)(T'+\lambda)^{-\frac{1}{2}}\right\|_{\mathrm{HS}}^j\\
& = \sum_{j=0}^\infty S_1^j  \stackrel{(B)}{=}  (1-S_1)^{-1} \leq 4,
\end{align*}
where each of $(A)$ and $(B)$ follows from the Neumann series expansion and \eqref{eq:S1_bound_by_one}.
Note also that  we have  $\left\|  T  (T'+\lambda)^{- 1} \right\|^{1/2} \leq G^{1/2}$ by Assumption \ref{ass:connection_te_'}.
Therefore by \eqref{eq:bound1922}, we have, with probability greater than $1-\delta/3$,
\begin{eqnarray*}
 && \left\| f'_{\mathbf{z}, \lambda}-f'_{\lambda} \right\|_{\rho_X^{\rm te}}   
 \leq 4 G^{1/2}   \left( S_2 + S_3 \right),
\end{eqnarray*}
where
$$
\begin{aligned}
S_{2} &:=\left\|(T'+\lambda)^{-\frac{1}{2}}\left(g_{\mathbf{z},\mathbf{v}}-g'\right)\right\|_{\mathcal{H}}, \quad S_{3} :=\left\|(T'+\lambda)^{-\frac{1}{2}}\left(T'-T_{\mathbf{x},\mathbf{v}}\right) f'_{\lambda}\right\|_{\mathcal{H}}.
\end{aligned}
$$

We use Lemma \ref{lemma:S-bound-gen} to bound $S_2$ and $S_3$.
For $S_2$, Lemma \ref{lemma:S-bound-gen}  can be used by defining $u_i := y_i$ for $i=1,\dots,n$ and $u:= y$ with $(x,y) \sim \rhotr$, and noting that $\mathbb{E}_{(x,y) \sim \rhotr}[ v(x) y K_x ] = \mathbb{E}_{x \sim \rhotr_X}[ v(x) f_\rho(x) K_x ] = g'$; thus the bound \eqref{eq:bound-S-gen} holds with $U = M$ with probability greater than $1-\delta/3$.
For $S_3$, Lemma \ref{lemma:S-bound-gen}  can be used by defining $u_i := f'_\lambda(x_i)$ for $i=1,\dots,n$ and $u:= f'_\lambda(x)$ with $x \sim \rhotr_X$, and letting $g_{\bf z, v} := T_{ {\bf x}, {\bf v} } f'_\lambda $ and $g' := T' f'_\lambda$;  thus the bound \eqref{eq:bound-S-gen} holds with $U = \left\| f'_\lambda \right\|_\infty$ with probability greater than $1-\delta/3$.
 Therefore, we have, with probability greater than $1- \delta$, 
\begin{eqnarray*}
 && \left\| f'_{\mathbf{z}, \lambda}-f'_{\lambda} \right\|_{\rho_X^{\rm te}}   
 \leq 16 G^{1/2}   \left(      M   +   \|f'_{\lambda}\|_{\infty}    \right)  \left(  \frac{ V}{n \sqrt{\lambda}} +   \gamma  \sqrt{\frac{\mathcal{N}'(\lambda)^{1-q'}}{n \lambda^{q'}}}\right) \log \left(\frac{6}{\delta}  \right),
\end{eqnarray*}
which concludes the proof.
\vspace{-5mm}
 
\end{proof}

\vspace{-5mm}

\subsection{Proof of Theorem~\ref{main_imperfect}}
\label{sec:proof-main-gen}

%\subsection{Upper bound for arbitrarily weighted KRR}
\begin{comment}
For the proof we need the following proposition from \cite[Proposition 8]{rudi2017generalization}
\begin{proposition} \label{mixed_operators}
Let $\mathcal{H}$ be a separable Hilbert space, let $A, B$ two bounded self-adjoint positive linear operators on $\mathcal{H}$ and $\lambda>0$. Then
$$
\left\|(A+\lambda I)^{-1 / 2} B^{1 / 2}\right\| \leq\left\|(A+\lambda I)^{-1 / 2}(B+\lambda I)^{1 / 2}\right\| \leq(1-\mu)^{-1 / 2}
$$
with
$$
\mu=\mu_{\max }\left[(B+\lambda I)^{-1 / 2}(B-A)(B+\lambda I)^{-1 / 2}\right]
$$
\end{proposition}

\end{comment}

\begin{proof}
By the triangle inequality, we have 
\begin{align} \label{eq:bound-decomp-2158}
\left\| f_{\mathbf{z},\lambda}'-f_{\mathcal{H}} \right\|_{\rho_X^{\rm te}} \leq \left\| f_{\mathbf{z},\lambda}'-f'_{\lambda} \right\|_{\rho_X^{\rm te}} + \left\| f'_{\lambda} - f'_{\mathcal{H}} \right\|_{\rho_X^{\rm te}} + \left\| f'_{\mathcal{H}}-f_{\mathcal{H}} \right\|_{\rho_X^{\rm te}}, 
\end{align}
where $f'_\lambda \in \mathcal{H}$ is defined in Theorem \ref{theo:bound-stoc-error-gen}. 
 By Lemma \ref{lemma:bound-1870} and  Assumption \ref{ass:connection_te_'}, we have
 \begin{align*}
     \left\|f'_{\lambda}-f'_{\mathcal{H}}\right\|_{\rhote_X} \leq \lambda^{r'} \left\| T (T'+\lambda)^{-1} \right\|^{1/2}   R'  \leq \lambda^{r'} G^{1/2}   R'.
 \end{align*}
Note that we have 
$$ 
\|f'_{\lambda}\|_{\infty} \leq \| f'_\lambda \|_{\mathcal{H}} = \| (T' + \lambda)^{-1} T' f'_\mathcal{H} \|_{\mathcal{H}}  \leq \| f'_\mathcal{H} \|_{\mathcal{H}}.
$$
 By Assumption \ref{ass:5}, we also have 
 \begin{equation} \label{eq:2226}
  \mathcal{N}'(\lambda)   \leq \lambda^{-s'} (E'_{s'})^2.
 \end{equation} 
Therefore, by Theorem \ref{theo:bound-stoc-error-gen}, with probability greater than $1 - \delta$, we have
 \begin{align*} 
\left\| f'_{\mathbf{z}, \lambda}-f'_{\lambda} \right\|_{\rho_X^{\rm te}}   
& \leq 16 G^{1/2}   \left(      M   +   \|f'_{\lambda}\|_{\infty}    \right)  \left( Vn^{-1} \lambda^{-1/2}   +   \gamma  \mathcal{N}'(\lambda)^{(1-q')/2} n^{-1/2} \lambda^{-q'/2} \right) \log \left( 6/\delta \right) \\
 & \leq 16 G^{1/2}   \left(      M   +  \| f'_\mathcal{H} \|_{\mathcal{H}}   \right)  \left( Vn^{-1} \lambda^{-1/2}   +   \gamma \lambda^{-s' (1-q') / 2} (E'_{s'})^{1-q'} n^{-1/2} \lambda^{-q'/2} \right) \log \left( 6/\delta \right) \\
 & = 16 G^{1/2}   \left(      M   +  \| f'_\mathcal{H} \|_{\mathcal{H}}   \right)  \left( Vn^{-1} \lambda^{-1/2}   +   \gamma  (E'_{s'})^{1-q'} n^{-1/2} \lambda^{- [s' (1-q')  + q' ]/2}   \right) \log \left( 6/\delta \right),
\end{align*}
provided that the condition \eqref{eq:n-lambda-lowerbound-thm1} is satisfied:    
  \begin{equation}  \label{eq:condition-2238}
n \lambda^{1+q'} \geq 64 (V+\gamma^2) \mathcal{N}'(\lambda)^{1-q'}  \log^2\left( 6/\delta \right).
\end{equation}  
Let $c > 0$ and $0 < \beta < 1$, and define 
$$
\lambda = c n^{- \beta}. 
$$
We assume that $c$ and $\beta$ are such that \eqref{eq:condition-2238} is satisfied, and that $n$ is large enough so that $\lambda \leq 1$.  We will determine concrete values of $c$ and $\beta$ later.

Define 
$$
A := s' (1-q')  + q'.
$$  
Thus, we have, with probability greater than $1-\delta$, 
\begin{align}
& \left\| f'_{\mathbf{z}, \lambda}-f'_{\lambda} \right\|_{\rho_X^{\rm te}} +  \left\|f'_{\lambda}-f'_{\mathcal{H}}\right\|_{\rhote_X}  \nonumber \\
&  \leq  16 G^{1/2}   \left(      M   +  \| f'_\mathcal{H} \|_{\mathcal{H}}   \right)  \left( Vn^{-1} \lambda^{-1/2}   +   \gamma  (E'_{s'})^{1-q'} n^{-1/2} \lambda^{-A/2}   \right) \log \left( 6/\delta \right) + \lambda^{r'} G^{1/2}   R'  \nonumber \\
&  \leq  16 G^{1/2}   \left(      M   +  \| f'_\mathcal{H} \|_{\mathcal{H}}  \right)  \left( V  c^{-1/2} n^{- (2 -\beta)/2 }   +   \gamma  (E'_{s'})^{1-q'}   c^{-A/2} n^{ - (1 - A\beta)/2}   \right) \log \left( 6/\delta \right) + c^{r'} n^{- r' \beta}  G^{1/2}   R'. \label{eq:bound2255}
\end{align}
First, we determine $\beta$ to balance the rates of the above three terms: 
$$
n^{- (2-\beta)/2 },\quad  n^{- (1 - A\beta )/2} \quad \text{and}\quad n^{- r' \beta}
$$
The rate of the first term is faster than the second term, because
$$
(2-\beta)/2 > 1/2 >  (1-A\beta)/2,
$$
which follows from $0< \beta < 1$ and  $0 \leq A \leq 1$. 
Therefore we determine $\beta$ to balance the rates of the second and third terms: $ - (1- A\beta) /2 = - r'\beta$. This leads to 
\begin{equation} \label{eq:beta-2267}
\beta = \frac{1}{2r' + A} = \frac{1}{2r' +  s' (1-q')  + q'}.
\end{equation}
We now determine $c$ to satisfy the condition \eqref{eq:condition-2238}. By $\lambda = c n^{-1/(2r' + A)}$ and \eqref{eq:2226}, we have
\begin{align*}
 & 64 (V+\gamma^2) \mathcal{N}'(\lambda)^{1-q'}  \log^2\left( 6/\delta \right) \\
& \leq   64 (V+\gamma^2) \lambda^{-s' (1-q')} (E'_{s'})^{2(1-q')}  \log^2\left( 6/\delta \right) \\
 & = 64 (V+\gamma^2) c^{-s' (1-q')}  n ^{s' (1-q') / (2r' + A)} (E'_{s'})^{2(1-q')}  \log^2\left( 6/\delta \right).
\end{align*}
Therefore, the condition \eqref{eq:condition-2238} is satisfied if 
\begin{align*}
& 64 (V+\gamma^2) c^{-s' (1-q')}  n ^{s' (1-q') / (2r' + A)} (E'_{s'})^{2(1-q')}  \log^2\left( 6/\delta \right) \\ &\quad \leq n \lambda^{1+q'} =  c^{1+q'} n^{  (2r' + A - 1- q') / (2r' + A) } \\
\Longleftrightarrow \ & 64 (V+\gamma^2)  (E'_{s'})^{2(1-q')}  \log^2\left( 6/\delta \right)  n^{  - (2r' - 1  ) / (2r' + A) }  \\
& \quad \leq c^{1+q' + s'(1-q')} n^{  \left(2r' + A - 1- q' - s' (1-q')  - 2 (r'-1) \right) / (2r' + A) } = c^{1+A} \\
\stackrel{(*)}{\Longleftarrow} \ & 64 (V+\gamma^2)  (E'_{s'})^{2(1-q')}  \log^2\left( 6/\delta \right)  
  \leq  c^{1+A}.
\end{align*}
where $(*)$ follows from $r' \geq 1/2$. 
Therefore, the condition \eqref{eq:condition-2238} is satisfied if $c$ satisfies 
$$
c \geq \left( 64 (V+\gamma^2)  (E'_{s'})^{2(1-q')}  \log^2\left( 6/\delta \right)  \right)^{1/(1+A)}.
$$

By \eqref{eq:bound2255},  $c^{-1/2} \leq c^{-A/2}$ and $n^{- (2 -\beta)/2 } \leq  n^{ - (1 - A\beta)/2} = n^{- r'\beta}$ with $\beta$ in \eqref{eq:beta-2267}, we have 
\begin{align*}
& \left\| f'_{\mathbf{z}, \lambda}-f'_{\lambda} \right\|_{\rho_X^{\rm te}} +  \left\|f'_{\lambda}-f'_{\mathcal{H}}\right\|_{\rhote_X}  \nonumber \\
&  \leq  16 G^{1/2}   \left(      M   +  \| f'_\mathcal{H} \|_{\mathcal{H}}   \right)  \left( V   +   \gamma  (E'_{s'})^{1-q'}   \right) \log \left( 6/\delta \right)  c^{-A/2} n^{- r' \beta}   + c^{r'} n^{- r' \beta}  G^{1/2}   R' \\
& =   n^{- r' \beta} G^{1/2} \left\{ 16  \left(      M   +  \| f'_\mathcal{H} \|_{\mathcal{H}}   \right)  \left( V   +   \gamma  (E'_{s'})^{1-q'}   \right) \log \left( 6/\delta \right)  c^{-A/2}    + c^{r'}    R' \right\} .
%& = n^{  - \frac{r'}{2r' +  s' (1-q')  + q'} } G^{1/2}  \left\{ 16  \left(      M   +  R'   \right)  \left( V   +   \gamma  (E'_{s'})^{1-q'}   \right) \log \left( 6/\delta \right)  c^{- ( s' (1-q')  + q' )  /2}    + c^{r'}    R' \right\}
\end{align*}
The assertion follows from this and \eqref{eq:bound-decomp-2158}.

\vspace{-5mm}
\end{proof}

\vspace{-5mm}

\section{Convergence Rates of Clipped IW-KRR}\label{trucated_KRR_proof}

Appendix \ref{sec:lemmas-for-clipped-IWKRR} contains lemmas required for proving Theorem \ref{theo:clipped-KRR-re},  and Appendix \ref{sec:proof-clipped-KRR-main} proves Theorem \ref{theo:clipped-KRR-re}.  

We first define the notation used in this section.  
For the IW function $w = d\rhote/d\rhotr$ and $D > 0$, let $w_D(x) := \min( w(x), D)$ be the clipped IW function.  
Define operators  $T_D : \mathcal{H} \mapsto \mathcal{H}$ and $L_D: L_2(\rhote_X) \mapsto \mathcal{H}$  by
\begin{align}
  T_D f &:= \int K_x f(x)w_D(x) d\rhotr(x) \nonumber \\
  & = \int K_x \left<K_x, f \right>_{\mathcal{H}} w_D(x) d\rhotr(x) = \int (T_x f) w_D(x) d\rhotr(x) \quad (\text{for}\ f \in \mathcal{H}), \label{eq:TD-def-2608} \\
  L_D f &:= \int K_x f(x)w_D(x) d\rhotr(x) \quad (\text{for}\ f \in L_2(\rhote_X)). \label{eq:LD-def-2609}
\end{align}

\subsection{Lemmas}
\label{sec:lemmas-for-clipped-IWKRR}

We present the lemmas required for proving the main theorem. 
\vspace{-1mm}

\begin{lemma} \label{lemma:bound-clippedIW-2601}
Suppose Assumption \ref{IW_assumption} holds with constants $q \in (0,1]$, $W \in (0,\infty)$ and $\sigma \in (0, \infty)$. Then for all $m \in \mathbb{N}$ with $m\geq 2$  and $D > 0$, we have
\begin{align} \label{eq:2603}
& \int  (1 - w_D(x) / w(x) )^2   d\rho_X^{\rm te}(x) \leq \left( 2^{-1}  D^{- (m-1)   } m! W^{m - 2} \sigma^2 \right)^{1/q}.
\end{align}
%Moreover, if Assumption \ref{ass:effective-dim} holds with constants $s \in [0,1]$ and $E_s \in [1,\infty)$ and $D$ satisfies
%\begin{align} \label{eq:2607}
%D \geq  \left( m! W^{m - 2} \sigma^2   E_s \lambda^{-(1+s)/2} \right)^{q / (m-1)},
%\end{align}
%then we have 
%\begin{align} \label{eq:2611}
%& \int  (1 - w_D(x) / w(x) )^2   d\rho_X^{\rm te}(x) \leq \frac{1}{2} E_s^{-1} \lambda^{(1+s)/2}.
%\end{align}

\end{lemma}

\begin{proof}
We have 
\begin{align*}
&  \int  (1 - w_D(x) / w(x) )^2   d\rho_X^{\rm te}(x)  =  \int_{w \geq D}  (1 - D / w(x) )^2   d\rho_X^{\rm te}(x) \leq  \int_{w \geq D}  1   d\rho_X^{\rm te}(x) \\
& = \rho_X^{\rm te} ( \{ x \in X: \ w(x) \geq D \} )  = \rho_X^{\rm te} ( \{ x \in X: \ w^{(m-1) / q }(x) \geq D^{(m-1) / q } \} ) \\
& \stackrel{(A)}{\leq}  D^{- (m-1) / q } \int w^{(m-1) / q }(x) d\rho^{\rm te}(x) \stackrel{(B)}{\leq}  \left(2^{-1}  D^{- (m-1)  } m! W^{m - 2} \sigma^2 \right)^{1/q},
\end{align*}
where $(A)$ follows from Markov's inequality and $(B)$ from Assumption \ref{IW_assumption}. %This proves \eqref{eq:2603}. 
%The second assertion \eqref{eq:2611} follows from \eqref{eq:2603} and  \eqref{eq:2607}.
\end{proof}

\begin{lemma} \label{lemma:bound-TD-2494}
 Suppose Assumption \ref{IW_assumption} holds with constants $q \in (0,1]$, $W \in (0,\infty)$ and $\sigma \in (0, \infty)$. Let $m \in \mathbb{N}$ with $m\geq 2$,   $\lambda > 0$ and $D > 0$.  
 Then  we have

\begin{align*}
\left\|(T - T_D) (T +\lambda)^{-1}\right\|  \leq \lambda^{-1/2} \mathcal{N}(\lambda)^{1/2}  \left( 2^{-1}  D^{- (m-1)   } m! W^{m - 2} \sigma^2 \right)^{1/2q} .
% \frac{m! W^{m - 2} \sigma^2}{2} \lambda^{-1/2} \mathcal{N}(\lambda)^{1/2}   D^{- (m-1) / q } 
\end{align*}
\end{lemma}

\begin{proof}
First, for any $f \in \mathcal{H}$, we have
\begin{align*}
& \left\| \int ( T_x  (T +\lambda)^{-1} f)  (w(x) - w_D(x) ) d\rho_X^{\rm tr}(x) \right\|_{ \mathcal{H} } \\
& = \left\| \int  K_x \left< K_x,  (T +\lambda)^{-1} f \right>_{\mathcal{H}} (w(x) - w_D(x) ) d\rho_X^{\rm tr}(x) \right\|_{ \mathcal{H} } \\
& = \left\| \int  K_x \left<  (T +\lambda)^{-1} K_x,  f \right>_{\mathcal{H}} (w(x) - w_D(x) ) d\rho_X^{\rm tr}(x) \right\|_{ \mathcal{H} } \\
& \stackrel{(A)}{\leq} \left\|  f  \right\|_{\mathcal{H}}  \int   \left\|  (T +\lambda)^{-1} K_x \right\|_{\mathcal{H}} (w(x) - w_D(x) ) d\rho_X^{\rm tr}(x) \\ 
& = \left\|  f  \right\|_{\mathcal{H}}  \int   \left\|  (T +\lambda)^{-1} K_x \right\|_{\mathcal{H}} (w(x) - w_D(x) ) d\rho_X^{\rm tr}(x)   \\
& \stackrel{(B)}{\leq}  \left\|  f  \right\|_{\mathcal{H}}   \lambda^{-1/2} \int    \left(  K_x^\top (T+ \lambda)^{-1} K_x \right)^{1/2}  (w(x) - w_D(x) ) d\rho_X^{\rm tr}(x), 
\end{align*}
where $(A)$ follows from Cauchy-Schwartz, $\| K_x \| = \sqrt{K(x,x)} \leq 1$ and $w(x) - w_D(x) \geq 0$,  and $(B)$ from
\begin{align*}
& \left\|  (T +\lambda)^{-1} K_x \right\|_{\mathcal{H}} \leq   \left\|  (T +\lambda)^{-1/2} \right\| \left\|  (T +\lambda)^{-1/2} K_x \right\|_{\mathcal{H}}  \leq    \lambda^{-1/2}   \left\|  (T +\lambda)^{-1/2} K_x \right\|_{\mathcal{H}} \\
& = \lambda^{-1/2}  \left<  (T +\lambda)^{-1/2} K_x, (T +\lambda)^{-1/2} K_x \right>_{\mathcal{H}}^{1/2}   =  \lambda^{-1/2}  \left(  K_x^\top (T+ \lambda)^{-1} K_x \right)^{1/2} .
\end{align*}
Therefore,  we have
\begin{align*}  
& \left\|(T - T_D) (T +\lambda)^{-1}\right\| 
= \left\| \int T_x (T +\lambda)^{-1} (w(x) - w_D(x) ) d\rho_X^{\rm tr}(x) \right\|  \\
& = \sup_{ \left\| f \right\|_{ \mathcal{H} } \leq 1 } \left\| \int ( T_x (T +\lambda)^{-1} f)  (w(x) - w_D(x) ) d\rho_X^{\rm tr}(x) \right\|_{ \mathcal{H} } \\
& \leq     \lambda^{-1/2} \int    \left(  K_x^\top (T+ \lambda)^{-1} K_x \right)^{1/2}  (w(x) - w_D(x) ) d\rho_X^{\rm tr}(x)  \\
& = \lambda^{-1/2} \int    \left(  K_x^\top (T+ \lambda)^{-1} K_x \right)^{1/2}  (w(x) - w_D(x) ) / w(x) d\rho_X^{\rm te}(x)  \\
& \stackrel{(A)}{\leq}  \lambda^{-1/2} \left( \int K_x^\top (T+ \lambda)^{-1} K_x   d\rho_X^{\rm te}(x) \right)^{1/2} \left(  \int  (w(x) - w_D(x) )^2 / w^2(x)  d\rho_X^{\rm te}(x)  \right)^{1/2} \\
& \stackrel{(B)}{=}  \lambda^{-1/2} \left( \int  {\rm Tr} ((T+ \lambda)^{-1} T_x )   d\rho_X^{\rm te}(x) \right)^{1/2} \left(  \int  (1 - w_D(x) / w(x) )^2   d\rho_X^{\rm te}(x)  \right)^{1/2} \\
& =  \lambda^{-1/2} \left( {\rm Tr} ((T+ \lambda)^{-1} T )    \right)^{1/2} \left(  \int  (1 - w_D(x) / w(x) )^2   d\rho_X^{\rm te}(x)  \right)^{1/2}
\end{align*}  
where $(A)$ follows from Cauchy-Schwartz, and $(B)$ from $ K_x^\top (T+ \lambda)^{-1} K_x = {\rm Tr}( K_x^\top (T+ \lambda)^{-1} K_x ) = {\rm Tr} ((T+ \lambda)^{-1} T_x )$. 

By Lemma \ref{lemma:bound-clippedIW-2601}, for all $m \geq 2$, we have 
\begin{align*}
&  \int  (1 - w_D(x) / w(x) )^2   d\rho_X^{\rm te}(x) \leq  \left( 2^{-1}  D^{- (m-1)   } m! W^{m - 2} \sigma^2 \right)^{1/q}.   
\end{align*}
Therefore, 
\begin{align*}
\left\|(T - T_D) (T +\lambda)^{-1}\right\|  \leq  \lambda^{-1/2} \mathcal{N}(\lambda)^{1/2}  \left( 2^{-1}  D^{- (m-1)   } m! W^{m - 2} \sigma^2 \right)^{1/2q}.
%D^{- (m-1) / q } m! W^{m - 2} \sigma^2
\end{align*}
\vspace{-8mm}
\end{proof}

\begin{lemma} \label{lemma:bound-TD-2518} 
Suppose that Assumptions \ref{IW_assumption} and \ref{ass:effective-dim} hold with constants $q \in (0,1]$, $W \in (0,\infty)$, $\sigma \in (0, \infty)$, $s \in [0,1]$ and $E_s \in (0, \infty)$. For arbitrary $m \in \mathbb{N}$ with $m \geq 2$ and $\lambda > 0$, let $D > 0$ be such that
\begin{equation} \label{eq:cond-D-trun-2688}
D   \geq    (2^{2q-1} E_s^{2q} m! W^{m - 2} \sigma^2 )^{1/(m-1)}   \lambda^{-(1+s)q/(m-1)} .    
\end{equation}
 Then  we have
\begin{align*}
(i)\ \ \left\| (T - T_D) (T+ \lambda)^{-1} \right\| \leq  \frac{1}{2} \quad \text{and} \quad  (ii)\ \ \left\| T  (T_D + \lambda)^{-1} \right\| \leq 2.
\end{align*}
\end{lemma}

\begin{proof}
By Assumption \ref{ass:effective-dim}, we have
$\mathcal{N}(\lambda) \leq E_s^2 \lambda^{-s}$.  
By Lemma \ref{lemma:bound-TD-2494},  for all $m \in \mathbb{N}$ with $m \geq 2$, we have
\begin{align*}
\left\|(T - T_D) (T +\lambda)^{-1}\right\| 
& \leq \lambda^{-1/2} \mathcal{N}(\lambda)^{1/2}  \left( 2^{-1}  D^{- (m-1)   } m! W^{m - 2} \sigma^2 \right)^{1/2q} \\
& \leq E_s \lambda^{-(1+s)/2}  \left( 2^{-1}  D^{- (m-1)   } m! W^{m - 2} \sigma^2 \right)^{1/2q} .
\end{align*}

Thus, assertion (i) holds if
\begin{align*}
 &   E_s \lambda^{-(1+s)/2}  \left( 2^{-1}  D^{- (m-1)   } m! W^{m - 2} \sigma^2 \right)^{1/2q} \leq 2^{-1} \\
\Longleftrightarrow \quad &       D^{- (m-1)   }  \leq  2  (m! W^{m - 2} \sigma^2 )^{-1} \left( 2^{-1} E_s^{-1} \lambda^{(1+s)/2} \right)^{2q} \\
\Longleftrightarrow \quad &       D   \geq    (2^{2q-1} E_s^{2q} m! W^{m - 2} \sigma^2 )^{1/(m-1)}   \lambda^{-(1+s)q/(m-1)}.
\end{align*}

We next prove assertion (ii). Note that
\begin{align*}
& (T_D+\lambda)^{-1} = (T+ \lambda - [T - T_D] )^{-1} = \left\{ (I - [T - T_D] [T+ \lambda]^{-1} ) (T+ \lambda) \right\}^{-1} \\
&  = (T+ \lambda)^{-1} (I - [T - T_D] [T+ \lambda]^{-1} )^{-1}.
\end{align*}
Thus, we have
\begin{align*}
& \left\| T  (T_D + \lambda)^{-1} \right\|  = \left\| T (T+ \lambda)^{-1} (I - [T - T_D] [T+ \lambda]^{-1} )^{-1} \right\| \\
& \leq \left\| T (T+ \lambda)^{-1} \right\| \left\| (I - [T - T_D] [T+ \lambda]^{-1} )^{-1} \right\| \\
& \leq \left\|  (I - [T - T_D] [T+ \lambda]^{-1} )^{-1} \right\| \stackrel{(A)}{=} \left\|   \sum_{j=0}^\infty \left( [T - T_D] [T+ \lambda]^{-1} \right)^{j}   \right\|  \\
& \leq   \sum_{j=0}^\infty  \left\| [T - T_D] [T+ \lambda]^{-1} \right\|^{j} \stackrel{(B)}{=} \left( 1 -   \left\| [T - T_D] [T+ \lambda]^{-1} \right\|  \right)^{-1} \stackrel{(C)}{\leq} 2.
\end{align*}
where each of $(A)$ and $(B)$ follows from the Neumann series expansion and the assertion (i), and $(C)$ from assertion (i).
\vspace{-5mm}
\end{proof}

\begin{lemma} \label{lemma:TDLDfrhofH}
Suppose that Assumptions \ref{ass:target-exist}, \ref{IW_assumption} and \ref{ass:effective-dim} hold with constants $q \in (0,1]$, $W \in (0,\infty)$, $\sigma \in (0, \infty)$, $s \in [0,1]$ and $E_s \in (0, \infty)$. 
%Let $w = d\rhote_X / d\rhotr_X$ be the importance weighting function, and  $w_D $  be the clipped importance weighting function at threshold $D > 0$: $w_D(x) := \min (w(x), D)$ for all $x \in X$.  Suppose that  the projection $f_\mathcal{H} \in \mathcal{H}$ in \eqref{eq:target-func} exists and is unique. 
Let $f_\rho \in L_2(\rhote_X)$ be the regression function \eqref{eq:regress} and   $f_\mathcal{H} \in \mathcal{H}$ be the projection in \eqref{eq:target-func}.  
For arbitrary $m \in \mathbb{N}$ with $m \geq 2$ and $\lambda > 0$, let $D > 0$ be such that \eqref{eq:cond-D-trun-2688} is satisfied.
%Define $T_D : \mathcal{H} \mapsto \mathcal{H}$ and $L_D: L_2(\rhote_X) \mapsto \mathcal{H}$  by $T_D f := \int K_x f(x)w_D(x) d\rhotr(x)$ for $f \in \mathcal{H}$ and $L_D f := \int K_x f(x)w_D(x) d\rhotr(x)$ for $f \in L_2(\rhote_X)$ with $w_D(x) := \min( w(x), D)$.
Then we have
\begin{align*}
& \left\| (T_D + \lambda)^{-1/2} L_D f_\rho - (T_D + \lambda)^{-1/2} T_D f_\mathcal{H}\right\|_{\rhote_X} \\
& \leq  \left(  \| f_\rho \|_{\rhote_X} + \| f_\mathcal{H} \|_{\rhote_X}  \right)    \left( 2^{q-1}  D^{- (m-1)   } m! W^{m - 2} \sigma^2 \right)^{1/2q}.
\end{align*}
\end{lemma}

\begin{proof} 
For all $m \geq 2$,  we have
\begin{align*}
& \left\| L_D f_\rho  - L f_\rho \right\|_{\mathcal{H}}  = \left\| \int K_x f_\rho(x) (w_D(x) - w(x)) d\rhotr_X(x)\right\|_{\mathcal{H}} \\
& \leq  \int  \left\| K_x \right\|_{\mathcal{H}} | f_\rho(x)|  w_D(x) - w(x)| d\rhotr_X(x)   \leq  \int   | f_\rho(x)|  | w_D(x) - w(x)| d\rhotr_X(x)  \\
& = \int   | f_\rho(x) |  |w_D(x) / w(x) - 1| d\rhote_X(x) \\
& \leq \left( \int   f^2_\rho(x)  d\rhote_X(x) \right)^{1/2} \left( \int  (w_D(x) / w(x) - 1)^2 d\rhote_X(x) \right)^{1/2} \leq \| f_\rho \|_{\rhote_X}    \left( 2^{-1}  D^{- (m-1)   } m! W^{m - 2} \sigma^2 \right)^{1/2q}
\end{align*}
where the last inequality follows from Lemma \ref{lemma:bound-clippedIW-2601}. 
Therefore we have
\vspace{-1mm}
\begin{align}
& \left\| (T_D + \lambda)^{-1/2} (L_D f_\rho - L f_\rho)  \right\|_{\rhote_X} = \left\| T^{1/2} (T_D + \lambda)^{-1/2} (L_D f_\rho - L f_\rho)  \right\|_{\mathcal{H}} \nonumber  \\
%& = \lambda^{-1/2} \left\| T^{1/2} (T_D + \lambda)^{-1/2} (T_D + \lambda)^{-1/2} \lambda^{1/2} (L_D f_\rho - L f_\rho)  \right\|_{\mathcal{H}}  \nonumber \\
& \leq    \left\| T^{1/2} (T_D + \lambda)^{-1/2} \right\|  \left\|  (L_D f_\rho - L f_\rho)  \right\|_{\mathcal{H}} \nonumber \\
& \leq   2^{1/2}    \| f_\rho \|_{\rhote_X} \left( 2^{-1}  D^{- (m-1)   } m! W^{m - 2} \sigma^2 \right)^{1/2q}  =     \| f_\rho \|_{\rhote_X} \left( 2^{q-1}  D^{- (m-1)   } m! W^{m - 2} \sigma^2 \right)^{1/2q},  \label{eq:bound-2810}  
\end{align} 
where we used Lemma \ref{lemma:bound-TD-2518} and Proposition \ref{Cordes_Inequality} in the last inequality. 
Similarly, we have
\vspace{-1mm}
\begin{align} \label{eq:bound-2815}
\left\| (T_D + \lambda)^{-1/2} (T f_\mathcal{H} - T_D f_\mathcal{H})  \right\|_{\rhote_X} \leq   2^{1/2}    \| f_\mathcal{H} \|_{\rhote_X} \left( 2^{-1}  D^{- (m-1)   } m! W^{m - 2} \sigma^2 \right)^{1/2q}.
\end{align}
Now, we have 
\vspace{-1mm}
\begin{align*}
& \left\| (T_D + \lambda)^{-1/2} L_D f_\rho - (T_D + \lambda)^{-1/2} T_D f_\mathcal{H}\right\|_{\rhote_X} \\
& \leq \left\| (T_D + \lambda)^{-1/2} (L_D f_\rho - L f_\rho)  \right\|_{\rhote_X} + \left\| (T_D + \lambda)^{-1/2} (L f_\rho - T  f_\mathcal{H})  \right\|_{\rhote_X}  + \left\| (T_D + \lambda)^{-1/2} (T f_\mathcal{H} - T_D f_\mathcal{H})  \right\|_{\rhote_X} \\
& \stackrel{(A)}{=} \left\| (T_D + \lambda)^{-1/2} (L_D f_\rho - L f_\rho)  \right\|_{\rhote_X}  + \left\| (T_D + \lambda)^{-1/2} (T f_\mathcal{H} - T_D f_\mathcal{H})  \right\|_{\rhote_X}  \\
& \stackrel{(B)}{\leq}    \left(  \| f_\rho \|_{\rhote_X} + \| f_\mathcal{H} \|_{\rhote_X}  \right)    \left( 2^{q-1}  D^{- (m-1)   } m! W^{m - 2} \sigma^2 \right)^{1/2q},
\end{align*}
where $(A)$ follows from  $Lf_\rho = T f_\mathcal{H}$ by \citet[Proposition 1 (ii)]{caponnetto2007optimal}
and $(B)$ from  Eqs.~\eqref{eq:bound-2810} and \eqref{eq:bound-2815}. 
\vspace{-5mm}

\end{proof}

\begin{lemma} \label{lemma:trunc-approx-err}
%Let $w = d\rhote_X / d\rhotr_X$ be the importance weighting function, and  $w_D $  be the clipped importance weighting function at threshold $D > 0$: $w_D(x) := \min (w(x), D)$ for all $x \in X$. 
Suppose that Assumptions \ref{ass:target-exist} and \ref{ass:1}  are satisfied for the projection $f_\mathcal{H} \in \mathcal{H}$ in \eqref{eq:target-func}  with constants $1/2 \leq r \leq 1$ and $R > 0$. 
 %Define $T_D : \mathcal{H} \mapsto \mathcal{H}$ by $T_D f := \int K_x f(x)w_D(x) d\rhotr(x)$ for $f \in \mathcal{H}$. %, and let $f_\lambda^D := (T_D + \lambda)^{-1} T_D f_\mathcal{H} \in \mathcal{H}$ for any $\lambda > 0$.  
Then for all $\lambda > 0$, we have 
$$
\left\| (T_D + \lambda)^{-1} T_D f_\mathcal{H} - f_\mathcal{H }\right\|_{\rhote_X} \leq \lambda^r \left\| T (T_D + \lambda)^{-1}\right\|^r R.
$$
\end{lemma}

\begin{proof}
By Assumption \ref{ass:1}, there exists $ g   \in L_2(\rho_X)$ such that $f_\mathcal{H} = L^r g = T^{r -1/2} L^{1/2} g$ and $\left\| g \right\|_{\rho_X} \leq R$. 
Let   $I_k: \mathcal{H} \mapsto L_2(\rho_X)$ be the embedding operator.  We then have
%Using the identity $(A+\lambda I)^{-1}A=I-\lambda(A+\lambda)^{-1}$, which holds for any bounded self-adjoint positive operator $A$, we then have
\begin{align*}
& \left\| (T_D + \lambda)^{-1} T_D f_\mathcal{H} - f_\mathcal{H }\right\|_{\rhote_X}  \\
& = \left\| \left( (T_D + \lambda)^{-1} T_D - (T_D + \lambda)^{-1} (T_D + \lambda) \right) f_\mathcal{H} \right\|_{\rhote_X} = \left\| \lambda (T_D+\lambda)^{-1} f_{\mathcal{H}} \right\|_{\rhote_X} \\ 
& \stackrel{(A)}{=} \left\| L^{1/2} I_k \lambda (T_D+\lambda)^{-1} f_{\mathcal{H}} \right\|_{\mathcal{H}} \stackrel{(B)}{=}
   \left\| T^{1/2}\lambda (T_D+\lambda)^{-1} f_{\mathcal{H}} \right\|_{\mathcal{H}} \\
& = \left\| T^{1/2}\lambda (T_D+\lambda)^{-1} T^{r-1/2} L^{1/2} g \right\|_{\mathcal{H}} \\
& = \lambda^{r} \left\| T^{1/2} (T_D+\lambda)^{-1/2} \lambda^{1-r}  (T_D+\lambda)^{r - 1} (T_D+\lambda)^{-r + 1/2} T^{r-1/2} L^{1/2} g \right\|_{\mathcal{H}} \\
& \leq \lambda^{r} \left\| T^{1/2} (T_D+\lambda)^{-1/2} \right\| \left\| \lambda^{1-r}  (T_D+\lambda)^{r - 1} \right\| \left\| (T_D+\lambda)^{-r + 1/2} T^{r-1/2} \right\| \left\| L^{1/2} g \right\|_{\mathcal{H}} \\
& \stackrel{(C)}{\leq} \lambda^{r} \left\| T (T_D+\lambda)^{-1} \right\|^{1/2}  \left\| (T_D+\lambda)^{-1}  T \right\|^{r-1/2}   \left\| L^{1/2} g \right\|_{\mathcal{H}}  \stackrel{(D)}{=} \lambda^{r} \left\| T (T_D+\lambda)^{-1} \right\|^{r}   \left\|   g \right\|_{\rho_X} \\
& \leq \lambda^{r} \left\| T (T_D+\lambda)^{-1} \right\|^{r}   R
\end{align*} 
where $(A)$ and $(D)$ follow from $L^{1/2}: L_2(\rho_X) \mapsto \mathcal{H}$ being an isometry, $(B)$ $T^{1/2} = L^{1/2} I_k$, and $(C)$ from  Proposition \ref{Cordes_Inequality}.

\vspace{-5mm}
\end{proof}

\begin{lemma} \label{lemma:effective-dims}
%Define $T: \mathcal{H} \mapsto \mathcal{H}$ and $T_D : \mathcal{H} \mapsto \mathcal{H}$ by $T f := \int K_x f(x) w(x) d\rhotr(x)$ and  $T_D f := \int K_x f(x)w_D(x) d\rhotr(x)$ for $f \in \mathcal{H}$, where $w_D(x) = \min( w(x), D )$. 
For all $\lambda > 0$ and $D > 0$, we have 
$$
{\rm Tr} \left( T (T+ \lambda)^{-1}  \right) \geq {\rm Tr} \left( T_D (T_D + \lambda)^{-1} \right).
$$
\end{lemma}

\begin{proof}
For any operators $A$ and $B$, we write $A \geq B$ to mean that $A-B$ is a non-negative operator.
 First we show that $T \geq T_D$.  For all $f \in \mathcal{H}$, 
\begin{align*}
(T-T_D) f  = \int K_x f(x) ( w(x) - w_D(x) ) d\rhotr(x) = \int_{w \geq D} K_x f(x) ( w(x) - D ) d\rhotr(x). 
\end{align*} 
Thus 
\begin{align*}
& \left< f, (T-T_D) f \right>_{\mathcal{H}} = \left< f,  \int_{w \geq D} K_x f(x) ( w(x) - D ) d\rhotr(x)\right>_{\mathcal{H}} \\
& = \int_{w \geq D} \left<f, K_x\right>_{\mathcal{H}} f(x) ( w(x) - D ) d\rhotr(x) = \int_{w \geq D}  f^2(x) ( w(x) - D ) d\rhotr(x) \geq 0,
\end{align*} 
which implies $T \geq T_D$.
 
Now we show $T (T+ \lambda)^{-1} \geq T_D (T_D + \lambda)^{-1}$, which implies the assertion. 
As we have $A(A + \lambda)^{-1} = I - \lambda (A+\lambda)^{-1}$ for any operator $A$, we have 
\begin{align*}
T(T + \lambda)^{-1} = I - \lambda (T+\lambda)^{-1}, \quad T_D(T_D + \lambda)^{-1} = I - \lambda (T_D + \lambda)^{-1}.
\end{align*}
 Thus we have 
 $$
 T(T + \lambda)^{-1}  - T_D (T_D + \lambda)^{-1 } = \lambda \left( - (T + \lambda)^{-1} + (T_D + \lambda)^{-1} \right) \geq 0,
 $$
 where the last inequality follows from $ (T + \lambda)^{-1} \leq (T_D + \lambda)^{-1} $, which follows from $ T \geq T_D $. 
\end{proof}

\vspace{-5mm}

\subsection{Proof of Theorem \ref{theo:clipped-KRR-re}}
\label{sec:proof-clipped-KRR-main}

%We are ready to prove Theorem \ref{theo:clipped-IW}.

\begin{proof}
We first present preliminaries. 
We will use Lemmas \ref{lemma:bound-TD-2518}  and
\ref{lemma:TDLDfrhofH} in the proof. To this end, we show that $\lambda$ and $D$ in \eqref{eq:lambda-D-2915} satisfy condition \eqref{eq:cond-D-trun-2688} for Lemmas \ref{lemma:bound-TD-2518}  and
\ref{lemma:TDLDfrhofH}:
\begin{align*}
 D   \geq    (2^{2q-1} E_s^{2q} m! W^{m - 2} \sigma^2 )^{1/(m-1)}   \lambda^{-(1+s)q/(m-1)} = B \lambda^{-(1+s)q/(m-1)},     
\end{align*}
where $B :=  (2^{2q-1} E_s^{2q} m! W^{m - 2} \sigma^2 )^{1/(m-1)}$. 
This condition is equivalent to 
\begin{align*}
& c_2 n^{\frac{4qr}{(s+2r)(m-1) + 4qr + \epsilon}} \geq  B c_1^{-(1+s)q/(m-1)} n^{  \frac{ (1+s)q   }{(s+2r)(m-1) + 4qr + \epsilon}  } \\
\Longleftrightarrow \quad &  c_2 n^{\frac{ q [ 4r - (1+s) ]}{(s+2r)(m-1) + 4qr + \epsilon}} \geq  B c_1^{-(1+s)q/(m-1)}   
\stackrel{(*)}{\Longleftarrow}    c_2   \geq  B c_1^{-(1+s)q/(m-1)},
\end{align*}
where $(*)$ follows from $4r - (1 + s) \geq 0$ and the last inequality is the same as condition \eqref{eq:c2-c1-cond-trun}. 

Define $\mathcal{N}^D(\lambda) := {\rm Tr}( (T_D+\lambda)^{-1} T_D)$. 
Lemma \ref{lemma:effective-dims} and Assumption \ref{ass:effective-dim}  imply that
\begin{equation} \label{eq:dof-3054}
\mathcal{N}^{D}(\lambda) \leq \mathcal{N}(\lambda) \leq E_s^2 \lambda^{-s}.
\end{equation}

\paragraph{\underline{Decomposing the Error.}}
Define  
$f_{\lambda}^D = (T_D+\lambda I)^{-1}T_D f_{\mathcal{H}} \in \mathcal{H}$.
Then, by the triangle inequality,
\begin{equation}\label{bias_variance_decomposition_D}
   \left\| f_{\mathbf{z},\lambda}^D-f_{\mathcal{H}} \right\|_{\rhote_X} \leq \left\|f_{\mathbf{z},\lambda}^D-f_{\lambda}^D \right\|_{\rhote_X} +\left\|f_{\lambda}^D -f_{\mathcal{H}} \right\|_{\rhote_X}.
\end{equation}
For the second term on the right-hand side of \eqref{bias_variance_decomposition_D},
 we have, by Lemmas \ref{lemma:bound-TD-2518} and \ref{lemma:trunc-approx-err},
\begin{equation} \label{eq:approx-error-2927}
\left\| f_\lambda^D - f_\mathcal{H }\right\|_{\rhote_X} \leq \lambda^r \left\| T (T_D + \lambda)^{-1}\right\|^r R \leq 2^r \lambda^r R.    
\end{equation}
Therefore, we will focus on bounding the first term $ \left\|f_{\mathbf{z},\lambda}^D-f_{\lambda}^D \right\|_{\rhote_X}$ on the right-hand side of \eqref{bias_variance_decomposition_D}.

Define $g_{{\bf z}, D} \in \mathcal{H}$ and  $T_{ {\bf x}, D}: \mathcal{H} \mapsto \mathcal{H}$  by  
\begin{align*}
& g_{ {\bf z}, D } :=  \frac{1}{n} \sum_{i=1}^n w_D(x_i) y_i k(\cdot, x_i), \quad  T_{\mathbf{x}, D } f := \sum_{i=1}^n w_D(x_i) f(x_i) K_{x_i} \quad  (\text{for}\ f \in \mathcal{H}).
\end{align*} 
Then  $
f_{\mathbf{z},\lambda}^D =   (T_{{\bf x},D} + \lambda)^{-1}  g_{ {\bf z}, D } $.  %and  $ f^D_\lambda  = (T_D + \lambda)^{-1} T_D f_\mathcal{H}$.
Now we have
\begin{align}
  &  \left\|f_{\mathbf{z},\lambda}^D-f_{\lambda}^D \right\|_{\rhote_X}  = \left\|T^{1/2} ( f_{\mathbf{z},\lambda}^D-f_{\lambda}^D ) \right\|_{\mathcal{H}} \nonumber \\
  & =  \left\|T^{1/2} \left(  (T_{{\bf x},D} + \lambda)^{-1}  g_{ {\bf z}, D }   -   (T_D + \lambda)^{-1} T_D f_\mathcal{H} \right) \right\|_{\mathcal{H}} \nonumber \\
&  \stackrel{(A)}{\leq} \left\|\  T  (T_D+\lambda)^{-1} \right\|^{1/2}  \left\|\left(I-(T_D+\lambda)^{-1/2}\left(T_D-T_{ {\bf x}, D }\right)(T_D+\lambda)^{-1/2}\right)^{-1} \right\| \nonumber  \\
&  \quad \times  \left( \left\|  (T_D+\lambda)^{-1/2}\left(g_{ {\bf z}, D } - T_D f_\mathcal{H} \right) \right\|_{\mathcal{H}} + \left\| (T_D+\lambda)^{-1/2}\left(T_D-T_{{\bf x}, D}\right) (T_D+\lambda)^{-1} T_D f_\mathcal{H}   \right\|_{ \mathcal{H} }  \right)  \nonumber \\
&  \leq \left\|\  T  (T_D+\lambda)^{-1} \right\|^{1/2}  \left\|\left(I-(T_D+\lambda)^{-1/2}\left(T_D-T_{ {\bf x}, D }\right)(T_D+\lambda)^{-1/2}\right)^{-1} \right\| \nonumber  \\
&  \quad \times  \left( \left\|  (T_D+\lambda)^{-1/2}\left(g_{ {\bf z}, D } - L_D f_\rho \right)  \right\|_{\mathcal{H}} + \left\|  (T_D+\lambda)^{-1/2}\left( L_D f_\rho - T_D f_\mathcal{H} \right)  \right\|_{\mathcal{H}}  \right.  \nonumber \\ 
& \quad \quad \quad \left.+  \left\| (T_D+\lambda)^{-1/2}\left(T_D-T_{{\bf x}, D}\right) (T_D+\lambda)^{-1} T_D f_\mathcal{H}   \right\|_{ \mathcal{H} }  \right)    \nonumber  
\\ 
& \stackrel{(B)}{\leq} 2^{1/2} S_* \left( S_2 + S_3 + S_4 \right), \label{eq:bound-2832}
\end{align}
where $(A)$ follows from Lemma \ref{lemma:ABC} with $A := T_{ {\bf x}, D }$, $B := T_D$, $C := T$, $g := g_{ {\bf z}, D }$ and $h := T_D f_\mathcal{H}$,  $(B)$ follows  from  Lemma \ref{lemma:bound-TD-2518}, and we defined  
\begin{align*}
& S_* :=  \left\|\left(I-(T_D+\lambda)^{-1/2}\left(T_D-T_{ {\bf x}, D }\right)(T_D+\lambda)^{-1/2}\right)^{-1} \right\|,  \\
& S_2 :=  \left\|  (T_D+\lambda)^{-1/2}\left(g_{ {\bf z}, D } - L_D f_\rho \right)  \right\|_{\mathcal{H}}, \\
& S_3 :=   \left\|  (T_D+\lambda)^{-1/2}\left( L_D f_\rho - T_D f_\mathcal{H} \right)  \right\|_{\mathcal{H}}, \\
& S_4 := \left\| (T_D+\lambda)^{-1/2}\left(T_D-T_{{\bf x}, D}\right) (T_D+\lambda)^{-1} g_D   \right\|_{ \mathcal{H} } = \left\| (T_D+\lambda)^{-1/2}\left(T_D-T_{{\bf x}, D}\right) f_\lambda^D   \right\|_{ \mathcal{H} } .
\end{align*}
Below we will bound these four quantities individually.   
 
%\left\|  (T_D+\lambda)^{-1/2}\left(g_{ {\bf z}, D } - L_D f_\rho \right)  \right\|_{\mathcal{H}} + \left\|  (T_D+\lambda)^{-1/2}\left( L_D f_\rho - T_D f_\mathcal{H} \right)  \right\|_{\mathcal{H}} 

\paragraph{\underline{Bounding $S_*$.}}
We use Lemma \ref{lemma:S-1-bound} with $d\rho'_X = w_D d\rhotr_X$ (and thus $\nu = d\rho'_X / d\rhotr_X = w_D$), $T_{\bf x, \nu} := T_{{\bf x}, D}$ and  $T' := T_D$. 
To this end, we first check the conditions required for Lemma \ref{lemma:S-1-bound}.
First, Assumption \ref{ass:4} is satisfied for  $v = w_D$ with $q' = 0$, $V = D$, and  $\gamma = D^{1/2}$. 
Moreover, we have $\left\| T_D \right\| \leq 1$, which can be shown as follows.
As  $\sup_{x \in X} K(x,x) \leq 1$ (by assumption), we have for all $f \in \mathcal{H}$
\begin{align}
&  \left\| T_D f \right\|_{\mathcal{H}} = \left\| \int K_x \left< K_x, f \right>_{\mathcal{H}}  w_D(x) d\rhotr(x)\right\|_{\mathcal{H}} \leq \int \left\| K_x \right\|^2 \| f \|_{\mathcal{H}}  w_D(x) d\rhotr(x) \nonumber  \\
& = \| f \|_{ \mathcal{H} } \int K(x,x)   w_D(x) d\rhotr(x)  \leq \| f \|_{ \mathcal{H} } \int  w(x) d\rhotr(x)   =  \| f \|_{ \mathcal{H} } \int   d\rhote(x)   \leq \| f \|_{ \mathcal{H} }, \label{eq:fDfH-2992}
\end{align}
and thus $\| T_D \| = \sup_{ \| f \|_{ \mathcal{H} } \leq 1} \| T_D f \| \leq 1$.
Hence, Lemma \ref{lemma:S-1-bound} is applicable. 
Therefore,  we have with probability greater than $1 - \delta/3$ 
\begin{align*}  
& S_1 := \left\|(T_D+\lambda)^{-\frac{1}{2}}\left(T_D-T_{\mathbf{x},D}\right)(T_D+\lambda)^{-\frac{1}{2}}\right\|_{\mathrm{HS}} \\ & \leq 4 \left( D\lambda^{-1} n^{-1}     +   (D \lambda^{-1} n^{-1})^{1/2} (\mathcal{N}^D(\lambda))^{1/2} \right) \log\left( 6/\delta\right), \\
%& \leq 4 \left( D\lambda^{-1} n^{-1}     +   (D \lambda^{-1} n^{-1})^{1/2} (\mathcal{N}(\lambda))^{1/2} \right) \log\left( 6/\delta\right) \\
& \leq 4 \left( D\lambda^{-1} n^{-1}     +   (D \lambda^{-1} n^{-1})^{1/2} E_s \lambda^{-s/2} \right) \log\left( 6/\delta\right),
\end{align*}
 where the last inequality follows from \eqref{eq:dof-3054}. 
Note that  
\begin{align*}
&  (D \lambda^{-1} n^{-1})^{1/2} E_s \lambda^{-s/2}  = D^{1/2} \lambda^{-(1+s)/2 } n^{-1/2} E_s   = E_s c_1^{ -(1+s)/2 } c_2^{1/2}  n^{ [\tau + \beta (1+s) - 1]/2} \\
& = E_s c_1^{ -(1+s)/2 } c_2^{1/2} n^{ - \frac{  (m-1) (2r-1) + \epsilon  }{2\left[ (s+2r)(m-1) + 4qr + \epsilon \right] } } \stackrel{(*)}{\leq} 3/ ( 32 \log(6/\delta) < 1,
\end{align*}
where  $(*)$ follows from \eqref{eq:condition-trunc-2921}. 
Note also that $E_s \lambda^{-s/2} \geq 1$, since $E_s \geq 1$ and $\lambda \leq 1$.
Therefore,
 \begin{align*}
1 >  (D \lambda^{-1} n^{-1})^{1/2} E_s \lambda^{-s/2} \geq (D \lambda^{-1} n^{-1})^{1/2} \geq  D \lambda^{-1} n^{-1}.
 \end{align*}
 Hence,  
 \begin{align}   \label{eq:bound-2859}
& S_1   \leq 8   (D \lambda^{-1} n^{-1})^{1/2} E_s \lambda^{-s/2}  \log\left( 6/\delta\right) 
\leq 3/4. 
\end{align}
 Thus,  the term $S_*$ is bounded as, with probability greater than $1- \delta/3$, 
 \begin{align}
S_* & = \left\|\left\{I-(T_D+\lambda)^{-\frac{1}{2}}\left(T_D-T_{\mathbf{x},D}\right)(T_D+\lambda)^{-\frac{1}{2}}\right\}^{-1}\right\| \nonumber \\ 
& \stackrel{(A)}{=}\left\|\sum_{j=0}^{\infty}\left[(T_D+\lambda)^{-\frac{1}{2}}\left(T-T_{\mathbf{x},D}\right)(T_D+\lambda)^{-\frac{1}{2}}\right]^j\right\| \nonumber \\
& \leq \sum_{j=0}^{\infty}\left\|(T_D+\lambda)^{-\frac{1}{2}}\left(T_D-T_{\mathbf{x},D}\right)(T_D+\lambda)^{-\frac{1}{2}}\right\|^j \nonumber \\
& \leq \sum_{j=0}^{\infty}\left\|(T+\lambda)^{-\frac{1}{2}}\left(T_D-T_{\mathbf{x},D}\right)(T_D+\lambda)^{-\frac{1}{2}}\right\|_{\mathrm{HS}}^j  = \sum_{j=0}^\infty S_1^j  \stackrel{(B)}{=}  (1-S_1)^{-1} \stackrel{(C)}{\leq} 4, \label{eq:Sstar-bound-3112}
\end{align}
where each of $(A)$ and $(B)$ follows from the Neumann series expansion and \eqref{eq:bound-2859}, and $(C)$ from \eqref{eq:bound-2859}.

%Therefore, by \eqref{eq:bound-2832}, we have, with probability greater than $1 - \delta/3$,
%\begin{align} \label{eq:bound-3024}
%  &  \left\|f_{\mathbf{z},\lambda}^D-f_{\lambda}^D \right\|_{\rhote_X} \leq  4 \left\|\  T  (T_D+\lambda)^{-1} \right\|^{1/2}   \left( S_2 +  S_3 + S_4  \right) \leq 2^{1/2}4 \left( S_2 + S_3 + S_4 \right)
%\end{align}
%where the second inequality follows from Lemma \ref{lemma:bound-TD-2518} and 

\paragraph{\underline{Bounding $S_2$.}}
We use Lemma \ref{lemma:S-bound-gen} with $v = w_D$,  $q' = 0$, $V = D$, $\gamma = D^{1/2}$, $T' := T_D$, $u_i := y_i$ (for $i=1,\dots,n$), $u:= y$ with $(x,y) \sim \rhotr$, and  
\begin{align*}
&  g_{{\bf z}, {\bf v}} := g_{{\bf z}, D}  = \frac{1}{n} \sum_{i=1}^n w_D(x_i) y_i k(\cdot, x_i),   \\
&  g' := L_D f_\rho = \mathbb{E}_{x \sim \rhotr_X}[ w_D(x) f_\rho(x) K_x ] = \mathbb{E}_{(x,y) \sim \rhotr}[ w_D(x) y K_x ]. 
\end{align*}
Thus the bound \eqref{eq:bound-S-gen} holds with $\mathcal{N}'(\lambda) = \mathcal{N}^D(\lambda) $ and  $U = M$. That is, we have, with probability greater than $1-\delta/3$,
\begin{align}
S_2  &\leq 4 M \left(  D n^{-1} \lambda^{-1/2}+   D^{1/2} n^{-1/2} (\mathcal{N}^{D}(\lambda))^{1/2} \right)  \log \left( 6/\delta\right)  \nonumber \\
%&  \leq 4 M \left(  D n^{-1} \lambda^{-1/2}+   D^{1/2} n^{-1/2} (\mathcal{N} (\lambda))^{1/2} \right)  \log \left( 6/\delta\right) \\
&  \leq 4 M \left(  D n^{-1} \lambda^{-1/2}+   D^{1/2} n^{-1/2}  E_s \lambda^{-s/2}  \right)  \log \left( 6/\delta\right), \label{eq:bound-S2-3052}
\end{align}
where the second inequality follows from \eqref{eq:dof-3054}.

 \paragraph{\underline{Bounding $S_3$.}}
By Lemma \ref{lemma:TDLDfrhofH},  the term $S_3$ can be bounded as
\begin{align} \label{eq:bound-3035}
S_3  & \leq  \left(  \| f_\rho \|_{\rhote_X} + \| f_\mathcal{H} \|_{\rhote_X}  \right)    \left( 2^{q-1}  D^{- (m-1)   } m! W^{m - 2} \sigma^2 \right)^{1/2q}.
\end{align}

\paragraph{\underline{Bounding $S_4$.}}
We use Lemma \ref{lemma:S-bound-gen} with $v = w_D$,  $q' = 0$, $V = D$, $\gamma = D^{1/2}$, $T' := T_D$,  $u_i := f_\lambda^D(x_i)$ (for $i=1,\dots,n$),  $u:= f_\lambda^D(x)$ with $x \sim \rhotr_X$, and
\begin{align*}
& g_{\bf z, v} := T_{{\bf x}, D} f_\lambda^D = \frac{1}{n} \sum_{i=1}^n w_D(x_i) K_{x_i } f_\lambda^D (x_i),     \\
& g' := T_D f_\lambda^D = \mathbb{E}_{x \sim \rhotr_X} \left[ w_D(x) K_x  f_\lambda^D(x) \right].  
\end{align*}
Thus the bound \eqref{eq:bound-S-gen} holds with $\mathcal{N}'(\lambda) = \mathcal{N}^D(\lambda) $ and  $U = \left\| f_\lambda^D \right\|_\infty$. That is, we have, with probability greater than $1-\delta/3$,
\begin{align}
S_4 & \leq    4 \left\| f_\lambda^D \right\|_\infty \left(  D n^{-1} \lambda^{-1/2} +   D^{1/2} ( \mathcal{N}^D(\lambda) )^{1/2}  n^{-1/2}  \right)  \log \left( 6/\delta  \right) \nonumber \\
  & \leq    4 \left\| f_\mathcal{H} \right\|_\mathcal{H} \left(  D n^{-1} \lambda^{-1/2} +   D^{1/2}  n^{-1/2} E_s \lambda^{-s/2}     \right)  \log \left( 6/\delta  \right), \label{eq:S4-3071}
\end{align}
where the second inequality follows from \eqref{eq:dof-3054} and 
$$
\left\| f_\lambda^D \right\|_\infty \leq \left\| f_\lambda^D \right\|_\mathcal{H} = \left\| (T_D + \lambda I)^{-1} T_D f_\mathcal{H} \right\|_\mathcal{H} \leq \left\| (T_D + \lambda I)^{-1} T_D \right\| \left\| f_\mathcal{H} \right\|_\mathcal{H} \leq \left\| f_\mathcal{H} \right\|_\mathcal{H}.
$$

\paragraph{\underline{Combining the bounds.}}
By \eqref{eq:bound-2832},  \eqref{eq:Sstar-bound-3112},  \eqref{eq:bound-S2-3052}, \eqref{eq:bound-3035}  and  \eqref{eq:S4-3071}, we have, with probability greater than $1-\delta$,
\begin{align*}
  &  \left\|f_{\mathbf{z},\lambda}^D-f_{\lambda}^D \right\|_{\rhote_X}   \leq 2^{1/2}4 \left( S_2 + S_3 + S_4 \right)   \\
& \leq \left(  \| f_\rho \|_{\rhote_X} + \| f_\mathcal{H} \|_{\rhote_X}  \right)    \left( 2^{6q-1}  D^{- (m-1)   } m! W^{m - 2} \sigma^2 \right)^{1/2q} \\  
&\quad +  2^{1/2} 16 (M+ \left\| f_\mathcal{H} \right\|_{\mathcal{H}}) \left(  D n^{-1} \lambda^{-1/2}+   D^{1/2} n^{-1/2}  E_s \lambda^{-s/2}  \right)  \log \left( 6/\delta\right)   \\
& = A_1 D^{- (m-1) / 2q   } + A_2  D n^{-1} \lambda^{-1/2} +  A_3 D^{1/2} n^{-1/2}   \lambda^{-s/2}, 
\end{align*}
where $A_1$, $A_2$ and $A_3$ are defined as \eqref{eq:A-constants-def-2949} in the assertion.

Therefore, using  \eqref{bias_variance_decomposition_D}, \eqref{eq:approx-error-2927}, and  the expressions of $D$ and $\lambda$ in \eqref{eq:lambda-D-2915}, we have
\begin{align*}
&   \left\| f_{\mathbf{z},\lambda}^D-f_{\mathcal{H}} \right\|_{\rhote_X}  \leq  \left\|f_{\mathbf{z},\lambda}^D-f_{\lambda}^D \right\|_{\rhote_X} + \left\| f_\lambda^D - f_\mathcal{H }\right\|_{\rhote_X} \\
& = A_1 D^{- (m-1) / 2q   } + A_2  D n^{-1} \lambda^{-1/2} +  A_3 D^{1/2} n^{-1/2}   \lambda^{-s/2}   + 2^r R \lambda^r \\
 & \leq A_1 ( c_2 n^\tau )^{- (m-1) / 2q   }  + A_2 c_2 n^\tau n^{-1} (c_1 n^{- \beta})^{-1/2} \\
 & \quad + A_3 (c_2 n^{\tau})^{1/2} n^{-1/2} ( c_1 n^{- \beta})^{-s/2} + 2^r R ( c_1 n^{- \beta} )^r \\
 & =  A_1 c_2^{- (m-1) / 2q   } n^{- \tau (m-1) / 2q   }  +  A_2 c_1^{-1/2} c_2 n^{\tau - 1 + \beta/2}  \\
 & \quad +  A_3 c_1^{-s/2} c_2^{1/2} n^{\tau/2 - 1/2 + \beta s /2}   + 2^r R c_1^r  n^{- \beta r} \\
 & = A_1 c_2^{- (m-1) / 2q   } n^{- \frac{2r (m-1)}{(s+2r)(m-1) + 4qr + \epsilon}  }  + A_2 c_1^{-1/2} c_2 n^{ - \frac{ (m-1)(s+2r-1/2) + \epsilon }{(s+2r)(m-1) + 4qr + \epsilon}  } \\
 & \quad +  A_3 c_1^{-s/2} c_2^{1/2} n^{ - \frac{r(m-1) + \epsilon/2 }{(s+2r)(m-1) + 4qr + \epsilon}  }  + 2^r R c_1^r  n^{-  \frac{r(m-1)}{(s+2r)(m-1) + 4qr + \epsilon}   } \\
 & \leq \left( A_1 c_2^{- (m-1) / 2q } + A_2 c_1^{-1/2} c_2 + A_3 c_1^{-s/2} c_2^{1/2} + 2^r R c_1^r \right) n^{-  \frac{r(m-1)}{(s+2r)(m-1) + 4qr + \epsilon}   } .
\end{align*}
  
\end{proof}

\vskip 0.2in
\bibliography{sample}

\begin{thebibliography}{61}
\providecommand{\natexlab}[1]{#1}
\providecommand{\url}[1]{\texttt{#1}}
\expandafter\ifx\csname urlstyle\endcsname\relax
  \providecommand{\doi}[1]{doi: #1}\else
  \providecommand{\doi}{doi: \begingroup \urlstyle{rm}\Url}\fi

\bibitem[Arora et~al.(2019)Arora, Du, Hu, Li, Salakhutdinov, and
  Wang]{arora2019exact}
Sanjeev Arora, Simon~S Du, Wei Hu, Zhiyuan Li, Russ~R Salakhutdinov, and
  Ruosong Wang.
\newblock On exact computation with an infinitely wide neural net.
\newblock \emph{Advances in Neural Information Processing Systems}, 32, 2019.

\bibitem[Bach and Jordan(2002)]{bach2002kernel}
Francis~R Bach and Michael~I Jordan.
\newblock Kernel independent component analysis.
\newblock \emph{Journal of Machine Learning Research}, 3:\penalty0 1--48, 2002.

\bibitem[Bartlett et~al.(2006)Bartlett, Jordan, and
  McAuliffe]{bartlett2006convexity}
Peter~L Bartlett, Michael~I Jordan, and Jon~D McAuliffe.
\newblock Convexity, classification, and risk bounds.
\newblock \emph{Journal of the American Statistical Association}, 101\penalty0
  (473):\penalty0 138--156, 2006.

\bibitem[Bauer et~al.(2007)Bauer, Pereverzev, and
  Rosasco]{bauer2007regularization}
Frank Bauer, Sergei Pereverzev, and Lorenzo Rosasco.
\newblock On regularization algorithms in learning theory.
\newblock \emph{Journal of Complexity}, 23\penalty0 (1):\penalty0 52--72, 2007.

\bibitem[Ben-David et~al.(2007)Ben-David, Blitzer, Crammer, and
  Pereira]{ben2007analysis}
Shai Ben-David, John Blitzer, Koby Crammer, and Fernando Pereira.
\newblock Analysis of representations for domain adaptation.
\newblock \emph{Advances in Neural Information Processing Systems},
  19:\penalty0 137, 2007.

\bibitem[Birman and Solomjak(1967)]{birman1967piecewise}
M~{\v{S}} Birman and MZ~Solomjak.
\newblock Piecewise-polynomial approximations of functions of the classes.
\newblock \emph{Mathematics of the USSR-Sbornik}, 2\penalty0 (3):\penalty0 295,
  1967.

\bibitem[Byrd and Lipton(2019)]{byrd2019effect}
Jonathon Byrd and Zachary Lipton.
\newblock What is the effect of importance weighting in deep learning?
\newblock In \emph{International Conference on Machine Learning}, pages
  872--881. PMLR, 2019.

\bibitem[Caponnetto and De~Vito(2007)]{caponnetto2007optimal}
Andrea Caponnetto and Ernesto De~Vito.
\newblock Optimal rates for the regularized least-squares algorithm.
\newblock \emph{Foundations of Computational Mathematics}, 7\penalty0
  (3):\penalty0 331--368, 2007.

\bibitem[Cortes and Mohri(2014)]{cortes2014domain}
Corinna Cortes and Mehryar Mohri.
\newblock Domain adaptation and sample bias correction theory and algorithm for
  regression.
\newblock \emph{Theoretical Computer Science}, 519:\penalty0 103--126, 2014.

\bibitem[Cortes et~al.(2008)Cortes, Mohri, Riley, and
  Rostamizadeh]{cortes2008sample}
Corinna Cortes, Mehryar Mohri, Michael Riley, and Afshin Rostamizadeh.
\newblock Sample selection bias correction theory.
\newblock In \emph{International Conference on Algorithmic Learning Theory},
  pages 38--53. Springer, 2008.

\bibitem[Cortes et~al.(2010)Cortes, Mansour, and Mohri]{cortes2010}
Corinna Cortes, Yishay Mansour, and Mehryar Mohri.
\newblock Learning bounds for importance weighting.
\newblock In \emph{Advances in Neural Information Processing Systems}, pages
  442--450, 2010.

\bibitem[Cucker and Smale(2002)]{cucker2002mathematical}
Felipe Cucker and Steve Smale.
\newblock On the mathematical foundations of learning.
\newblock \emph{Bulletin of the American Mathematical Society}, 39\penalty0
  (1):\penalty0 1--49, 2002.

\bibitem[De~Vito et~al.(2005)De~Vito, Caponnetto, and Rosasco]{de2005model}
Ernesto De~Vito, Andrea Caponnetto, and Lorenzo Rosasco.
\newblock Model selection for regularized least-squares algorithm in learning
  theory.
\newblock \emph{Foundations of Computational Mathematics}, 5\penalty0
  (1):\penalty0 59--85, 2005.

\bibitem[Fang et~al.(2020)Fang, Lu, Niu, and Sugiyama]{fang2020rethinking}
Tongtong Fang, Nan Lu, Gang Niu, and Masashi Sugiyama.
\newblock Rethinking importance weighting for deep learning under distribution
  shift.
\newblock In \emph{Proceedings of the 34th International Conference on Neural
  Information Processing Systems}, pages 11996--12007, 2020.

\bibitem[Furuta(2001)]{furuta2001invitation}
Takayuki Furuta.
\newblock \emph{Invitation to Linear Operators: From Matrices to Bounded Linear
  Operators on a Hilbert Space}.
\newblock Taylor \& Francis, 2001.

\bibitem[Heckman(1979)]{heckman1979sample}
James~J Heckman.
\newblock Sample selection bias as a specification error.
\newblock \emph{Econometrica: Journal of the Econometric Society}, pages
  153--161, 1979.

\bibitem[Hendrycks and Dietterich(2019)]{hendrycks2018benchmarking}
Dan Hendrycks and Thomas Dietterich.
\newblock Benchmarking neural network robustness to common corruptions and
  perturbations.
\newblock In \emph{International Conference on Learning Representations}, 2019.

\bibitem[Hendrycks et~al.(2021)Hendrycks, Basart, Mu, Kadavath, Wang, Dorundo,
  Desai, Zhu, Parajuli, Guo, et~al.]{hendrycks2021many}
Dan Hendrycks, Steven Basart, Norman Mu, Saurav Kadavath, Frank Wang, Evan
  Dorundo, Rahul Desai, Tyler Zhu, Samyak Parajuli, Mike Guo, et~al.
\newblock The many faces of robustness: A critical analysis of
  out-of-distribution generalization.
\newblock In \emph{Proceedings of the IEEE/CVF International Conference on
  Computer Vision}, pages 8340--8349, 2021.

\bibitem[Huang et~al.(2006)Huang, Smola, Gretton, Borgwardt, and
  Scholkopf]{huang2006correcting}
Jiayuan Huang, Alexander~J Smola, Arthur Gretton, Karsten~M Borgwardt, and
  Bernhard Scholkopf.
\newblock Correcting sample selection bias by unlabeled data.
\newblock In \emph{Proceedings of the 19th International Conference on Neural
  Information Processing Systems}, pages 601--608, 2006.

\bibitem[Jacot et~al.(2018)Jacot, Gabriel, and Hongler]{jacot2018neural}
Arthur Jacot, Franck Gabriel, and Cl{\'e}ment Hongler.
\newblock Neural tangent kernel: Convergence and generalization in neural
  networks.
\newblock \emph{Advances in Neural Information Processing Systems}, 31, 2018.

\bibitem[Jiang and Zhai(2007)]{jiang2007instance}
Jing Jiang and ChengXiang Zhai.
\newblock Instance weighting for domain adaptation in {NLP}.
\newblock In \emph{Proceedings of the 45th Annual Meeting of the Association of
  Computational Linguistics}, pages 264--271, 2007.

\bibitem[Kpotufe and Martinet(2021)]{kpotufe2021marginal}
Samory Kpotufe and Guillaume Martinet.
\newblock Marginal singularity and the benefits of labels in covariate-shift.
\newblock \emph{The Annals of Statistics}, 49\penalty0 (6):\penalty0
  3299--3323, 2021.

\bibitem[Lei et~al.(2021)Lei, Hu, and Lee]{lei2021near}
Qi~Lei, Wei Hu, and Jason Lee.
\newblock Near-optimal linear regression under distribution shift.
\newblock In \emph{International Conference on Machine Learning}, pages
  6164--6174. PMLR, 2021.

\bibitem[Ma et~al.(2022)Ma, Pathak, and Wainwright]{ma2022optimally}
Cong Ma, Reese Pathak, and Martin~J Wainwright.
\newblock Optimally tackling covariate shift in rkhs-based nonparametric
  regression.
\newblock \emph{arXiv preprint arXiv:2205.02986}, 2022.

\bibitem[MacKay(1992)]{mackay1992information}
David~JC MacKay.
\newblock Information-based objective functions for active data selection.
\newblock \emph{Neural Computation}, 4\penalty0 (4):\penalty0 590--604, 1992.

\bibitem[Mammen and Tsybakov(1999)]{mammen1999smooth}
Enno Mammen and Alexandre~B Tsybakov.
\newblock Smooth discrimination analysis.
\newblock \emph{The Annals of Statistics}, 27\penalty0 (6):\penalty0
  1808--1829, 1999.

\bibitem[Mansour et~al.(2009{\natexlab{a}})Mansour, Mohri, and
  Rostamizadeh]{10.5555/1795114.1795157}
Yishay Mansour, Mehryar Mohri, and Afshin Rostamizadeh.
\newblock Multiple source adaptation and the {R\'{e}nyi} divergence.
\newblock In \emph{Proc. of the 25th Conference on Uncertainty in Artificial
  Intelligence}, UAI '09, page 367–374. AUAI Press, 2009{\natexlab{a}}.

\bibitem[Mansour et~al.(2009{\natexlab{b}})Mansour, Mohri, and
  Rostamizadeh]{MansourMR09}
Yishay Mansour, Mehryar Mohri, and Afshin Rostamizadeh.
\newblock Domain adaptation: Learning bounds and algorithms.
\newblock In \emph{{COLT} 2009 - The 22nd Conference on Learning Theory,
  Montreal, Quebec, Canada, June 18-21, 2009}, 2009{\natexlab{b}}.

\bibitem[Mendelson and Neeman(2010)]{mendelson2010regularization}
Shahar Mendelson and Joseph Neeman.
\newblock Regularization in kernel learning.
\newblock \emph{The Annals of Statistics}, 38\penalty0 (1):\penalty0 526--565,
  2010.

\bibitem[Pathak et~al.(2022)Pathak, Ma, and Wainwright]{pathak2022new}
Reese Pathak, Cong Ma, and Martin Wainwright.
\newblock A new similarity measure for covariate shift with applications to
  nonparametric regression.
\newblock In \emph{International Conference on Machine Learning}, pages
  17517--17530. PMLR, 2022.

\bibitem[Precup et~al.(2000)Precup, Sutton, and Singh]{precup2000eligibility}
Doina Precup, Richard~S Sutton, and Satinder~P Singh.
\newblock Eligibility traces for off-policy policy evaluation.
\newblock In \emph{Proceedings of the Seventeenth International Conference on
  Machine Learning}, pages 759--766, 2000.

\bibitem[Pukelsheim(2006)]{pukelsheim2006optimal}
Friedrich Pukelsheim.
\newblock \emph{Optimal Design of Experiments}.
\newblock SIAM, 2006.

\bibitem[Quinonero-Candela et~al.(2008)Quinonero-Candela, Sugiyama,
  Schwaighofer, and Lawrence]{quinonero2008dataset}
Joaquin Quinonero-Candela, Masashi Sugiyama, Anton Schwaighofer, and Neil~D
  Lawrence.
\newblock \emph{Dataset Shift in Machine Learning}.
\newblock MIT Press, 2008.

\bibitem[Rahimi and Recht(2007)]{rahimi2007random}
Ali Rahimi and Benjamin Recht.
\newblock Random features for large-scale kernel machines.
\newblock \emph{Advances in Neural Information Processing Systems}, 20, 2007.

\bibitem[Raskutti et~al.(2012)Raskutti, J~Wainwright, and
  Yu]{raskutti2012minimax}
Garvesh Raskutti, Martin J~Wainwright, and Bin Yu.
\newblock Minimax-optimal rates for sparse additive models over kernel classes
  via convex programming.
\newblock \emph{Journal of Machine Learning Research}, 13\penalty0 (2), 2012.

\bibitem[Rasmussen and Williams(2006)]{williams2006gaussian}
Carl~Edward Rasmussen and Christopher~KI Williams.
\newblock \emph{Gaussian Processes for Machine Learning}.
\newblock MIT Press, 2006.

\bibitem[Rudi and Rosasco(2017)]{rudi2017generalization}
Alessandro Rudi and Lorenzo Rosasco.
\newblock Generalization properties of learning with random features.
\newblock In \emph{NIPS}, pages 3215--3225, 2017.

\bibitem[Schmidt-Hieber and Zamolodtchikov(2022)]{schmidt2022local}
Johannes Schmidt-Hieber and Petr Zamolodtchikov.
\newblock Local convergence rates of the least squares estimator with
  applications to transfer learning.
\newblock \emph{arXiv preprint arXiv:2204.05003}, 2022.

\bibitem[Sch{\"o}lkopf and Smola(2002)]{scholkopf2002learning}
Bernhard Sch{\"o}lkopf and Alexander~J Smola.
\newblock \emph{Learning with Kernels: Support Vector Machines, Regularization,
  Optimization, and Beyond}.
\newblock MIT Press, 2002.

\bibitem[Shimodaira(2000)]{shimodaira2000}
Hidetoshi Shimodaira.
\newblock Improving predictive inference under covariate shift by weighting the
  log-likelihood function.
\newblock \emph{Journal of Statistical Planning and Inference}, 90\penalty0
  (2):\penalty0 227--244, 2000.

\bibitem[Smale and Zhou(2004)]{smale2004shannon}
Steve Smale and Ding-Xuan Zhou.
\newblock Shannon sampling and function reconstruction from point values.
\newblock \emph{Bulletin of the American Mathematical Society}, 41\penalty0
  (3):\penalty0 279--305, 2004.

\bibitem[Smale and Zhou(2007)]{smale2007learning}
Steve Smale and Ding-Xuan Zhou.
\newblock Learning theory estimates via integral operators and their
  approximations.
\newblock \emph{Constructive Approximation}, 26\penalty0 (2):\penalty0
  153--172, 2007.

\bibitem[Steinwart and Christmann(2008)]{steinwart2008support}
Ingo Steinwart and Andreas Christmann.
\newblock \emph{Support Vector Machines}.
\newblock Springer Science \& Business Media, 2008.

\bibitem[Steinwart et~al.(2009)Steinwart, Hush, and
  Scovel]{steinwart2009optimal}
Ingo Steinwart, Don~R Hush, and Clint Scovel.
\newblock Optimal rates for regularized least squares regression.
\newblock In \emph{COLT}, pages 79--93, 2009.

\bibitem[Sugiyama et~al.(2012)Sugiyama, Suzuki, and
  Kanamori]{sugiyama2012density}
Masashi Sugiyama, Taiji Suzuki, and Takafumi Kanamori.
\newblock \emph{Density Ratio Estimation in Machine Learning}.
\newblock Cambridge University Press, 2012.

\bibitem[Thomas et~al.(2015)Thomas, Theocharous, and
  Ghavamzadeh]{thomas2015high}
Philip Thomas, Georgios Theocharous, and Mohammad Ghavamzadeh.
\newblock High-confidence off-policy evaluation.
\newblock In \emph{Proceedings of the AAAI Conference on Artificial
  Intelligence}, volume~29, 2015.

\bibitem[Tripuraneni et~al.(2021)Tripuraneni, Adlam, and
  Pennington]{tripuraneni2021overparameterization}
Nilesh Tripuraneni, Ben Adlam, and Jeffrey Pennington.
\newblock Overparameterization improves robustness to covariate shift in high
  dimensions.
\newblock \emph{Advances in Neural Information Processing Systems}, 34, 2021.

\bibitem[Tsybakov(2004)]{tsybakov2004optimal}
Alexander~B Tsybakov.
\newblock Optimal aggregation of classifiers in statistical learning.
\newblock \emph{The Annals of Statistics}, 32\penalty0 (1):\penalty0 135--166,
  2004.

\bibitem[Vapnik(1998)]{Vapnik1998}
Vladimir~N. Vapnik.
\newblock \emph{Statistical Learning Theory}.
\newblock Wiley-Interscience, 1998.

\bibitem[Wang(2023)]{wang2023pseudo}
Kaizheng Wang.
\newblock Pseudo-labeling for kernel ridge regression under covariate shift.
\newblock \emph{arXiv preprint arXiv:2302.10160}, 2023.

\bibitem[Wang et~al.(2022)Wang, Chatterji, Haque, and Hashimoto]{wang2022is}
Ke~Alexander Wang, Niladri~Shekhar Chatterji, Saminul Haque, and Tatsunori
  Hashimoto.
\newblock Is importance weighting incompatible with interpolating classifiers?
\newblock In \emph{International Conference on Learning Representations}, 2022.

\bibitem[Wen et~al.(2014)Wen, Yu, and Greiner]{wen2014robust}
Junfeng Wen, Chun-Nam Yu, and Russell Greiner.
\newblock Robust learning under uncertain test distributions: Relating
  covariate shift to model misspecification.
\newblock In \emph{International Conference on Machine Learning}, pages
  631--639. PMLR, 2014.

\bibitem[White(1981)]{white1981consequences}
Halbert White.
\newblock Consequences and detection of misspecified nonlinear regression
  models.
\newblock \emph{Journal of the American Statistical Association}, 76\penalty0
  (374):\penalty0 419--433, 1981.

\bibitem[Williams and Seeger(2000)]{williams2000using}
Christopher Williams and Matthias Seeger.
\newblock Using the nystr{\"o}m method to speed up kernel machines.
\newblock \emph{Advances in Neural Information Processing Systems}, 13, 2000.

\bibitem[Xu et~al.(2021)Xu, Ye, and Ruan]{xu2021understanding}
Da~Xu, Yuting Ye, and Chuanwei Ruan.
\newblock Understanding the role of importance weighting for deep learning.
\newblock In \emph{International Conference on Learning Representations}, 2021.

\bibitem[Yamazaki et~al.(2007)Yamazaki, Kawanabe, Watanabe, Sugiyama, and
  M{\"u}ller]{yamazaki2007asymptotic}
Keisuke Yamazaki, Motoaki Kawanabe, Sumio Watanabe, Masashi Sugiyama, and
  Klaus-Robert M{\"u}ller.
\newblock Asymptotic {Bayesian} generalization error when training and test
  distributions are different.
\newblock In \emph{Proceedings of the 24th International Conference on Machine
  Learning}, pages 1079--1086, 2007.

\bibitem[Yao et~al.(2007)Yao, Rosasco, and Caponnetto]{yao2007early}
Yuan Yao, Lorenzo Rosasco, and Andrea Caponnetto.
\newblock On early stopping in gradient descent learning.
\newblock \emph{Constructive Approximation}, 26\penalty0 (2):\penalty0
  289--315, 2007.

\bibitem[Zhai et~al.(2023)Zhai, Dan, Kolter, and
  Ravikumar]{zhai2023understanding}
Runtian Zhai, Chen Dan, J~Zico Kolter, and Pradeep~Kumar Ravikumar.
\newblock Understanding why generalized reweighting does not improve over
  {ERM}.
\newblock In \emph{The Eleventh International Conference on Learning
  Representations}, 2023.

\bibitem[Zhang et~al.(2012)Zhang, Zhang, and Ye]{zhang2012generalization}
Chao Zhang, Lei Zhang, and Jieping Ye.
\newblock Generalization bounds for domain adaptation.
\newblock \emph{Advances in Neural Information Processing Systems}, 4:\penalty0
  3320, 2012.

\bibitem[Zhang(2005)]{zhang2005learning}
Tong Zhang.
\newblock Learning bounds for kernel regression using effective data
  dimensionality.
\newblock \emph{Neural Computation}, 17\penalty0 (9):\penalty0 2077--2098,
  2005.

\bibitem[Zhou(2002)]{zhou2002covering}
Ding-Xuan Zhou.
\newblock The covering number in learning theory.
\newblock \emph{Journal of Complexity}, 18\penalty0 (3):\penalty0 739--767,
  2002.

\end{thebibliography}

\end{document}